\documentclass[11pt, letterpaper]{article} 
\usepackage[margin=1in]{geometry}
\usepackage{xspace}
\usepackage{multirow}
\usepackage{tabularx, makecell, booktabs}

\usepackage{amsmath,amsfonts,bm,physics,bbm}

\newcommand{\Unif}{{\mathrm{Unif}}}

\def\vxi{{\bm{\xi}}}

\def\gopt{\pi^\mathsf{G}}
\def\pap{\pi^\mathsf{A}}
\def\pmap{\pi^\mathsf{MA}}
\def\psp{\pi^\mathsf{S}}
\def\pmsp{\pi^\mathsf{MS}}
\DeclareMathOperator{\val}{\sV}
\DeclareMathOperator{\info}{\sI}
\def\rval{\sV^{(\lambda)}}

\DeclareMathOperator{\ind}{\mathbbm{1}}

\def\eqref#1{equation~\ref{#1}}

\def\1{\bm{1}}

\def\rmM{{\mathbf{M}}}

\def\vzero{{\bm{0}}}

\def\vtheta{{\bm{\theta}}}

\def\ve{{\bm{e}}}

\def\vs{{\bm{s}}}

\def\vu{{\bm{u}}}
\def\vv{{\bm{v}}}

\def\vx{{\bm{x}}}
\def\vy{{\bm{y}}}

\def\mA{{\bm{A}}}
\def\mB{{\bm{B}}}

\def\mF{{\bm{F}}}
\def\mG{{\bm{G}}}

\def\mI{{\bm{I}}}
\def\mJ{{\bm{J}}}

\def\mM{{\bm{M}}}
\def\mN{{\bm{N}}}

\def\mP{{\bm{P}}}

\def\mU{{\bm{U}}}
\def\mV{{\bm{V}}}
\def\mW{{\bm{W}}}
\def\mX{{\bm{X}}}
\def\mY{{\bm{Y}}}

\def\mLambda{{\bm{\Lambda}}}
\def\mSigma{{\bm{\Sigma}}}

\DeclareMathAlphabet{\mathsfit}{\encodingdefault}{\sfdefault}{m}{sl}
\SetMathAlphabet{\mathsfit}{bold}{\encodingdefault}{\sfdefault}{bx}{n}

\def\gA{{\mathcal{A}}}
\def\gB{{\mathcal{B}}}
\def\gC{{\mathcal{C}}}
\def\gD{{\mathcal{D}}}

\def\gF{{\mathcal{F}}}

\def\gK{{\mathcal{K}}}

\def\gM{{\mathcal{M}}}
\def\gN{{\mathcal{N}}}

\def\gS{{\mathcal{S}}}
\def\gT{{\mathcal{T}}}

\def\gV{{\mathcal{V}}}

\def\sB{{\mathbb{B}}}

\def\sI{{\mathbb{I}}}

\def\sR{{\mathbb{R}}}
\def\sS{{\mathbb{S}}}

\def\sV{{\mathbb{V}}}

\def\sZ{{\mathbb{Z}}}

\newcommand{\E}{\mathbb{E}}

\newcommand{\R}{\mathbb{R}}

\newcommand{\softmax}{\mathrm{softmax}}

\DeclareMathOperator*{\argmax}{arg\,max}

\usepackage{hyperref}

\usepackage[numbers,semicolon,square,authoryear]{natbib}

\usepackage{url}
\usepackage{amsthm, amssymb}
\usepackage{verbatim}
\usepackage[inline]{enumitem}
\usepackage{xcolor}
\usepackage[ruled,noline,noend,linesnumbered]{algorithm2e}

\hypersetup{
  colorlinks   = true, 
  urlcolor     = blue, 
  linkcolor    = blue, 
  citecolor    = blue, 
  linktoc      = page
}
\setlist[enumerate, 1]{label=(\alph*)}
\setlist[enumerate]{nosep}

\theoremstyle{plain}
\newtheorem{thm}{Theorem}
\newtheorem{lem}{Lemma}
\newtheorem{fact}{Fact}

\theoremstyle{definition}
\newtheorem{defn}{Definition}

\theoremstyle{remark}

\SetCommentSty{rmfamily}

\SetKwComment{Comment}{$\triangleright$\ }{}

\SetNlSty{textbf}{}{:}

\makeatletter
\def\namedlabel#1#2{\begingroup
	#2%
	\def\@currentlabel{#2}%
	\phantomsection\label{#1}\endgroup
}
\makeatother
\newcommand{\BatchedLinUCB}{{\sc BatchLinUCB}\xspace}
\newcommand{\BatchedLinUCBKW}{{\sc BatchLinUCB-KW}\xspace}
\newcommand{\BatchedLinUCBDG}{{\sc BatchLinUCB-DG}\xspace}
\newcommand{\SupLinUCB}{{\sc RarelySwitch-SupLinUCB}\xspace}
\newcommand{\defeq}{\stackrel{\mathrm{def}}{=}}
\renewcommand{\eqref}[1]{\hyperref[#1]{(\ref{#1})}}
\newcommand{\appref}[1]{\hyperref[#1]{Appendix~\ref{#1}}}
\renewcommand{\E}{\mathop{\mathbb{E}}}
\makeatletter
\def\blfootnote{\xdef\@thefnmark{}\@footnotetext}
\makeatother

\title{Linear Bandits with Limited Adaptivity and \\ Learning Distributional Optimal Design\blfootnote{Author names are listed in alphabetical order.}}

\author{
Yufei Ruan\thanks{Department of Industrial \& Enterprise Systems Engineering, University of Illinois at Urbana-Champaign. Email: \texttt{yufeir3@illinois.edu}.}
\and  Jiaqi Yang\thanks{Institute for Interdisciplinary Information Sciences, Tsinghua University. Work done while visiting the University of Illinois at Urbana-Champaign.  Email: \texttt{yangjq17@gmail.com}.} 
\and Yuan Zhou\thanks{Department of Industrial \& Enterprise Systems Engineering,
University of Illinois at Urbana-Champaign. Email: \texttt{yuanz@illinois.edu}.}
}

\begin{document}

\maketitle

\thispagestyle{empty}
\begin{abstract}
Motivated by practical needs such as large-scale learning, we study the impact of adaptivity constraints to linear contextual bandits, a central problem in online learning and decision making. We consider two popular limited adaptivity models in literature: batch learning and rare policy switches. We show that, when the context vectors are adversarially chosen in $d$-dimensional linear contextual bandits, the learner needs $O(d \log d \log T)$ policy switches to achieve the minimax-optimal regret, and this is optimal up to $\mathrm{poly}(\log d, \log \log T)$ factors; for stochastic context vectors, even in the more restricted batch learning model, only $O(\log \log T)$ batches are needed to achieve the optimal regret. Together with the known results in literature, our results present a complete picture about the adaptivity constraints in linear contextual bandits. Along the way, we propose the \emph{distributional optimal design}, a natural extension of the optimal experiment design, and provide a both statistically and computationally efficient learning algorithm for the problem, which may be of independent interest.
\end{abstract}

\newpage

\thispagestyle{empty}

\tableofcontents
\newpage
\setcounter{page}{1}

\section{Introduction}

Online learning and decision making is a fundamental research direction in machine learning where the learner conducts sequential interactions, once per time step, with the environment in order to learn the optimal policies and maximize the total reward. To achieve optimal learning performance, the learner must seek a balance between exploration and exploitation, which is usually done by adaptively selecting actions based on all historical observations. However, full adaptivity at a per-time-step scale significantly sacrifices parallelism and hinders the large-scale deployment of learning algorithms. To facilitate scalable learning, it is worthwhile to study the following question: 
\begin{quote}
\emph{What is the minimum amount of adaptivity needed to achieve optimal performance in online learning and decision making?}
\end{quote}

In this paper, we address the above question through studying the impact of two popular types of  adaptivity constraints to the linear contextual bandits, a central problem in online learning literature. We prove tight adaptivity-regret trade-offs for two natural settings of the problem. Along the way, we make a new connection to optimal experiment design: we propose the natural \emph{distributional optimal design} problem, prove the existence of parametric forms for the optimal design, and present sample-efficient algorithms to learn the parameters. Our proposed framework contributes a novel learning component to the classical field of experiment design in statistics, and may be of independent interest.

\paragraph{Linear Contextual Bandits.} The linear contextual bandits (or linear bandits for short), also known as ``associative reinforcement learning'' \citep{abe1999associative,auer2002using}, are a generalization of the ordinary multi-armed bandits. While also encapsulating the fundamental dilemma of ``exploration vs.~exploitation'' in online learning and decision making, linear contextual bandits highlight the guidance of contextual information for decisions, enabling personalized treatments and recommendations in real-world applications such as clinical trial, recommendation systems, and advertisement selection.

In a bandit game, there are $T$ time steps in total. At each time step $t \in [T]$, the learner has to make a decision among $K$ candidate actions (a.k.a.~arms in bandit literature). While in ordinary multi-armed bandits, the mean rewards of the actions have to be completely independent from each other, linear bandits allow a linear model for the mean rewards. More specifically, at time step $t$, each action $i \in [K]$ is associated with a $d$-dimensional context vector $\vx_{ti}$ (a.k.a., the feature vector), and the context vectors are presented to the learner. The expected reward for the $i$-th action is $\vtheta^\top \vx_{ti}$, where $\vtheta \in \sR^d$ is hidden from the learner. The goal is to gradually learn $\vtheta$ and maximize the cumulative expected reward, or equivalently, minimize the expected \emph{regret} (i.e., the difference between the received rewards and the rewards of the best actions in hindsight, as later defined in \eqref{eq:regret}). For example, in clinical trial, the candidate actions correspond to the $K$ involved treatments. At time step $t$, an individual patient arrives with the context vectors $\{\vx_{ti}\}_{i=1}^{k}$  characterizing his/her response to the candidate treatments, and the recovery probability given treatment $i$ is modeled by the linear function $\vtheta^\top \vx_{ti}$, which corresponds to the expected reward in linear bandits.

There are two natural settings of the linear bandits: adversarial and stochastic contexts. The first setting is harder for the learner, as the context vectors are chosen by an adversary and the learner has to minimize the regret in the worst case. In the second setting, in contrast, the sets of context vectors are independently drawn from an unknown distribution $\gD$ (while correlation may still exist among the contexts during the same time step), and the learner aims at minimizing the expected regret over $\gD$. Note that in the clinical trial example, the individual patients can often be viewed as independent samples from the population which is characterized by $\gD$.

\paragraph{Limited Adaptivity Models: Batch Learning and Rare Policy Switches.} We consider two popular models of adaptivity constraints. The first model is batch learning, where the time steps are grouped into pre-defined batches. Within a batch, the same (possibly randomized) policy is used to select actions for all data and the rewards are observed only at the end of the batch. The amount of adaptivity is measured by the number of batches, which is expected to be as small as possible. A notable example is designing clinical trials, where each phase (batch) of the trial involves simultaneously applying medical treatments to a batch of patients. The outcomes are observed at the end of the phase, and may be used for designing experiments in future phases. Finding the correct number and sizes of the batches may achieve optimal efficiency for the trial by creating sufficient intra-batch parallelism while still providing sufficient adaptivity at the inter-batch scale. 

The other model is learning with rare policy switches, where the amount of adaptivity is measured by the number of times allowed for the learner to change the action-selection policy. For the same amount of adaptivity measure, this model can be viewed as a relaxation of the batch learning model, because the learner in the batch learning model can only change the policy at the pre-defined time steps. 

Both of the above models are closely connected to parallel learning, as we will discuss at the end of \autoref{sec:our-contribution}. We also note that another natural limited adaptivity model is ``batch learning with adaptive grid'' \citep{gao2019batched}. This model allows the learner to adaptively decide the size of a batch at the beginning of the batch, which is a more relaxed constraint than batch learning with pre-defined batches (a.k.a., the static grid model) but more restricted than the rare policy switch model, given the same amount of adaptivity measure.\footnote{Indeed, in the adaptive grid model, the time for a policy switch has to be decided when the previous policy switch happens, while in the rare policy switch model, the learner can freely switch the policy, as long as the total number of switches is limited.}  Simple arguments will show that the bounds for the adaptive grid model are the same as the static grid model in both linear bandit settings. Therefore, for succinct exposition, we omit further discussions about the adaptive grid model.\footnote{A simple argument will prove the $\Omega(\sqrt{T})$ batch lower bound for achieving the asymptotically minimax-optimal regret for the adaptive grid model with adversarial contexts, and the rest bounds can be derived by direct corollaries of this work and the existing results in \citep{gao2019batched,han2020sequential}.}

\paragraph{Optimal Experiment Design.} Optimal experiment design seeks to minimize the estimation variances of parameters via intelligently choosing queries to the given set of data points. Among the multiple optimization criteria, the one most related to linear bandits is the G-optimality criterion which seeks to minimize the maximum estimation variance among the given data points. More precisely, given a set of data points $X \subseteq \sR^d$ that spans the full dimension, the goal is to find a distribution $\gK$ supported on $X$, such that $\max_{\vx \in X} \vx^\top (\E_{\vy \sim \gK} \vy\vy^\top)^{-1} \vx$ is minimized. Here, $\info(\gK) = \E_{\vy \sim \gK} \vy\vy^\top$ is the \emph{information matrix} of the design $\gK$, and $\vx^\top \info(\gK)^{-1} \vx$ is the variance of the estimate for data point $\vx$. 
The General Equivalence Theorem of \citet{kiefer1960equivalence} implies that there always exists a design $\gK$ such that $\max_{\vx \in X} \vx^\top \info(\gK)^{-1} \vx \leq d$ and such designs have been used for linear bandits with fixed candidate action set (see Chapter 22 of \citep{lattimore2020bandit}, and \citep{esfandiari2019batched}). However, to the best of our knowledge, traditional optimal design does not address the problem when the candidate action set $X$ is stochastic. In this work, motivated by the algorithmic needs from batch linear bandits, we address this problem and develop a framework named \emph{distributional optimal design} that runs at the core of our algorithm. We will introduce this framework in the next subsection.

\subsection{Our Contributions} \label{sec:our-contribution}

Adaptivity constraints in online learning and decision making have attracted much attention recently. It has been shown that multi-armed bandits only need $O(\log \log T)$ batches to achieve asymptotically minimax-optimal regret \citep{perchet2016batched,gao2019batched}. For linear contextual bandits with adversarial contexts, when $\ln K \geq \Omega(d)$, \citet{abbasi2011improved} showed an optimal-regret algorithm with $O(d \log T)$ policy switches. In contrast, for the batch model, \citet{han2020sequential} recently showed that as many as $\Omega(\sqrt{T})$ batches are needed to achieve the optimal regret bound, implying that batch learning is significantly more restrictive than policy switch constraints for adversarial contexts. 

In light of these partial results, quite a few questions are intriguing and remain to be explored -- What makes the adaptivity requirements of linear contextual bandits fundamentally different from multi-armed bandits? What is the limitation for algorithms with rare policy switches, or in other words, can we extend the algorithm by \cite{abbasi2011improved} to the full parameter range of $K$, and further improve the number of policy switches to $O(\log \log T)$?  Do linear bandits with stochastic contexts require substantially less adaptivity than the adversarial setting? We address these questions and summarize our answers as follows.
\setlist[description]{font=\normalfont}
\begin{description}
\item[\namedlabel{contrib:1}{(C1)}] (Contribution~\#1, informal statements of \autoref{thm:suplinucb} and \autoref{thm:lb}) For linear bandits with adversarial contexts, we show that $d \log T$ (up to $\mathrm{poly}(\log d, \log \log T)$ factors) is the tight amount of policy switches needed to achieve the minimax-optimal regret. To this end, we first extend the algorithm by \cite{abbasi2011improved} to the case where $\ln K \leq o(d)$. Our algorithm achieves the asymptotically minimax-optimal regret with $O(d \log d \log T)$ policy switches. We then prove that our algorithm and the one by \cite{abbasi2011improved} achieve the near-optimal policy switch vs.~regret trade-off. In particular, $\Omega(d\log T / \log (d \log T))$ policy switches are needed to achieve any $\sqrt{T}$-type regret.
\item[\namedlabel{contrib:2}{(C2)}] (Contribution~\#2, an informal statement of \autoref{thm:blinucbdg}) For linear bandits with stochastic contexts, even in the more restricted batch learning model, it is possible to achieve the asymptotically minimax-optimal regret using only $O(\log \log T)$ batches. Our algorithm can be easily adapted to use $M$ batches and achieve $\sqrt{d\log K} T^{\frac{1}{2(1 - 2^{-M})}} \times \mathrm{poly}\log T$ regret, for any $M$. 
\end{description}
\begin{table}[h]
\centering
\vspace{-1em}
\caption{Amount of adaptivity needed in various models and settings for linear bandits.}
\begin{tabular}{|p{1.9cm}|l|l|}
    \hline
         &  \makecell[c]{Batch Learning Model}    &  \makecell[c]{Rare Policy Switch Model} \\
    \hline
    \multirow{2}{*}{\parbox{1.9cm}{Adversarial Contexts}}  &   UB: $O(\sqrt{dT})$ \citep{han2020sequential} &   \makecell[l]{UB: $O(d \log T)$ for $\ln K \geq \Omega(d)$ \citep{abbasi2011improved}\\ ~~~~~~$O(d \log d \log T)$ for $\ln K \leq o(d)$ (by \ref{contrib:1})} \\
    &  LB:  $\Omega(\sqrt{T})$ \citep{han2020sequential}  &   LB: $\Omega(\frac{d \log T}{\log (d \log T)})$ (by \ref{contrib:1})   \\
    \hline
    \multirow{2}{*}{\parbox{1.9cm}{Stochastic Contexts}}   & UB: $O(\log \log T)$ (by \ref{contrib:2})   &   UB: $O(\log \log T)$ (implied by \ref{contrib:2})   \\
    &  LB: $\Omega(\log \log T)$ \citep{gao2019batched}\footnotemark &       LB: $\Omega(\log \log T)$ \citep{simchi2019phase}\footnotemark\\
    \hline
\end{tabular}
\label{tab:results}
\end{table}
\addtocounter{footnote}{-1}
\footnotetext{Implied by the lower bound for multi-armed bandits.}
\addtocounter{footnote}{1}
\footnotetext{Implied by the lower bound for  multi-armed bandits with rare policy switches. Note that the lower bound by \cite{simchi2019phase} is for deterministic action-selection policies, and becomes $\Omega(K \log \log T)$. A simple adaptation of their argument will prove the $\Omega(\log \log T)$ policy switch lower bound for randomized action-selection policies in multi-armed bandits, and imply the same lower bound for linear bandits.}

Together with the known results in literature, we are able to present an almost complete picture about the adaptivity constraints for linear bandits in \autoref{tab:results}. Most interestingly, compared to ordinary multi-armed bandits, linear bandits exhibit a richer set of adaptivity requirements, and strong separations among different models and settings. We also find that adversarially chosen context vectors are the main source of difficulty for reducing adaptivity requirements.

\paragraph{Comparison of \ref{contrib:2} and \citep{han2020sequential}.}  Compared to \ref{contrib:1}, our result in \ref{contrib:2} requires substantially more technical effort and is also the main motivation for us to develop the framework of distributional optimal design (which will be elaborated soon). We note that \citet{han2020sequential} also studied batch learning for linear bandits with stochastic contexts and showed an algorithm with $O(\log \log T)$ batches. However, their results are for a special case of the problem with the following assumptions: the context vectors are drawn from a Gaussian distribution, the ratio between the maximum and minimum eigenvalues of the Gaussian co-variance matrix should be $O(1)$, and the number of candidate actions $K$ cannot be greater than a polynomial of $d$. The design and analysis of their algorithm crucially rely on these three assumptions and it seems not obvious that their result can be directly extended to the general context set distribution. Indeed, their algorithm can safely choose the action to maximize the estimated mean reward, thanks to the isotropic Gaussian assumption ensuring sufficient exploration towards other directions. In contrast, without these assumptions, much effort in our algorithm is spent on the careful design of the exploration policy using many candidate actions, which motivates the problem of distributional optimal design.

\paragraph{Distributional Optimal Design.} As mentioned above, to facilitate the algorithm for stochastic contexts, we have to extend the traditional experiment design results to the regime where the set $X$ of contexts/data points is stochastic. Suppose that $X$ follows the distribution $\gD$, the goal of our proposed \emph{distributional optimal design} problem is to find a sample policy $\pi$ that maps any set $X$ to a probability distribution supported on $X$, so as to minimize the \emph{distributional G-variation}, defined as 
$\E_{X \sim \gD} \max_{\vx \sim X} \vx^\top \info_{\gD}(\pi)^{-1} \vx$, where $\info_{\gD}(\pi) = \E_{X \sim \gD} \E_{\vy \sim \pi(X)} \vy \vy^\top$ is the information matrix of sample policy $\pi$ over $\gD$.\footnote{For simplicity of presentation, we assume that the vectors in the sets of $\gD$ span the full dimension, so that there always exists a sample policy with invertible information matrix. Please refer to \autoref{defn:distributional-g-opt-design} for the general definition.} Note that the traditional G-optimal design is the special case of our problem when $\gD$ is deterministic, which was used in the algorithm for linear bandits with fixed candidate action sets (see, e.g., Chapter 22 of \cite{lattimore2020bandit}). In contrast, the stochasticity of $X\sim \gD$ in our problem arises due to the stochastic context in linear bandits.

The first natural question about our proposed problem is on the existence of a good sample policy. Regarding this, we prove the following result.
\begin{description}
\item[\namedlabel{contrib:3}{(C3)}] (Contribution~\#3, an informal statement of \autoref{thm:ms-policy}) For any $\gD$, there exists a sample policy $\pi$ such that the distributional G-variance is bounded by $O(d \log d)$.\footnote{This bound can be improved to $O(d)$ with additional techniques, which will be included in the full version of the paper.} Moreover, we can construct such a policy from the class of so-called \emph{mixed-softmax policies}, which admits a succinct description using $O(d^3 \log d)$ real-valued parameters. 
\end{description}
Since $\gD$ is not known beforehand in linear bandits, we have to learn a good sample policy $\pi$ via finite samples from $\gD$. Since even the input of $\pi$ lie in a continuous space with $dK$ dimensions, proving the existence of the succinct parametric form of $\pi$ in \ref{contrib:3} is a good news to learning. However, we find that directly constructing a policy based on the uniform distribution over empirical samples does not generalize to the true distribution $\gD$. We will come up with a more careful learning procedure to achieve the following goal.
\begin{description}
\item[\namedlabel{contrib:4}{(C4)}] (Contribution~\#4, an informal statement of \autoref{thm:learn-design}) For any $\gD$, we design an algorithm to learn a good mixed-softmax policy $\pi$ using only $\mathrm{poly}(d)$ independent samples from $\gD$.\footnote{More precisely, the good policy here is defined by the \emph{distributional G-deviation}. Please refer to \autoref{thm:learn-design} for more details.}
\end{description}
We remark that the introduction of the distribution $\gD$ brings a unique learning challenge to optimal experiment design. It is hopeful that our results and the future study on other criteria in distributional optimal design may lead to broader applications in machine learning and statistics.

\paragraph{Implications for Collaborative and Concurrent Learning.} The idea of letting multiple learning agents learn in parallel so as to save overall running time has been studied a lot recently in online learning and decision making, which is also the main motivation of this study (as mentioned in the very beginning of the paper). Below we discuss the implications of our algorithmic results for a few parallel learning models.

The first implication is for the \emph{collaborative learning with limited interaction} model, which was recently  studied for pure exploration (i.e., top arm(s) identification) in multi-armed bandits \citep{hillel2013distributed,tao2019collaborative,karpov2020collaborative}. In this model, there are $\mathfrak{K}$ learning agents, and the learning process is partitioned into rounds of pre-defined time intervals. During each round (which is also referred to as the \emph{communication round}), each of the $\mathfrak{K}$ agents learns individually like in the centralized model -- image that there is a global buffer of the context vectors, and the agents repeatedly draw a set of context vectors from the buffer and make corresponding decisions. Each play of an arm takes one time step, and the agents may choose to skip a few time steps without playing. The agents can only communicate at the end of each round. The collective regret is defined to be the sum of the regret incurred by each agent. Suppose there are $T$ sets of context vectors in the global buffer, the goal is to finish the game in $O(\lceil T/\mathfrak{K}\rceil)$ time (i.e., achieving the \emph{full speedup}), while minimizing the collective regret and the number of communication rounds $R$. 

Observe that a batch learning algorithm with $M$ batches can be easily transformed to a collaborative algorithm with $R = M$ communication rounds, where in each round $i$, each agent uses the policy for the $i$-th batch to play for $\lfloor \gT_i/\mathfrak{K}\rfloor$ or $\lceil \gT_i/\mathfrak{K}\rceil$ times, where $\gT_i$ is the size of the $i$-th batch. The total running time for collaborative learning is at most $T /\mathfrak{K} + M$, achieving the full speedup when $M \cdot \mathfrak{K} \leq O(T)$. Therefore, when $\mathfrak{K} \leq O(T/\log\log T)$, our algorithmic result \ref{contrib:2} implies a collaborative algorithm for stochastic-context linear bandits with full speedup and minimax-optimal collective regret, using only $O(\log \log T)$ communication rounds.

The second implication is for the \emph{concurrent learning} model which was recently studied in \citep{guo2015concurrent,bai2019provably,zhang2020almost}. In this model, there is no limit on the number of communication rounds and the $\mathfrak{K}$ learning agents may communicate at the end of every time step. By a simple reduction described in \citep{bai2019provably}, any algorithm with at most $M$ policy switches can be transformed to a $\mathfrak{K}$-agent concurrent learning algorithm with full speedup, and the collective regret is at most $M \cdot \mathfrak{K}$ plus the original regret bound. Therefore, our algorithmic result in \ref{contrib:1} implies a concurrent learning algorithm for adversarial-context linear bandits with full speedup and minimax-optimal collective regret, as long as $\mathfrak{K} \leq O(\sqrt{(T \log K)/d})$.

\subsection{Additional Related Works}

The linear contextual bandit problem is a central question in online learning and decision making, and its regret minimization task has been studied during the past decades \citep{auer2002using, abe2003reinforcement, dani2008stochastic, rusmevichientong2010linearly, chu2011contextual, abbasi2011improved, li2019nearly}. The minimax-optimal regret is proved to be $\sqrt{d T \min\{\log K, d\}}$ up to $\mathrm{poly} \log T$ factors, which is also the target regret for our algorithms with limited adaptivity. When the candidate action set is fixed, the task of identifying the best action has also been studied \citep{soare2014best, tao2018best, xu2018fully}, and many of these works borrow the idea of G-optimal design.

Batch regret minimization for multi-armed bandits was introduced by \citet{perchet2016batched} with 2 arms, and the $K$-arm general setting was recently studied by \citet{gao2019batched}. \citet{simchi2019phase} studied the $K$-arm setting with the rare policy switch constraint and achieved comparable results. For batch linear bandits, \citet{esfandiari2019batched} and \citet{han2020sequential} recently studied the problem with aforementioned additional assumptions. For batch stochastic contextual bandits, \citet{simchi2020bypassing} recently proposed an algorithm with $O(\log \log T)$ batches to achieve the minimax-optimal regret. We note that another usage of batch learning (mainly in reinforcement learning) refers to learning from a fixed set of a priori-known samples with no adaptivity allowed, which is very different from the definition in our work.

For the rare policy switch model, \cite{abbasi2011improved} showed a rarely switching algorithm for linear bandits. Rare policy switch constraints have also been studied for a broader class of online learning and decision making problems, such as multinomial logit bandits \citep{dong2020multi} and Q-learning \citep{bai2019provably}. 

Under the broader definition of adaptivity constraints including batch learning and learning with low switching cost (which might not exactly align with the models defined in this work), many other online learning problems are studied, such as adversarial multi-armed bandits \citep{cesa2013online, dekel2014bandits}, the best (multiple-)arm identification problem \citep{jun2016top,agarwal2017learning}, and convex optimization \citep{duchi2018minimax}. 

The optimal design of experiments is a fundamental problem in statistics, with various optimality criteria proposed and many statistical models studied (see, e.g., \citep{pukelsheim2006optimal,atkinson2007optimum}). When the sample budget is finite, finding the exact solutions to certain optimality criteria is NP-Hard \citep{welch1982algorithmic, ccivril2009selecting, summa2014largest}, thus a sequence of recent works have studied approximation algorithms for the problem \citep{wang2017computationally, singh2018approximate, madan2019combinatorial, nikolov2019proportional, allen2020near}. However, to the best of our knowledge, all previous works have considered the fixed set of all possible experiments. In contrast, we propose and study the distributional optimal design problem where the set of candidate experiments might be stochastic.

\section{Technical Overview}

\subsection{Batch Algorithms for Stochastic Contexts}

As the main technical contribution, we first describe the techniques developed in \autoref{sec:blinucb}, \autoref{sec:dist-opt-design} and \autoref{sec:learn-design} for proving our algorithmic result \ref{contrib:2}. Along the way, the proof techniques for \ref{contrib:3} and \ref{contrib:4} are also explained. In \autoref{sec:blinucbdg}, we combine all these technical components and prove the main theorem.

\paragraph{The Batch Elimination Framework.} All our algorithms are elimination-based: at each time step, the confidence intervals are estimated for each candidate action, and the actions whose confidence intervals completely fall below those of other actions are eliminated. All survived actions are likely to be the optimal one, and the learner has to design an intelligent sample policy $\pi$ to select the action from the survived set. In such a way, the incurred regret can be bounded by the order of the length of the longest confidence interval in the survived set. 

We note that this elimination-based approach is not new: it is adopted by the batch algorithms for multi-armed bandit (e.g., \citep{gao2019batched}) as well as the recent batch algorithm for linear bandits with fixed action set \citep{esfandiari2019batched}. However, thanks to the simple structures of the two problems, during each batch, both of their algorithms are able to construct confidence intervals for survived actions with a \emph{uniform} length, so that the regret can be relatively more easily bounded. Indeed, although the algorithm by \citet{han2020sequential} does not explicitly eliminate actions, their analysis relies on the uniform estimation confidence for the actions (which requires the isotropic Gaussian assumption for context vectors). In contrast, we have to deal with confidence intervals with wildly different lengths because of the inherent non-uniformity of the probability mass assigned to each context direction in the general distribution $\gD$. 

To deal with such non-uniformity, in \autoref{sec:blinucb}, we provide an analysis framework to relate the regret bound to the distributional G-variation of $\pi$ over $\gD$, as introduced in \autoref{sec:our-contribution}. In particular, we show that if we let $\pi(X) = \gopt(X)$, which returns the G-optimal design of the input context set $X$ (regardless of $\gD$), its distributional G-variation can be bounded by $d^2$ (for all $\gD$), leading to $O(d \sqrt{T \log K}) \times \mathrm{poly} \log T$ regret with $O(\log \log T)$ batches. This regret is $\sqrt{d}$ times greater than the minimax-optimal target. To achieve optimality, we need to improve the distributional G-variation to $O(d)$ (up to logarithmic factors), which requires to optimize $\pi$ specifically according to $\gD$.

\paragraph{Existence of Distributional Optimal Design and its Parametric Form.} In \autoref{sec:dist-opt-design}, we show that, given $\gD$, there exists a sample policy $\pi$ whose distributional G-variation is $O(d \log d)$. Our proof is constructive and the algorithm involves an innovative application of the rarely switching linear bandit algorithm \citep{abbasi2011improved}. We consider a long enough sequence of independent samples from $\gD$: $X_1, X_2, \dots, X_N$, and sequentially feed the context vector sets to the rarely switching algorithm. Instead of minimizing the regret (as the reward is undefined), the rarely switching algorithm selects the context vector $\vx$ that maximizes the variance according to the \emph{delayed information matrix}, and updates the total information matrix by adding $\vx \vx^\top$ to it. 

Borrowing the regret analysis techniques in linear bandits literature, and together with an adapted form of the celebrated Elliptical Potential Lemma, we are able to prove that, with the proper configuration of the initial information matrix, the average maximum confidence interval length throughout the $N$ time steps is $O(d\log d)$. Moreover, the rarely switching trick makes sure that the delayed information matrix switches for at most $O(d \log d)$ times. This allows us to extract $O(d \log d)$ (deterministic) sample policies $\{\pi_j\}$ from the execution trajectory of the algorithm, each of which chooses the variance maximizer according to a delayed information matrix in the trajectory. We also associate each $\pi_j$ with a probability mass $p_j$, which is proportional to the number of time steps when the corresponding delayed information matrix is used in the trajectory. We can then construct a so-called \emph{mixed-argmax policy} $\pi$ as follows: with probability $1/2$, $\pi$ acts the same as $\gopt$; otherwise, $\pi$ acts the same as $\pi_j$ with probability $p_j$.

We are then able to prove that the distributional G-variance of $\pi$ over $\gD$ is $O(d \log d)$. This is done mainly by showing that $\info_{\gD}(\pi)$ is comparable to the final information matrix in the trajectory, so that the distributional G-variance of $\pi$ can be bounded by the empirical average of the maximum confidence interval lengths. To lower bound $\info_\gD(\pi)$ using the total information matrix in the trajectory, while the portion corresponding to the larger switching window (i.e., greater $p_j$) in the trajectory can be directly compared, the smaller switching window will be handled by the $\gopt$ component in $\pi$. We note that the $\gopt$ component is also crucial to configuring the ``proper'' initial information matrix in the rarely switching algorithm. 

We finally observe that $\pi$ can be characterized by $O(d^3 \log d)$ parameters, because each $\pi_j$ is parameterized by a $d \times d$ information matrix. Since the $\argmax$ operator could be very sensitive to noise when the top input elements are close, to facilitate learning, we will also work on the \emph{mixed-softmax} policy where each $\pi_j$ uses the $\softmax$ operator instead.

\paragraph{{\sc CoreLearning} for Distributional Optimal Design.} It is tempting to build the natural learning algorithm that computes the distributional optimal design from the empirical samples, with the hope that the Lipschitz-continuity property of the softmax policies provides a small covering of the policy space, which leads to uniform concentration results, and finally prove that the learned policy generalizes to the true distribution $\gD$. However, in \autoref{sec:learn-design}, we construct an example to show that such an approach requires much higher sample complexity than we can afford.

To enable sample-efficient learning, we propose a new algorithm, {\sc CoreLearning}, that first identifies a \emph{core} set, which is a subset of the empirical samples, and then computes a mixed-softmax policy from the core. To identify the core, we develop a novel procedure to iteratively prune away the sets that contain less explored directions among the empirical samples, so that the set of the remaining samples at the end of the procedure becomes the core. Via a volumetric argument, we show that the directions in the core can be sufficient explored even if \emph{only} using the sets in the core, and the core is still overwhelmingly large. Both properties are crucially used in the {\sc CoreLearning} algorithm.

The high-level idea behind {\sc CoreLearning} is that, on one hand, we can prove fast uniform concentration for the information matrix if all directions are sufficiently explored, so that the directions spanned by the core can be handled. On the other hand, the directions not included in the core are infrequent in $\gD$ (because the core is large enough), and can be dealt with by the $\gopt$ component in the mixed-softmax policy. 

Much technical effort is devoted to the analysis of {\sc CoreLearning} because (1) it seems not quite obvious whether a core with the desired properties even exists, and (2) a careful analysis is needed when combining the analysis for sufficiently explored directions and infrequent directions, since the (possible) directions of the context vectors are continuous, and the boundary between the two types of directions may not be always clear. Please refer to \autoref{sec:learn-design} for more detailed explanation.

\subsection{Policy Switch Bounds for Adversarial Contexts}

\paragraph{The Algorithm with Rare Policy Switches.} 
We first recall that \citet{abbasi2011improved} proposed a determinant-based doubling trick that only updates the policy when the determinant of the associated information matrix doubles. When applying to the {\sc OFUL} algorithm in \citep{abbasi2011improved}, the doubling trick leads to $O(d \log T)$ policy switches. However, due to technical difficulty, the state-of-the-art analysis for the {\sc OFUL} algorithm shows the asymptotically  minimax-optimal regret only for $\ln K \geq \Omega(d)$. While it is still an open question whether a simple adaptation of {\sc OFUL} (such as {\sc LinUCB} proposed in \citep{chu2011contextual}) also achieves  the asymptotically  minimax-optimal regret for $\ln K \leq o(d)$, the only known technique in literature to achieve the optimality is via building a more sophisticated ``super algorithm'' based on the idea of {\sc LinUCB} (e.g., {\sc SupLinUCB} \citep{chu2011contextual} and {\sc SupLinRel} \citep{auer2002using}). There are $\Theta(\log T)$ information matrices maintained in these super algorithms, and therefore a na\"ive application of the determinant-based doubling trick to these super algorithms leads to $O(d \log^2 T)$ policy switches. 

To improve the number of policy switches, we adopt a simple combination of {\sc OFUL} and {\sc SupLinUCB}, so that only $O(\log d)$ information matrices are maintained, leading to $O(d \log d \log T)$ policy switches. Please refer to \autoref{sec:rarely-switch-suplinucb} for more detailed explanation.


\paragraph{The Lower Bound.} In \autoref{sec:lb}, we prove that to achieve any $\sqrt{T}$-type regret in the adversarial context setting, the algorithm has to switch the policy for at least $\Omega(d \log T / \log (d \log T))$ times. We first observe that the classical hard instances for linear bandits in \citep{dani2008stochastic} cannot work for our goal since their context vector set does not change with time and therefore their instances can be solved by our algorithm for stochastic contexts using $O(\log \log T)$ batches. Instead, we divide the $T$ time steps into stages of consecutive time periods, and design different context vectors for different stages. We will design a class of specially structured hard instances, where the hidden vector $\vtheta$ delicately matches with the context vectors in each instance. We then lower bound the average-case regret over the class for any rarely switching learner, which implies the worst-case regret lower bound.  

At a higher level, our construction is more similar to the recent work by \citet{li2019nearly}. However, the difference is that, in their construction, the regret that can be incurred by the worst learner is no more than $\sqrt{T}$ (up to polynomial factors in $d$ and $\log T$). In contrast, in our task, we need to show that the learner could easily incur $T^{1/2 + \Omega(1)}$ regret when using few policy switches. To achieve this, we need to design a class of hidden vectors $\vtheta$ and context vectors $\{\vx_{ti}\}$ so that the mean rewards of the candidate actions are much more separated from each other, while we still have to make sure that a rarely switching learner cannot learn enough information. 

\section{Preliminaries}

\paragraph{Notations.} Throughout the paper, we denote $[N] \defeq \{1, 2, \dots, N\}$ for any integer $N$. We define $\log x \defeq \log_2 x$ and $\ln x \defeq \log_e x$. We use $\ind[\cdot]$ to denote the indicator variable for a given event (i.e., the value of the variable is $1$ if the event happens, and $0$ otherwise). We use $\norm{\cdot}$ to denote the 2-norm of matrices and vectors. Matrix and vector variables are displayed in bold letters. For any discrete set $X$, we use $\triangle_X$ to denote the set of all probability distributions supported on $X$.

\paragraph{Linear Contextual Bandits.} There is a hidden vector $\vtheta$ ($\norm{\vtheta} \leq 1$). For a given time horizon $T$, the context vectors $\{\{\vx_{ti}\}_{i=1}^K\}_{t=1}^T$ are drawn from the product distribution $\mathcal{D}_1 \otimes \mathcal{D}_2 \otimes \dots \otimes \mathcal{D}_T$, where $\mathcal{D}_t$ is the distribution for the context vectors at time step $t$. We assume $\norm{\vx_{ti}} \leq 1$ for all $i$ and $t$ almost surely. Before the game starts, the learner only knows $T$.

At each time step of the game $t = 1, 2,\dots, T$, the learner has to first decide a policy $\chi_t$ that maps any set of context vectors $X$ to a distribution in $\triangle_X$. The learner then observes $X_t = \{\vx_{ti}\}_{i=1}^K$, samples an action $i_t$ from $\chi_t(X_t)$,\footnote{When clear from the context, we interchangeably use the arm indices and their corresponding context vectors.} plays arm $i_t$, and finally receives the reward  $r_t = \vtheta^\top \vx_{t, i_t} + \varepsilon_{t}$, where $\varepsilon_{t}$ is an independent sub-Gaussian noise with variance proxy at most $1$.

The goal of the learner is to minimize the expected regret
\begin{align}
R^T \defeq \E\left[\sum_{t = 1}^T \max_{i \in [K]} \vx_{ti}^\top \vtheta - \vx_{t,i_t}^\top \vtheta\right], \label{eq:regret}
\end{align}
where the expectation is taken over $\mathcal{D}_1 \otimes \mathcal{D}_2 \otimes \dots \otimes \mathcal{D}_T$, the noises, and the internal randomness of the learner. In our algorithmic results, we also prove \emph{$(1-\delta)$-high probability expected regret}, which is defined as $\sup_{A} \E\left[\ind[A] \cdot \sum_{t = 1}^T \max_{i \in [K]} \vx_{ti}^\top \vtheta - \vx_{t,i_t}^\top \vtheta\right]$ where the supremum is taken over all events $A$ such that $\Pr[A] \geq 1 - \delta$. In this definition, setting $\delta = O(1/T)$ recovers the usual expected regret up to an additive error of $O(1)$.

\paragraph{Settings of Adversarial and Stochastic Contexts.} In the setting of adversarial contexts, there are no additional constraints for the distributions $\{\gD_t\}$. Note that this corresponds to the \emph{oblivious adversary} in bandit literature, meaning that the adversary has to choose all context vectors beforehand. In contrast, the stronger \emph{non-oblivious adversary} may adaptively choose context vectors for any time step according to all game history before that time. Since we only prove lower bounds for the adversarial context setting in this work, dealing with a weaker adversary actually means a stronger lower bound result.

In the setting of stochastic contexts, we have the additional assumption that $\gD = \gD_1 = \dots = \gD_T$. However, correlation may still exist among the contexts at the same time step.

\paragraph{Models for Limited Adaptivity.} In the batch learning model, the learner has to first choose a grid  $\gT = \{\gT_0, \gT_1, \dots, \gT_M\}$  where $1 = \gT_0 < \gT_1 < \gT_2 < \cdots < \gT_{M-1} < \gT_M = T$. For any $i \in [M]$, the $i$-th batch consists of the time steps $t = \gT_{i-1} + 1, \dots, \gT_i$. During the $i$-th batch, the learner must choose the policy $\chi^{(i)}$ at the beginning of the batch, and the same policy will be used throughout the batch. The amount of adaptivity is measured by $M$, the number of batches.

In the rare policy switch model, a policy switch occurs at time step $t > 1$ if $\chi_t \neq \chi_{t-1}$, and there is always a policy switch at time step $1$. The amount of adaptivity is measured by the number of policy switches. 

As mentioned before, the goal for the learner is to achieve the target minimax-optimal regret $\sqrt{d T \min\{\log K, d\}}$ (up to $\mathrm{poly}\log T$ factors) with as little adaptivity (measured in each specific model) as possible. We also remark that the rare policy switch model is a relaxation of the batch learning model, because the learner can decide whether to change the policy at any time step. Therefore, the amount of adaptivity needed in the rare policy switch model is always less than or equal to the batch model.

\section{Batch Elimination Framework and the G-Optimal Design}
\label{sec:blinucb}

\begin{algorithm}[t]
\caption{\BatchedLinUCB}
\label{algo:linucb}
$M = \lceil \log \log T \rceil, \alpha \gets   10\sqrt{\ln \frac{2dKT}{\delta}}$, $\gT = \{\gT_1, \gT_2, \dots, \gT_M\}, \gT_0 = 0, \gT_M = T, \forall i \in [M - 1]: \gT_i = T^{1 - 2^{-i}}$\;
\For{$k \gets 1, 2, \dots, M$}{
 $\lambda \gets 16 \ln (2dT/\delta), \mLambda_{k} \gets \lambda \mI, \vxi_{k} \gets \vzero$\;
\For{$t \gets\gT_{k - 1} + 1, \gT_{k - 1} + 2, \dots, \gT_k$}{
$A_{t}^{(0)} \gets [K], \hat r_{ti}^{(0)} \gets 0, \omega_{ti}^{(0)} \gets 1$\;
\For(\Comment*[f]{Eliminate}){$\kappa \gets 1, 2, \dots, k - 1$ \label{line:batchedlinucb-6}}{
$\forall i \in A_{t}^{(\kappa-1)} : \hat r_{ti}^{(\kappa)} \gets  \vx_{t i}^\top \hat \vtheta_{\kappa}, \omega_{t i}^{(\kappa)} \gets \alpha \sqrt{\vx_{t i}^\top \mLambda_{\kappa}^{-1} \vx_{t i}}$\;
$ A_{t }^{(\kappa)} \gets \{i \in  A_{t }^{(\kappa-1)} \mid \hat r_{ti}^{(\kappa)} +   \omega_{ti}^{(\kappa)} \ge \hat r_{tj}^{(\kappa)} -  \omega_{tj}^{(\kappa)},  \forall j \in  A_{t }^{(\kappa-1)}\}$\; \label{line:batchedlinucb-8}
}
$A_t \gets A_{t}^{(k - 1)}$\; 
play arm $i_t \sim \Unif(A_{t})$, and receive reward $r_t$\;\label{line:batchedlinucb-10}
    $\vx_t \gets \vx_{t, i_t}, \mLambda_{k} \gets \mLambda_{k} + \vx_{t} \vx_{t}^\top, \vxi_{k} \gets \vxi_{k} + r_t \vx_t$\;
}
$\hat \vtheta_{k} \gets \mLambda_{k}^{-1} \vxi_k$\;
}
\end{algorithm}

As a warm-up, in this section, we first present  \BatchedLinUCB (\autoref{algo:linucb}) to illustrate the batch elimination framework for the linear bandit problem with stochastic contexts. Later in \autoref{sec:blinucbkw}, we will introduce the G-optimal experiment design and show how it helps to reduce the regret bound of the algorithm. While the regret bound in \autoref{thm:BatchedLinUCBKW} is improved, it still has an extra $\sqrt{d}$ factor compared to the optimal minimax regret bound (without adaptivity constraints). The quest for optimal regret will be addressed in the later sections.

We now introduce our first algorithm. \BatchedLinUCB (\autoref{algo:linucb}) uses $M = O(\log \log T)$ batches and a pre-defined static grid $\gT = \{\gT_1, \gT_2, \dots, \gT_M\}$. For each batch $k$, \BatchedLinUCB keeps an estimate $\hat{\vtheta}_k$ for the hidden vector $\vtheta$, which is learned using the samples obtained in the batch. To decide an arm during any time $t$ in the $k$-th batch, the algorithm first performs an elimination procedure that is based on the estimate $\hat{\vtheta}_{\kappa}$ and the corresponding confidence region for each previous batch $\kappa \in \{1, 2, \dots, k-1\}$. Let $A_t$ be the set of survived arms after the elimination. The algorithm then plays a uniformly random arm from $A_t$. The following theorem upper bounds the regret of \BatchedLinUCB.

\begin{thm}\label{thm:BatchedLinUCB}
With probability at least $(1 - \delta)$, the expected regret of \BatchedLinUCB is 
\begin{align*}
R^T_{\text{\BatchedLinUCB}} \leq O(\sqrt{dKT \log (dKT/\delta)} \times  \log \log T).
\end{align*}
\end{thm}

To prove \autoref{thm:BatchedLinUCB}, we first introduce the following lemma that constructs the confidence intervals of the estimated rewards. 

\begin{lem} \label{lem:conf} Fix any batch $k$, for each time step $t$ in batch $k$, with probability at least $(1 - \delta/T^2)$, for all $\kappa \in \{1, 2, \dots, k-1\}$ and all $i \in A_t$, we have that
\begin{align*}
    \abs{\vx_{ti}^\top \hat\vtheta_{\kappa} - \vx_{ti}^\top \vtheta} \le \omega_{ti}^{(\kappa)}.
\end{align*}
\end{lem}
The proof of \autoref{lem:conf} can be found in many papers in linear bandit literature (e.g., \citep{chu2011contextual, li2019nearly}), and is included in \appref{app:analysis-lr} for completeness.

We now start proving \autoref{thm:BatchedLinUCB}.

Fix any batch $k$ such that $k \geq 2$, when conditioned on the first $(k-1)$ batches, we let $\gD_k$ be the distribution of the survived candidate arms $X = \{\vx_{ti} : i \in A_t\}$ at any time $t$ during the $k$-th batch. We also let $\gD_0 = \{\vx_{ti}\}$ be the distribution of all candidate arms at any time $t$.

Suppose that the desired event in \autoref{lem:conf} happens for every time step during the $k$-th batch (which happens with probability at least $(1 - \delta \gT_k/T^2)$ by a union bound), it is straightforward to verify that for each time $t$ during the $k$-th batch, the optimal arm is not eliminated by the elimination procedure (\autoref{line:batchedlinucb-6} to \autoref{line:batchedlinucb-8}) in \BatchedLinUCB. In other words, we have that $i^*_t \defeq \argmax_{i\in [K]} \vx_{ti}^\top \vtheta \in A_t$ for each time step $t$ in the $k$-th batch. Therefore, we can now upper bound the expected regret incurred during batch $k$ as 
\begin{align}
R_k & = \E \sum_{t \text{~in batch $k$}} (\max_{i \in [K]} \vx_{ti}^\top \vtheta - \vx_{t, i_t}^\top \vtheta) 
 \leq  \E \sum_{t \text{~in batch $k$}} (\vx_{t,i^*_t}^\top \hat{\vtheta}_{k-1} - \vx_{t,i_t}^\top \hat{\vtheta}_{k-1} + \omega_{t, i^*_t}^{(k-1)} + \omega_{t, i_t}^{(k-1)}) \label{eq:thm-blinucb-10}\\
& \leq  \E \sum_{t \text{~in batch $k$}}  2 \cdot (\omega_{t, i^*_t}^{(k-1)} + \omega_{t, i_t}^{(k-1)}) \leq 4 \E \sum_{t \text{~in batch $k$}} \max_{i \in A_t} \omega_{ti}^{(k-1)} , \label{eq:thm-blinucb-20}
\end{align}
where \eqref{eq:thm-blinucb-10} is due to the successful events of \autoref{lem:conf}, the both inequalities in \eqref{eq:thm-blinucb-20} are due to the elimination process and that $i^*_t \in A_t$. By the definition of $\omega_{ti}^{(k-1)}$ and the definition of $\gD_k$, we further have that
\begin{align} 
R_k\leq 4\alpha \E\sum_{t \text{~in batch $k$}} \max_{i \in A_t} \sqrt{\vx_{ti}^\top \mLambda_{k-1}^{-1} \vx_{ti}}  \leq 4\alpha \times \sum_{t \text{~in batch $k$}} \E_{X \sim \gD_{k}}\max_{\vx \in X} \sqrt{\vx^\top \mLambda_{k-1}^{-1} \vx}.  \label{eq:thm-blinucb-30}
\end{align}
We finally observe that $X \sim \gD_k$ can be sampled by drawing an $X' \sim \gD_{k-1}$ and performing an elimination process using $\hat\vtheta_{k-1}$ as well as the corresponding confidence region for $X'$. We note that $X \subseteq X'$. Therefore, continuing with \eqref{eq:thm-blinucb-30}, we have that
\begin{align}
R_k \leq 4\alpha \times \sum_{t \text{~in batch $k$}} \E_{X \sim \gD_{k-1}}\max_{\vx \in X} \sqrt{\vx^\top \mLambda_{k-1}^{-1} \vx} = 4\alpha \gT_k \times \E_{X \sim \gD_{k-1}}\max_{\vx \in X} \sqrt{\vx^\top \mLambda_{k-1}^{-1} \vx}. \label{eq:thm-blinucb-40}
\end{align}

Now the goal is to upper bound $\E_{X \sim \gD_{k-1}}\max_{\vx \in X} \sqrt{\vx^\top \mLambda_{k-1}^{-1} \vx}$. The following lemma is a direct application of  \autoref{lem:conconcut}  in \appref{app:concentration}.

\begin{lem} \label{lem:exploration} For each batch $k$ ($k < M$), with probability $(1 - \delta/T^2)$, we have that 
\begin{align}
    \mLambda_k \succcurlyeq \frac{\gT_{k}}{16} \left(\frac{ \ln T}{ \gT_{k}} \mI + \E_{X \sim \gD_{k}} \E_{\vx \sim \Unif(X)}[\vx \vx^\top]\right). \label{eq:lem-exploration}
\end{align}
\end{lem}
Assuming that \eqref{eq:lem-exploration} holds for batch $(k-1)$, we have that 
\begin{align*}
& \quad \E_{X \sim \gD_{k-1}}\max_{\vx \in X} \sqrt{\vx^\top \mLambda_{k-1}^{-1} \vx}  \leq \E_{X \sim \gD_{k-1}}\sum_{\vx \in X} \sqrt{\vx^\top \mLambda_{k-1}^{-1} \vx} \\
&\leq   \frac{4}{\sqrt{\gT_{k-1}}} \sqrt{ \E_{X \sim \gD_{k-1}}\sum_{\vx \in X}\vx^\top \left(\frac{ \ln T}{ \gT_{k-1}} \cdot \mI + \E_{Y \sim \gD_{k-1}} |Y|^{-1} \cdot \sum_{\vy \in Y} \vy \vy^\top\right)^{-1} \vx} \\
&\leq   \frac{4}{\sqrt{\gT_{k-1}}}\sqrt{ \Tr \left( \left(\frac{ \ln T}{ \gT_{k-1}} \cdot \mI + \E_{Y \sim \gD_{k-1}} K^{-1} \cdot \sum_{\vy \in Y} \vy \vy^\top\right)^{-1}  \E_{X \sim \gD_{k-1}}\sum_{\vx \in X} \vx \vx^\top \right)} \\
&\leq 4 \sqrt{dK/\gT_{k-1}}.
\end{align*}
Together with \eqref{eq:thm-blinucb-40}, and collecting the probabilities, we have that with probability at least $(1 - \delta \gT_k/T^2 - \delta/T^2)$, the expected regret incurred during batch $k$ ($k \geq 2$) is
\begin{align}
R_k \leq 16\alpha \gT_k \cdot  \sqrt{dK/\gT_{k-1}} \leq 16\alpha \sqrt{dKT}. \label{eq:thm-blinucb-100}
\end{align}
Note that \eqref{eq:thm-blinucb-100} also holds for $k=1$ almost surely, because $\gT_1 \leq \sqrt{dT}$ and the maximum regret incurred per time step is at most $1$.

Finally, summing up the expected regret incurred across all batches and collecting the probabilities, we have that, with probability at least $(1 - \delta)$, the expected regret is bounded by 
\begin{align}
R^T \leq M \times 16\alpha  \sqrt{dKT} = O(\sqrt{dKT \log (dKT/\delta)} \times  \log \log T). \notag 
\end{align}
This concludes the proof of \autoref{thm:BatchedLinUCB}.

\subsection{Improved Regret via the G-Optimal Design}\label{sec:blinucbkw}

In this subsection, we show how a simple application of the G-optimal design can help to replace the $K$ factor in \autoref{thm:BatchedLinUCB} by (the usually smaller quantity) $d$. To achieve this, we first introduce the following lemma on G-optimal design, which is a direct corollary of the General Equivalence Theorem of \citet{kiefer1960equivalence}.
\begin{lem} For any subset $X \subseteq \R^d$, there exists a distribution $\gK_X$ supported on $X$, such that for any $\varepsilon > 0$, it holds that 
\begin{align}
\max_{\vx \in X} \vx^\top \left(\varepsilon \mI + \E_{\vy \sim \gK_X} \vy \vy^\top \right)^{-1} \vx \leq d . \label{eq:thm-KW}
\end{align}
Furthermore, if $X$ is a discrete set with finite cardinality, one can find a distribution such that the right-hand side of \eqref{eq:thm-KW} is relaxed to $2d$ in time  $\mathrm{poly}(\abs{X})$.
\end{lem}

We now describe the new \BatchedLinUCBKW algorithm. It is almost the same as \BatchedLinUCB, while the only difference is that at \autoref{line:batchedlinucb-10} of \autoref{algo:linucb}, letting $X = \{\vx_{t i} : i \in A_t\}$, we compute a distribution $\gK_X$ satisfying \eqref{eq:thm-KW} (up to the factor $2$ relaxation) and randomly select the action
\begin{align}
    i_t \sim \gopt(X) \defeq \gK_X.  \label{eq:def-gopt}
\end{align}
For completeness, a full description of \BatchedLinUCBKW is provided in \appref{app:blinucbkw}. 

We now prove the expected regret of \BatchedLinUCBKW as follows.
\begin{thm}\label{thm:BatchedLinUCBKW}
With probability at least $(1-\delta)$, the expected regret of \BatchedLinUCBKW is
\begin{align}
R^T_{\text{\BatchedLinUCBKW}} \leq O(d\sqrt{T \log (dKT/\delta)} \times  \log \log T). \notag
\end{align}
\end{thm}
We now prove \autoref{thm:BatchedLinUCBKW}. Note that the analysis for \BatchedLinUCB also applies to \BatchedLinUCBKW up to \eqref{eq:thm-blinucb-40}. Thus, we will focus on bounding $\E_{X \sim \gD_{k-1}}\max_{\vx \in X} \sqrt{\vx^\top \mLambda_{k-1}^{-1} \vx}$ while keeping in mind that $\mLambda_{k-1}^{-1}$ is a different quantity due to $\gopt$.

Similarly to \autoref{lem:exploration}, for each batch $k$ ($k < M$), with probability $(1 - \delta/T^2)$, we have that 
\begin{align}
    \mLambda_k \succcurlyeq \frac{\gT_{k}}{16} \left(\frac{ \ln T}{ \gT_{k}} \mI + \E_{X \sim \gD_{k}} \E_{\vx \sim \gopt(X)}[\vx \vx^\top]\right). \label{eq:lem-exploration-KW}
\end{align}
Assuming that \eqref{eq:lem-exploration-KW} holds for batch $(k-1)$, letting $\vx^*(X) = \argmax_{\vx \in X} \vx^\top \mLambda_{k-1}^{-1} \vx$, we have that
\begin{align}
&  \E_{X \sim \gD_{k-1}}\max_{\vx \in X} \sqrt{\vx^\top \mLambda_{k-1}^{-1} \vx}  = \E_{X \sim \gD_{k-1}} \sqrt{(\vx^*(X))^\top \mLambda_{k-1}^{-1} \vx^*(X)} \nonumber \\
&\qquad \qquad \leq  \sqrt{\E_{X \sim \gD_{k-1}} (\vx^*(X))^\top \mLambda_{k-1}^{-1} \vx^*(X)} = \sqrt{\Tr(\mLambda_{k-1}^{-1} \E_{X \sim \gD_{k-1}}  \vx^*(X) (\vx^*(X))^\top)},
 \label{eq:blinucbkw-1800}
\end{align}
where the inequality is by Jensen's inequality. By \autoref{lem:unif} and \eqref{eq:thm-KW} (up to the factor $2$ relaxation), we have that 
\begin{align}
 \vx^*(X) (\vx^*(X))^\top \preccurlyeq 2d \times \E_{\vy \sim \gopt(X)} \vy\vy^\top . \label{eq:blinucbkw-1900}
\end{align}
Combining \eqref{eq:blinucbkw-1800} and  \eqref{eq:blinucbkw-1900}, we have that
\begin{align}
\E_{X \sim \gD_{k-1}}\max_{\vx \in X} \sqrt{\vx^\top \mLambda_{k-1}^{-1} \vx} \leq \sqrt{2d \times \Tr(\mLambda_{k-1}^{-1} \E_{X \sim \gD_{k-1}}  \E_{\vy \sim \gopt(X)} \vy\vy^\top )} \leq 4\sqrt{2} d/ \sqrt{\gT_{k-1}},\label{eq:blinucbkw-2000}
\end{align}
where the last inequality is due to \eqref{eq:lem-exploration-KW}. Combining \eqref{eq:blinucbkw-2000} and \eqref{eq:thm-blinucb-40}, we have that with probability at least  $(1 - \delta \gT_k/T^2 - \delta/T^2)$, the expected regret incurred during batch $k$ ($k \geq 2$) is
\begin{align*}
R_k \leq 4\alpha \gT_k \cdot  4\sqrt{2}d / \sqrt{\gT_{k-1}} \leq 16\sqrt{2} \alpha d\sqrt{T}. 
\end{align*}
Using the similar argument as the analysis for \autoref{algo:linucb}, we have that with probability at least $(1-\delta)$, the expected regret of \BatchedLinUCBKW is at most 
\begin{align*}
R^T \leq O(d\sqrt{T \log (dKT/\delta)} \times \log \log T),
\end{align*}
proving \autoref{thm:BatchedLinUCBKW}.

\section{Distributional G-Optimal Design: Existence \& Parametric Forms} \label{sec:dist-opt-design}

We now work towards removing the extra $\sqrt{d}$ factor in the regret of \autoref{thm:BatchedLinUCBKW}, so as to achieve the optimal $\sqrt{dT}$-type regret. The high level idea is to use a difference sample policy other than uniform sampling over all (survived) candidate arms or the G-optimal-design-based $\gopt$.

Given a sample policy $\pi$ that maps any set of arms ($X \subseteq \R^d$) to a distribution in $\triangle_X$, we will be interested in its performance, defined as follows.
\begin{defn}[$\lambda$-distributional G-variation and information matrix] \label{defn:distributional-g-opt-design} For any distribution $\gD$ of the set of arms $X \subseteq \R^d$ and any sample policy $\pi$, we define the \emph{$\lambda$-distributional G-variation}, or \emph{$\lambda$-variation} for short ($\lambda > 0$), of $\pi$ over $\gD$ as 
\begin{align*}
    \val_\gD^{(\lambda)}(\pi)\defeq \E_{X \sim \gD}\max_{\vx \in X}  \vx^\top  \left(\lambda \mI + \info_\gD(\pi)\right)^{-1} \vx, 
\end{align*}
where we define the \emph{information matrix} by
\begin{align*}
  \info_\gD(\pi) \defeq \E_{X \sim \gD} \info_X(\pi), \qquad \text{where}~ \info_X(\pi) \defeq \E_{\vx \sim \pi(X)} \vx \vx^\top . 
\end{align*}
Since $\val_\gD^{(\lambda)}$ is non-increasing as $\lambda$ grows, when the limit exists, we also define
\begin{align}
\val_\gD^{(0)}(\pi) \defeq \lim_{\lambda \to 0^+} \val_\gD^{(\lambda)}(\pi), \label{val-gopt}
\end{align}
and set $\val_\gD^{(0)}(\pi) = + \infty$ otherwise.
\end{defn}
Indeed, the arguments in \autoref{sec:blinucb} imply the following lemma.
\begin{lem}\label{lem:variation-unif-optg}
For any distribution $\gD$ on the context vectors of the $K$ arms, we have that 
\begin{align}
\val_\gD^{(0)}(\Unif) \leq O(dK), ~~\text{and}~~ \val_\gD^{(0)}(\gopt) \leq O(d^2) . \label{eq:util-unif-gopt}
\end{align}
\end{lem}
In light of \autoref{lem:variation-unif-optg}, the question whether the regret of our algorithms can be improved to $O(\sqrt{dT \mathrm{poly}\log (KT/\delta)})$ boils down to whether one can find a sample policy $\pi$ such that the bounds in \eqref{eq:util-unif-gopt} are improved to $O(d) \times \mathrm{poly} \log d$. In this section, we will show that such policies not only exist, but also admit a succinct parametric form so that we can later study how to efficiently learn the relevant parameters. 

To better explain our results, we first define the following class of parameterized sample policies.
\begin{defn}[Argmax and mixed-argmax policies]  Suppose we are given a positive semi-definite matrix $\mV \succcurlyeq \vzero$. We define the associated \emph{argmax policy} by
\begin{align*}
    \pap_\mV(X) = \argmax_{\vx \in X} \vx^\top \mV \vx,
\end{align*}
where in the $\argmax$ operator, ties are broken in a deterministic manner.

In this subsection, we use $\gopt$ to denote a \emph{fixed} policy with respect to \eqref{eq:def-gopt} and satisfying \eqref{eq:thm-KW} (up to the factor $2$ relaxation). Suppose we are given a set $\gV = \{(p_i, \mV_i)\}_{i = 1}^n$ such that $p_i \ge 0$ and $p_1 + \cdots + p_n = 1$. We define the associated \emph{mixed-argmax policy} by
\begin{align*}
\pmap_\gV(X) = \begin{cases}
\gopt(X), & \text{with probability } 1/2, \\
\pap_{\mV_i}(X), & \text{with probability } p_i / 2.
\end{cases}
\end{align*}
\end{defn}

The following theorem states that for any $\gD$, there exists a good mixed-argmax policy with only $O(d \log d)$ argmax policies in the mixture.\footnote{Note that although the theorem only works for the uniform distribution over a multi-set, since the properties to be proved in the theorem statement do not truly depend on $\Gamma$, the theorem can be generalized to any distribution via a simple discretization argument.}
\begin{thm}\label{thm:argmax-policy}
Fix any distribution $\gD = \Unif(S)$ where $S = \{X_1, X_2, \dots, X_{\Gamma}\}$ (which may be a multi-set) and any $\lambda \in (0, 1)$. There exists a mixed-argmax policy with parameters $\gV = \{(p_i, \mV_i)\}_{i = 1}^n$ such that
\begin{enumerate}
    \item $n \leq 4 d \log d$;
    \item for all $i \in [n]$, $p_i \geq 1/d^3$ and $ d^{-1} \mI \preccurlyeq \mV_i \preccurlyeq \lambda^{-1} \mI $;
    \item  $\val_\gD^{(\lambda)}(\pmap_{\gV}) \leq O(d \log d)$.
\end{enumerate}
\end{thm}
\begin{proof}
We will assume $\Gamma > \lambda^{-1}$ without loss of generality, as the properties to be proved do not depend of $\Gamma$ and $S$ is a multi-set so that we can always duplicate the elements by finitely many times.

We prove the theorem constructively. We consider \autoref{algo:rarelysw}, which is very similar to the linear bandits algorithms in literature. For $N = \Theta(d^2 \log d)$, the algorithm creates $\Gamma N$ times steps, which includes $N$ blocks, each of which contains $\Gamma$ consecutive time steps. In each block, the $\Gamma$ sets of arms $X_1, \dots, X_\Gamma$ are sequentially presented. The algorithm then simulates the linear bandit algorithms, where at each time step, the arm with the maximum variance (according to the information matrix $\mW_n$) is selected. Inspired by the rarely switching algorithm for linear bandits \citep{abbasi2011improved}, the information matrix $\mW_n$ is only updated when its determinant doubles. This significantly reduces the number of updates and is crucial to upper bounding the number of individual argmax policies in the returned mixed-argmax policy. We refer to the consecutive time steps between two neighboring updates as a \emph{stage}.  Each of the information matrices in a stage corresponds to an individual argmax policy in the returned policy, and the corresponding probability weight is proportional to the length of the stage. The only exception is that we discard the stages that contain less than $\Gamma$ time steps (i.e., the ones that are shorter than a block).

\begin{algorithm}[t]
\caption{Algorithm for Computing a Distributional G-Optimal Design}
\label{algo:rarelysw}
\KwIn{A context set sequence $X_1, \dots, X_\Gamma$}
\KwOut{A mixed-argmax policy $\pi$}
$N \gets 2 d^2 \log d, \forall (i, j) \in [N] \times [\Gamma]: X_{(i-1)\Gamma + j} \gets X_j$\;
$\mU_0 \gets \lambda N\Gamma \mI + \frac{N}{2}\sum_{i = 1}^\Gamma \E_{\vx \sim \gopt(X_i)}[\vx \vx^\top] \succcurlyeq \mI,  n \gets 1, \tau_n \gets \emptyset, \mW_n = \mU_0$\;
\For{$t \gets 1,2, \dots, N\Gamma$}{
$\tau_n \gets \tau_n \cup \{t\}$\;
$\vx_t \gets \pap_{\mW_n^{-1}}(X_t) = \argmax_{\vx \in X_t} \vx^\top \mW_{n}^{-1} \vx$ \label{line:alg-rarelysw-5}\Comment*[r]{Ties are broken in a deterministic manner.}
$\mU_{t} \gets \mU_{t-  1} + \vx_t \vx_t^\top$\; \label{line:alg-rarelysw-6}
\If{$\det \mU_t > 2\det \mW_n$}{
$n \gets n + 1$,
$\tau_n \gets \emptyset$,
$\mW_n \gets \mU_t$\;
}
}
\textbf{for} all $i \in [n]$, \lIf{$|\tau_i| < \Gamma$}{$\tau_i \gets \emptyset$}\label{line:rarelysw-9}

\textbf{for} all $i \in [n]$, set $p_i = \abs{\tau_i} / \sum_j \abs{\tau_j}$\;

\KwRet{$\{(p_i, N \Gamma \mW_i^{-1}) : i \in [n] \text{~and~} p_i > 0\}$}
\end{algorithm}

\paragraph{Proof of Item (a).} Note that \begin{align}
    \mU_{N\Gamma} =  \mU_0 + \sum_{t=1}^{N\Gamma} \vx_t \vx_t^\top = \lambda N\Gamma \mI + \frac{N\Gamma}{2} \info_\gD(\gopt) + \sum_{t=1}^{N\Gamma} \vx_t \vx_t^\top. \label{eq:thm-ap-2700}
\end{align}
By \autoref{lem:unif} and \eqref{eq:thm-KW} (up to the factor $2$ relaxation), for all $t$, we have that 
\begin{align}
    \vx_t \vx_t^\top \preccurlyeq 2d \times  \E_{\vy \in \gopt(X_t)} \vy\vy^\top.  \label{eq:thm-ap-2800}
\end{align}
Combining \eqref{eq:thm-ap-2700} and \eqref{eq:thm-ap-2800}, we have that
\begin{align*}
 \mU_{N\Gamma} \preccurlyeq \lambda N\Gamma \mI + (1/2 + 2d) N\Gamma \times \info_\gD(\gopt) \preccurlyeq 4d \mU_0.
\end{align*}
Therefore, we have 
\begin{align} 
\det \mU_{N\Gamma} \leq \det (4d \mU_{0}) = d^{4d} \det\mU_0,\label{eq:thm-ap-3000}
\end{align}
and $n \leq \log (d^{4d}) = 4d \log d$.

\paragraph{Proof of Item (b).} Because we discard the stages whose lengths are less than $\Gamma$, for $p_i > 0$, we have that 
\begin{align*}
    p_i \geq \frac{\Gamma}{N\Gamma} \geq \frac{1}{d^3}
\end{align*}
for large enough $d$.

For each $\mW_i$, we have $\mW_i \succcurlyeq \mU_0 \succcurlyeq \lambda N \Gamma \mI$, and $\mW_i \preccurlyeq 3 N\Gamma \mI$. Since $\mV_i = N \Gamma \mW_i^{-1}$, we have that $ d^{-1}  \mI \preccurlyeq \mV_i \preccurlyeq \lambda^{-1} \mI$.

\paragraph{Proof of Item (c).}
We finally upper bound the $\lambda$-variation of the returned policy $\pi = \pmap_{\gV}$. Note that
\begin{align}
\val^{(\lambda)}_\gD(\pi) &= \E_{X \sim \gD}[\max_{\vx \in X}  \vx^\top (\lambda\mI + \E_{X \sim \gD} \E_{\vx \sim \pi(X)} \vx \vx^\top)^{-1} \vx] \notag \\ 
&= \sum_{t = 1}^{N\Gamma} \max_{\vx \in X_t}  \vx^\top (N\Gamma (\lambda \mI + \info_\gD(\pi)))^{-1} \vx \notag \\ 
&= \sum_{i = 1}^n \sum_{t \in \tau_i} \max_{\vx \in X_t}  \vx^\top (N\Gamma (\lambda \mI + \info_\gD(\pi)))^{-1} \vx  + \sum_{t \in \gB} \max_{\vx \in X_t}  \vx^\top (N\Gamma (\lambda \mI +\info_\gD(\pi)))^{-1} \vx, \label{eq:thm-ap-3300}
\end{align}
where we let $\gB$ be the set of time steps that are discarded in \autoref{line:rarelysw-9} of \autoref{algo:rarelysw}.

It remains to show that both terms are $O(d \log d)$. For the second term, we have 
\begin{align}
    \sum_{t \in \gB} \max_{\vx \in X_t}  \vx^\top (N\Gamma (\lambda\mI +\info_\gD(\pi)))^{-1} \vx &= \frac{1}{N\Gamma} \sum_{t \in \gB} \max_{\vx \in X_t}  \vx^\top (\lambda \mI + \info_\gD(\pi))^{-1} \vx  \nonumber \\
    &\le \frac{2}{N\Gamma} \sum_{t \in \gB} \max_{\vx \in X_t}  \vx^\top (\lambda\mI + \info_\gD(\gopt))^{-1} \vx, \label{eq:thm-ap-3100}
\end{align}
where the inequality is because by definition of a mixed-argmax policy, with probability $1/2$, $\gopt$ is invoked, and therefore
\begin{align}
    \info_\gD(\pi) = \E_{X \sim \gD, \vx \sim \pi(X)}\vx \vx^\top \succcurlyeq \E_{X \sim \gD} \frac{1}{2} \times\E_{\vx \sim \gopt(X)} \vx \vx^\top. \notag 
\end{align}
Continuing with \eqref{eq:thm-ap-3100}, since $\gB$ contains at most $n$ stages that are shorter than a block, therefore, we have that
\begin{align}
  &\frac{2}{N\Gamma} \sum_{t \in \gB} \max_{\vx \in X_t}  \vx^\top (\lambda\mI + \info_\gD(\gopt))^{-1} \vx \leq \frac{2}{N\Gamma} \times n \times \sum_{t=1}^{\Gamma}    \max_{\vx \in X_t}  \vx^\top (\lambda\mI + \info_\gD(\gopt))^{-1} \vx \nonumber\\
 &\qquad = \frac{2n}{N} \E_{X \sim \gD}  \max_{\vx \in X_t}  \vx^\top (\lambda\mI + \info_\gD(\gopt))^{-1} \vx
  = \frac{2n}{N} \val_{\gD}^{(\lambda)}(\gopt) \leq \frac{2n}{N} \times O(d^2) \leq O(d \log d), 
\end{align}
where the second inequality is due to \eqref{val-gopt}, \eqref{eq:util-unif-gopt}, and the monotonicity of $\val_{\gD}^{(\lambda)}$.

For the first term in \eqref{eq:thm-ap-3300}, we claim that 
\begin{align}
 \info_\gD(\pi) \succcurlyeq  \frac{1}{4N\Gamma} \sum_{t = 1}^{N\Gamma} \vx_t \vx_t^\top, \label{eq:thm-ap-4000}
\end{align}
which will be established at the end of this proof. Once we have \eqref{eq:thm-ap-4000}, also noting that $ \info_\gD(\pi) \succcurlyeq (1/2) \info_\gD(\gopt)$ because of the $1/2$ portion of $\gopt$ in the definition of the mixed-argmax policy, we get that
\begin{align}
     \lambda \mI + \info_\gD(\pi) &\succcurlyeq \lambda \mI  + \frac{1}{2}( \frac{1}{2}  \info_\gD(\gopt) + \frac{1}{4N\Gamma} \sum_{t = 1}^{N\Gamma} \vx_t \vx_t^\top)
     \succcurlyeq \frac{1}{8 N\Gamma} \mU_{N\Gamma} \succcurlyeq \frac{1}{8 N\Gamma} \mW_n.   \label{eq:thm-ap-4200}
\end{align}

Therefore, 
\begin{align}
    &  \sum_{i = 1}^n \sum_{t \in \tau_i} \max_{\vx \in X_t}  \vx^\top (N\Gamma (\lambda \mI + \info_\gD(\pi)))^{-1} \vx  \le  8 \sum_{i = 1}^n \sum_{t \in \tau_i} \max_{\vx \in X_t}\vx^\top \mW_n^{-1} \vx \notag \\
    & \qquad \qquad \qquad \qquad \le 8 \sum_{i = 1}^n \sum_{t \in \tau_i} \max_{\vx \in X_t}\vx^\top  \mW_i^{-1} \vx 
    = 8 \sum_{i = 1}^n \sum_{t \in \tau_i} \vx_t^\top  \mW_i^{-1}
    \vx_t  \notag \\ 
    & \qquad \qquad \qquad \qquad  \le 16 \sum_{i = 1}^n  \sum_{t \in \tau_i} \vx_t^\top \mU_t^{-1} \vx_t \le 16 \sum_{t = 1}^{N\Gamma} \vx_t^\top \mU_t^{-1} \vx_t  \label{eq:thm-ap-4500} \\ 
     & \qquad \qquad \qquad \qquad \le 32 \ln \frac{\det \mU_{N\Gamma}}{\det \mU_0}  \le  O(d \log d).   \label{eq:thm-ap-4600}
\end{align}
where the first inequality in \eqref{eq:thm-ap-4500} is by \autoref{lem:comp}, the first inequality in \eqref{eq:thm-ap-4600} is by the elliptical potential lemma (\autoref{lem:ellip}),\footnote{This is a generalized version and we invoke the lemma by letting $\mX_t$ in the lemma statement be $\vx_t \vx_t^\top$ and letting $\mLambda_t$ in the lemma statement be $\mU_t$. Note that $\mLambda_0 = \mU_0 \succcurlyeq \mI$ so that $\Tr(\mX_t \mLambda_0^{-1}) \leq 1$ is satisfied.} and the second inequality in \eqref{eq:thm-ap-4600} is due to \eqref{eq:thm-ap-3000}.

It remains to establish \eqref{eq:thm-ap-4000}. Note that
\begin{align}
    \info_\gD(\pi) 
    &=  \frac{1}{2} \info_\gD(\gopt) +  \frac{1}{2} \sum_{i = 1}^n \frac{\abs{\tau_i}}{\abs{\tau_1} + \cdots + \abs{\tau_n}} \E_{X \sim \gD} \info_X(\pap_{\mW_i^{-1}})   \nonumber\\
    &\succcurlyeq  \frac{1}{2} \info_\gD(\gopt) +  \frac{1}{2} \sum_{i = 1}^n \frac{\abs{\tau_i}}{\abs{\tau_1} + \cdots + \abs{\tau_n}} \frac{1}{2\abs{\tau_i}}\sum_{t \in \tau_i} \vx_t \vx_t^\top  \nonumber \\ 
    &=  \frac{1}{2} \info_\gD(\gopt) + \frac{1}{4} \frac{1}{N\Gamma - \abs{\gB}} \sum_{i = 1}^n \sum_{t \in \tau_i} \vx_t \vx_t^\top. \label{eq:thm-ap-4800}
\end{align}
By \eqref{eq:thm-ap-2800}, we have
\begin{align}
\info_\gD(\gopt) = \frac{1}{n\Gamma} \sum_{t=1}^{\Gamma} n \times \E_{\vx \sim \gopt(X_t)} \vx \vx^\top \succcurlyeq \frac{1}{n\Gamma} \sum_{t\in \gB} \frac{1}{2d} \vx_t \vx_t^\top. \notag
\end{align}
Therefore, continuing with \eqref{eq:thm-ap-4800}, we have that
\begin{align}
    \info_\gD(\pi) &
    \succcurlyeq \frac{1}{2 n d\Gamma} \sum_{t\in \gB}  \vx_t \vx_t^\top +  \frac{1}{4 N \Gamma} \sum_{i = 1}^n \sum_{t \in \tau_i} \vx_t \vx_t^\top \nonumber \\
    &\succcurlyeq \frac{1}{4 N \Gamma} \sum_{t\in \gB}  \vx_t \vx_t^\top +  \frac{1}{4 N \Gamma} \sum_{i = 1}^n \sum_{t \in \tau_i} \vx_t \vx_t^\top = \frac{1}{4 N\Gamma} \sum_{t = 1}^{N \Gamma} \vx_t \vx_t^\top,  \label{eq:thm-ap-5100}
\end{align}
which concludes the proof of the theorem.
\end{proof}

\subsection{The Mixed-Softmax Policies with More Robustness}

To make the sample policy learnable, instead of the mixed-argmax policies, we will deal with the more robust mixed-softmax policies. To define this class of policies, we first define the softmax function as a distribution such that
\begin{align}
\softmax_\alpha(s_1, \dots, s_k) = i \quad \text{with probability} \quad \frac{s_i^{\alpha}}{s_1^\alpha + \cdots + s_k^\alpha}, \notag
\end{align}
where we assume that $s_i \geq 0$ for all $i \in [k]$.

It is easy to check the following fact.
\begin{fact}\label{fact:softmax}
Suppose $\alpha \geq \log k$, then 
\begin{align}
\E_{i \sim \softmax_\alpha(s_1, \dots, s_k)} [s_i] \geq \frac14 \times \max\{s_1, \dots, s_k\}. \notag
\end{align}
\end{fact}
\begin{proof}
Let $i^*$ be an index that maximizes $s_{i}$. Note that for all $j$ such that $s_j \leq (1/2) \times s_{i^*}$, the probability mass that softmax put for $j$ is at most $(1/k)$ of that for $i^*$. Therefore, 
\begin{align}
    \Pr_{i \sim \softmax_\alpha(s_1, \dots, s_k)} [s_i \geq \frac12 \times  s_{i^*}] \geq \frac{1}{2}, \notag
\end{align}
and the fact follows.
\end{proof}

We now define the class of mixed-softmax policies.
\begin{defn}[Softmax and mixed-softmax policies] \label{defn:sp} Fix $\alpha = \log K$ (where $K$ is the number of arms per time step). Suppose we are given a positive semi-definite matrix $\mM \succcurlyeq \vzero$. We define the \emph{softmax policy} \begin{align*}
    \psp_\mM(X) = \vx_i, \qquad  \text{where} \quad X = \{\vx_1, \dots, \vx_k\}, k \leq K, \text{~and~} i \sim \softmax_\alpha(\vx_1^\top \mM \vx_1, \dots, \vx_k^\top \mM \vx_k). 
\end{align*}

Suppose we are given a set $\gM = \{(p_i, \mM_i)\}_{i = 1}^n$ such that $p_i \geq 0$ and $p_1 + \cdots + p_n = 1$. We define the \emph{mixed-softmax policy}
\begin{align*}
\pmsp_\gM(X) = \begin{cases}
\gopt(X), & \text{with probability } 1/2, \\
\psp_{\mM_i}(X), & \text{with probability } p_i / 2.
\end{cases}
\end{align*}
\end{defn}

Similarly to \autoref{thm:argmax-policy}, we prove the following theorem on the existence of good mixed-softmax policies.
\begin{thm}\label{thm:ms-policy}
Fix any distribution $\gD = \Unif(S)$ where $S = \{X_1, X_2, \dots, X_{\Gamma}\}$ (which may be a multi-set) and any $\lambda \in (0, 1)$. There exists a mixed-softmax policy $\pmsp_\gM$ with parameters $\gM = \{(p_i, \mM_i)\}_{i = 1}^n$ such that
\begin{enumerate}
    \item $n \leq 4 d \log d$;
    \item for all $i \in [n]$, $p_i \geq 1/d^3$ and $ d^{-1} \mI \preccurlyeq \mM_i \preccurlyeq \lambda^{-1} \mI $;
    \item  $\val_\gD^{(\lambda)}(\pmsp_{\gM}) \leq O(d \log d)$.
\end{enumerate}
\end{thm}

The proof of \autoref{thm:ms-policy} is very similar to that of \autoref{thm:argmax-policy}. Here we only point out the differences as follows.

First, at \autoref{line:alg-rarelysw-5} of \autoref{algo:rarelysw}, we let $\mX_t \gets \E_{\vx \sim \psp_{W_n^{-1}}(X_t)} \vx \vx^\top$, and at Line~\ref{line:alg-rarelysw-6}, we let $\mU_{t} \gets \mU_{t-  1} + \mX_t$. Note that $\Tr(\mX_t) \leq 1$. Let $\gM$ be the output of the algorithm.

The proof of Items (a) and (b) remains the same except for the occurrences of $\vx_t \vx_t^\top$ are replaced by $\mX_t$ in \eqref{eq:thm-ap-2700} and \eqref{eq:thm-ap-2800}.

For the proof of Item (c), let $\pi = \pmsp_{\gM}$, we still get \eqref{eq:thm-ap-3300}, and the second term of \eqref{eq:thm-ap-3300} is bounded by the same way. For the first term,  replacing $\vx_t\vx_t^\top$ by $\mX_t$ in \eqref{eq:thm-ap-4000} (and its proof from \eqref{eq:thm-ap-4800} to \eqref{eq:thm-ap-5100}), we still get \eqref{eq:thm-ap-4200}. Therefore, 
\begin{align}
    &  \sum_{i = 1}^n \sum_{t \in \tau_i} \max_{\vx \in X_t}  \vx^\top ( N\Gamma (\lambda \mI + \info_\gD(\pi)))^{-1} \vx  \le  8 \sum_{i = 1}^n \sum_{t \in \tau_i} \max_{\vx \in X_t}\vx^\top \mW_n^{-1} \vx \nonumber  \\ 
    & \quad \le 8 \sum_{i = 1}^n \sum_{t \in \tau_i} \max_{\vx \in X_t}\vx^\top  \mW_i^{-1} \vx 
    \leq 32 \sum_{i = 1}^n \sum_{t \in \tau_i} \E_{\vx \sim \pmsp_{\mW_i^{-1}}(X_t)}\vx^\top  \mW_i^{-1}
    \vx      = 32 \sum_{i = 1}^n \sum_{t \in \tau_i} \Tr(\mW_i^{-1} \mX_t),  \label{eq:thm-sp-5700} 
\end{align}
where the third inequality in \eqref{eq:thm-sp-5700}  is by \autoref{fact:softmax}. Again, by \autoref{lem:comp} and the elliptical potential lemma (\autoref{lem:ellip}), we have that
\begin{align}
   & 32 \sum_{i = 1}^n \sum_{t \in \tau_i}  \Tr(\mW_i^{-1} \mX_t) \leq 64 \sum_{i = 1}^n \sum_{t \in \tau_i} \Tr(\mU_t^{-1} \mX_t) \nonumber  \\
    & \qquad \qquad \qquad \qquad \qquad\leq 64 \sum_{t = 1}^{N\Gamma}  \Tr(\mU_t^{-1} \mX_t) \leq 128 \ln  \frac{\det \mU_{N\Gamma}}{\det \mU_0}  \le  O(d \log d). \notag
\end{align}
Combining the bounds on both terms of \eqref{eq:thm-ap-3300}, we prove the theorem.

\section{Learning the Distributional G-Optimal Design} \label{sec:learn-design}

In this section, we present an algorithm to learn a good mixed-softmax policy using only $\mathrm{poly}(d) \log \delta^{-1}$ samples with success probability at least $(1 - \delta)$.

\paragraph{The Natural Idea and its Counterexample.} The most natural idea is to first draw $\gamma$ independent samples $X_1, \dots, X_\gamma \sim \gD$ and form an empirical distribution $\gS = \Unif\{X_1, \dots, X_\gamma\}$, learn a good policy $\pi$ for $\gS$ according to \autoref{thm:ms-policy}, and hope that $\pi$ also works well for $\gD$ (i.e., $\pi$ generalizes to the true distribution). Unfortunately, such an approach is unlikely to work. Below we illustrate an example where, even when the number of samples $\gamma$ is very large, a good policy for $\gS$ still fails to generalize to $\gD$ with significant probability.

Let $\{\ve_i\}_{i=1}^d$ be the set of canonical basis, and $\varepsilon > 0$ be a parameter to be determined later. Let $Y_1 = \{\ve_1\}$ and $Y_i = \{\sqrt{1-\varepsilon^2} \ve_i + \varepsilon \ve_1, \ve_i\}$ for $i \in \{2, 3, \dots, d\}$. Consider $\gD$ supported on $\{Y_1, \dots, Y_d\}$ the probability mass for $Y_1$  is $1/(d\gamma)$ and the probability for $Y_i$ ($i \geq 2$) is $q = (1 - 1/(d\gamma))/(d-1)$. If we make $\gamma$ independent samples $X_1, \dots, X_\gamma \sim \gD$, with probability $\Omega(1/d)$, we will see $Y_1$ once among the samples, and the probability mass of $Y_1$ in $\gS$ becomes $1/\gamma$, which is $d$ times its true probability mass. Due to this discrepancy, we will show that a good sample policy for the empirical distribution $\gS$ does not work as well on true distribution $\gD$.

We consider the sample policy $\pi$ such that $\pi(X) = \ve_i$ when $X = Y_i$. When the event above happens, we have that $\info_\gS(\pi)=  \mathrm{diag}(1/\gamma, p_2, \dots, p_d)$ where $p_i$ is the probability mass for $Y_i$ in $\gS$ (for $i \geq 2$). When $\varepsilon = \sqrt{d/\gamma}$, we can verify that $\pi$ is a good policy for the empirical distribution $\gS$ since 
\begin{align}
\val_\gS^{(0)}(\pi) = \E_{X \sim \gS} \max_{\vx \in X} \vx^\top \info_\gS(\pi)^{-1} \vx =  \frac1\gamma \cdot \gamma + \sum_{i=2}^d p_i \cdot \max\{\varepsilon^2 \gamma + (1 - \varepsilon^2) \cdot \frac{1}{p_i}, \frac{1}{p_i}\} \leq O(d). \notag 
\end{align}

However, for the true distribution $\gD$, we have that $\info_\gD(\pi)=  \mathrm{diag}(1/(d\gamma), q, \dots, q)$, and for any $\lambda \in [0, 1/(d\gamma))$, it holds that
\begin{align}
& \val_\gD^{(\lambda)}(\pi) = \E_{X \sim \gD} \max_{\vx \in X} \vx^\top (\lambda \mI + \info_\gD(\pi))^{-1} \vx \notag \\
&\quad =  \frac{1}{d\gamma} \cdot \frac{1}{\lambda + 1/(d\gamma)} + (1 - \frac{1}{d\gamma}) \cdot \max\{\varepsilon^2 \cdot \frac{1}{\lambda + 1/(d\gamma)} + (1-\varepsilon^2) \cdot \frac1{\lambda + q}, \frac1{\lambda +q}\} \geq \Omega(d^2). \notag 
\end{align}

Note that in this example, the only constraint for $\gamma$ is that $1/(d\gamma) > \lambda \Leftrightarrow \gamma < 1/(d\lambda)$. Therefore, we have illustrated that, even when $\gamma$ is greater than an arbitrary polynomial of $d$,  with probability $\Omega(1/d)$, a good policy for the empirical distribution $\gS$ does not generalize to the true distribution $\gD$.\footnote{Although in our later algorithm, we only learn a policy with small \emph{$\lambda$-deviation} as defined in \eqref{eq:def-root-variation}, however, one can also verify that the $\lambda$-deviation of $\pi$ over $\gD$ in this counterexample is also high.} By adding more dimensions, we can even strengthen this counterexample so that the failure probability becomes $(1 - o(1))$. Using similar tricks, we can also show that a good mixed-softmax policy does not generalize well.

\paragraph{Our Algorithm: {\sc CoreLearning}.} The key message from the counterexample above is that if a context direction in $\sR^d$ appears with tiny probability in $\gD$, a limited amount of samples might greatly change its probability in the empirical distribution $\gS$, and fail the generalization argument. To address this issue, the idea of our new algorithm is to prune these infrequent context directions, learn a mixed-softmax policy over the remaining ``core'' directions, and finally argue that the infrequent directions can be properly handled by the $\gopt$ component in the mixed-softmax policy.

In light of this idea, we propose {\sc CoreLearning} (\autoref{algo:learn-design}). In this algorithm, instead of directly learning the policy from the whole set of samples, we first find a large enough core set $C$ at \autoref{loc:learn-2}, and then learn the mixed-softmax policy only using the samples in $C$. The key property of the core is specified by \eqref{eq:core1}, which is a technical realization of our pruning idea. The property requires that every direction in $C$ should be well explored by the $\gopt$ policy and \emph{only} the context vectors within $C$. To see how the core set helps to resolve the issue in our counterexample, we note that the infrequent set $Y_1$ is the main trouble-maker. However, even if $Y_1$ happens to appear among the samples $\{X_1, \dots, X_\gamma\}$, it will not be included in the core since its corresponding variation $\max_{\vy \in Y_1}\vy^\top (\lambda \mI + \info_{\Unif(C)}(\gopt))^{-1} \vy \geq (\lambda + 1/\gamma)^{-1} > d^c$ when $\lambda$ is sufficiently small and $\gamma \gg d^c$. Therefore, {\sc CoreLearning} will learn a sample policy with $Y_1$ pruned away, and void our counterexample.

While the core set property \eqref{eq:core1} is much desirable, even whether such a core set with cardinality constraint \eqref{eq:core2} exists is not obvious. In \autoref{sec:core}, we prove \autoref{lem:core} to show its existence, and provide an efficient algorithm {\sc CoreIdentification} to find one. The analysis of \autoref{algo:learn-design} also relies on a few uniform concentration inequalities (\autoref{lem:uniform-con-1} and \autoref{lem:uniform-con-2}) which are proved later in \autoref{sec:uniform-con}.

\begin{algorithm}[t]
\caption{{\sc CoreLearning} for the Distributional G-Optimal Design}
\label{algo:learn-design}
\KwIn{$\lambda \in (\exp(-d), 1)$, and $S = \{X_1, \dots, X_\gamma\}$}
\KwOut{A mixed-softmax policy $\pi$}
Set constant $c = 6$\;
Find a core $C \subseteq S = \{X_1, \dots, X_\gamma\}$ (using {\sc CoreIdentification} (\autoref{algo:core}), see \autoref{lem:core}) such that \label{loc:learn-2} 
\begin{align}
   & \max_{X \in C} \max_{\vx \in X} \{\vx^\top (\lambda \mI + \info_{\Unif(C)}(\gopt))^{-1} \vx\} \leq d^{c}, \qquad\qquad\qquad\qquad \label{eq:core1}\\
\text{and}\qquad\qquad\qquad\qquad & \qquad \qquad   \frac{\abs{C}}{\gamma} \geq 1 - O(d^{3-c} \log \lambda^{-1}), \label{eq:core2}
\end{align}
which is at least $1/2$ for sufficiently large $d$\;
Compute the mixed-softmax policy $\pi$ for the samples in $C$ (according to \autoref{thm:ms-policy}) and return $\pi$\;
\end{algorithm}

For now, assuming the lemmas introduced above, we prove the following main theorem of this section (the guarantee for \autoref{algo:learn-design}).

\begin{thm}\label{thm:learn-design}
Suppose that $\lambda \in (\exp(-d), 1)$. Let $X_1, \dots, X_\gamma \sim \gD$ be {\it i.i.d.}\ drawn from the distribution $\gD$. Let $\pi$ be the returned policy of \autoref{algo:learn-design}. We have that
\begin{align}
\Pr[\widetilde{\val}_\gD^{(\lambda)}(\pi) \leq O(\sqrt{d\log d})] &\geq   1 - \exp(O(d^4 \log^2 d) - \gamma d^{-2c}  \cdot 2^{-16}) \notag \\
&=  1 - \exp(O(d^4 \log^2 d) - \gamma d^{-12}  \cdot 2^{-16}), \notag 
\end{align}
where we define the \emph{$\lambda$-deviation} of $\pi$ over $\gD$ by
\begin{align}
\widetilde{\val}_\gD^{(\lambda)}(\pi) \defeq \E_{X \sim \gD} \sqrt{\max_{\vx \in X} \{\vx^\top (\lambda \mI + \info_\gD(\pi))^{-1} \vx\}}. \label{eq:def-root-variation}
\end{align}
\end{thm}

Note that we are only able to provide the upper bound for $\widetilde{\val}_\gD^{(\lambda)}(\pi)$ instead of $\val_\gD^{(\lambda)}(\pi)$. However, this is still enough for our linear bandit application.

We now prove \autoref{thm:learn-design}. For notation convenience, we define $\gS \defeq \Unif(S)$, $\gC \defeq \Unif(C)$, and we define the mollifier 
\begin{align}
    \varphi_\beta(x) \defeq \begin{cases}
    1, & \text{when~} x \le \beta, \\
    \frac{2 \beta - x}{\beta} , & \text{when~} \beta \le x \le 2 \beta, \\ 
    0, & \text{when~} x > 2 \beta. 
    \end{cases} \notag 
\end{align}
which is a continuous surrogate of the indicator function $\ind[x \leq \beta]$.

We now condition on the successful events of the uniform convergence lemmas (\autoref{lem:uniform-con-1} and \autoref{lem:uniform-con-2}), which, by a union bound, happens with probability 
\[
1 - \exp(O(d^3 \log d \log (d \lambda^{-1})) - \gamma d^{-2c}  \cdot 2^{-16}) \geq 1 - \exp(O(d^4 \log^2 d) - \gamma d^{-2c}  \cdot 2^{-16}).
\]
Then, the proof of \autoref{thm:learn-design} consists of the following four steps.

\paragraph{Step \uppercase\expandafter{\romannumeral 1}: Lower Bounding the Information Matrix.}  The goal of this step is to establish \eqref{eq:sketch-1}. Let $\mU = \lambda \mI + \info_{\gC}(\pi)$. Note that $\lambda \mI \preccurlyeq \mU \preccurlyeq (1 + \lambda) \mI$. By the successful event in \eqref{eq:core-eve1} of \autoref{lem:uniform-con-2} (letting $\mW = \mU$), we have that
\begin{align}
\E_{X \sim \gD} \varphi_{2d^c} (\max_{\vx \in X}\{\vx^\top \mU^{-1} \vx\}) \cdot \info_X(\pi) \succcurlyeq \frac{1}{\gamma}\sum_{i=1}^\gamma \varphi_{2d^c}(\max_{\vx \in X_i} \{\vx^\top \mU^{-1} \vx\}) \cdot \info_{X_i}(\pi) -\frac14 \mU. \label{eq:sketch-1a}
\end{align}

Since $\info_{\gD}(\pi) \succcurlyeq \E_{X \sim \gD} \varphi_{2d^c} (\max_{\vx \in X}\{\vx^\top \mU^{-1} \vx\}) \cdot \info_X(\pi)$ and $\frac{1}{\gamma}\sum_{i=1}^\gamma \varphi_{2d^c}(\max_{\vx \in X_i} \{\vx^\top \mU^{-1} \vx\}) \cdot \info_{X_i}(\pi) \succcurlyeq \frac12 \info_\gC(\pi)$, \eqref{eq:sketch-1a} implies  that
\begin{align}
\lambda \mI + \info_{\gD}(\pi)& \succcurlyeq \lambda \mI + \frac12 \info_\gC(\pi) - \frac14 \mU \succcurlyeq \frac14 (\lambda \mI + \info_\gC (\pi)) \notag \\
& \succcurlyeq \frac14 \left(\lambda \mI + \info_\gS (\pi) - \frac1\gamma \sum_{X_i \in S \setminus C} \info_{X_i}(\pi)\right)  \succcurlyeq \frac14 \left(\lambda \mI + \info_\gS (\pi) - \frac{d}{\gamma} \sum_{X_i \in S \setminus C} \info_{X_i}(\gopt)\right) \notag \\
& \succcurlyeq \frac18 (\lambda \mI + \info_\gS (\pi)),
\label{eq:sketch-1}
\end{align}
where the last inequality is for $c \geq 6$.

\paragraph{Step \uppercase\expandafter{\romannumeral 2}: Upper Bounding the Variation in the ``Core Directions''.}   Let $\mW = \lambda \mI + \info_{\gS}(\pi) \succcurlyeq \frac12 (\lambda \mI + \info_{\gC}(\pi) )$. The goal of this step is to establish \eqref{eq:sketch-2}. By the successful event in \eqref{eq:core-eve2} of \autoref{lem:uniform-con-1}, we have that
\begin{align}
& \E_{X \sim \gD} \varphi_{4d^c}(\max_{\vx \in X}\{\vx^\top \mW^{-1} \vx\}) \cdot \sqrt{ \max_{\vx \in X}\{\vx^\top \mW^{-1} \vx\}} \notag \\
& \qquad\qquad\qquad\qquad\qquad\qquad \leq d + \frac1\gamma \sum_{i=1}^{\gamma} \varphi_{4d^c}(\max_{\vx \in X_i}\{\vx^\top \mW^{-1} \vx\}) \cdot \sqrt{\max_{\vx \in X_i}\{\vx^\top \mW^{-1} \vx\} }. \notag
\end{align}
This implies that
\begin{align}
\E_{X \sim \gD} \varphi_{4d^c}(\max_{\vx \in X}\{\vx^\top \mW^{-1} \vx\}) \cdot \sqrt{\max_{\vx \in X}\{\vx^\top \mW^{-1} \vx\}} \leq d +  \frac1\gamma \sum_{i=1}^{\gamma}  \sqrt{\max_{\vx \in X_i}\{\vx^\top \mW^{-1} \vx\}} . \label{eq:sketch-2aa}
\end{align}
Let $\zeta = 1 - \abs{C}/\abs{S} = 1 - \abs{C}/\gamma \leq O(d^{3-c} \log (1/\lambda))$. Note that 
\begin{align}
 &\quad \frac1\gamma \sum_{i=1}^{\gamma}  \sqrt{\max_{\vx \in X_i}\{\vx^\top \mW^{-1} \vx\}} \nonumber \\
 &\le \frac1\gamma \sum_{X_i \in C} \sqrt{\max_{\vx \in X_i}\{2 \vx^\top (\lambda \mI + \info_{\gC}(\pi))^{-1} \vx\}} + \frac1\gamma \sum_{X_i \in S \setminus C} \sqrt{\max_{\vx \in X_i}\{\vx^\top (\lambda \mI + (\zeta / 2) \info_{\Unif(S \setminus C)}(\gopt))^{-1} \vx\}}. \label{eq:sketch-2a}
\end{align} 
For the first term in \eqref{eq:sketch-2a}, by the guarantee of \autoref{thm:ms-policy}, we have 
\begin{align}
    \frac1\gamma \sum_{X_i \in C} \sqrt{\max_{\vx \in X_i}\{2 \vx^\top (\lambda \mI + \info_{\gC}(\pi))^{-1} \vx\}} \le \sqrt{2 \rval_\gC(\pi)} \le O( \sqrt{d \log d} ). \label{eq:sketch-2b}
\end{align}
For the second term in \eqref{eq:sketch-2a}, by the variation bound for $\gopt$ (\autoref{lem:variation-unif-optg}), we have that
\begin{align}
&\frac1\gamma \sum_{X_i \in S \setminus C} \sqrt{\max_{\vx \in X_i}\{\vx^\top (\lambda \mI + (\zeta / 2) \info_{\Unif(S \setminus C)}(\gopt))^{-1} \vx\}}  \nonumber \\
& \qquad \leq \zeta \sqrt{\frac{1}{\zeta\gamma}  \sum_{X_i \in S \setminus C} {\max_{\vx \in X_i}\{\vx^\top (\lambda \mI + (\zeta / 2) \info_{\Unif(S \setminus C)}(\gopt))^{-1} \vx\}}} \leq O(\sqrt{\zeta d^2}) \leq O(\sqrt{d}),
\label{eq:sketch-2c}
\end{align}
where the first inequality is Jensen and the last inequality is for $c \geq 5$.

Combining \eqref{eq:sketch-2aa}, \eqref{eq:sketch-2a}, \eqref{eq:sketch-2b}, \eqref{eq:sketch-2c}, we have that
\begin{align}
    \E_{X \sim \gD} \varphi_{4d^c}(\max_{\vx \in X}\{\vx^\top \mW^{-1} \vx\}) \cdot \sqrt{\max_{\vx \in X}\{\vx^\top \mW^{-1} \vx\}} \leq O(\sqrt{d \log d}). \label{eq:sketch-2}
\end{align}

\paragraph{Step \uppercase\expandafter{\romannumeral 3}: Upper Bounding the Variation  in the ``Infrequent Directions''.} The goal of this step is to establish \eqref{eq:sketch-3}. By the successful event in \eqref{eq:core-eve3} of \autoref{lem:uniform-con-1}, we have that
\begin{align}
\E_{X \sim \gD}\varphi_{4d^c}(\max_{\vx \in X}\{\vx^\top \mW^{-1} \vx\}) \geq -d^{-1} + \frac1\gamma \sum_{i=1}^{\gamma}\varphi_{4d^c}(\max_{\vx \in X_i}\{\vx^\top \mW^{-1} \vx\}). \notag 
\end{align}
This implies that
\begin{align}
\E_{X \sim \gD}\varphi_{4d^c}(\max_{\vx \in X}\{\vx^\top \mW^{-1} \vx\}) & \geq - d^{-1} + \frac1\gamma \sum_{i=1}^{\gamma} \varphi_{2d^c}(\max_{\vx \in X_i}\{\vx^\top \mU^{-1} \vx\}) \notag \\
&\geq 1 - d^{-1} - O(d^{3-c} \log (1/\lambda)) \geq  1- O(d^{-1}), \notag
\end{align}
where the last inequality is for $c \geq 5$.
Let $\tau_X = 1 -\varphi_{4d^c}(\max_{\vx \in X}\{\vx^\top \mW^{-1} \vx\})$. We have that $\E_{X \sim \gD} \tau_X \leq O(d^{-1})$. Note that, 
\begin{align}
&\quad \E_{X \sim \gD} \tau_X \sqrt{\max_{\vx \in X}\{\vx^\top \mW^{-1} \vx\}} = \E_{X \sim \gD} \tau_X \sqrt{\max_{\vx \in X}\{2 \vx^\top (\lambda \mI +  \E_{X \sim \gD} \tau_X \info_X(\gopt))^{-1} \vx\}} \notag \\
&= \E_{X \sim \gD} \sqrt{\tau_X} \cdot \sqrt{\tau_X} \sqrt{\frac{1}{\E_X \tau_X} \cdot \max_{\vx \in X}\{2 \vx^\top (\frac{\lambda}{\E_X \tau_X} \mI +  \E_{X \sim \gD} \frac{\tau_X}{\E_X \tau_X} \info_X(\gopt))^{-1} \vx\}} \notag \\
&\leq \sqrt{\E_{X \sim \gD} \tau_X} \cdot \sqrt{ \E_{X \sim \gD} \frac{\tau_X}{\E_X \tau_X} \cdot \max_{\vx \in X}\{2 \vx^\top (\frac{\lambda}{\E_X \tau_X} \mI +  \E_{X \sim \gD} \frac{\tau_X}{\E_X \tau_X} \info_X(\gopt))^{-1} \vx\}} \label{eq:sketch-3a}\\
&\leq \sqrt{O(d^{-1})} \cdot \sqrt{O(d^2)} = O(\sqrt{d}). \label{eq:sketch-3b}
\end{align}
Here, \eqref{eq:sketch-3a} is due to Cauchy-Schwarz and the first inequality in \eqref{eq:sketch-3b} is by the variation bound for $\gopt$ (\autoref{lem:variation-unif-optg}). Altogether, we have that
\begin{align}
    \E_{X \sim \gD}(1 -\varphi_{4d^c}(\max_{\vx \in X}\{\vx^\top \mW^{-1} \vx\})) \cdot \max_{\vx \in X}\{\vx^\top \mW^{-1} \vx\}\leq O(\sqrt{d}). \label{eq:sketch-3}
\end{align} 

\paragraph{Step \uppercase\expandafter{\romannumeral 4}: Putting Things Together.} Combining \eqref{eq:sketch-2} and \eqref{eq:sketch-3}, we have
\begin{align}
\E_{X \sim \gD}  \sqrt{\max_{\vx \in X}\{\vx^\top \mW^{-1} \vx\}} \leq O(\sqrt{d \log d}). \notag 
\end{align}
By the definition of $\mW$, and together with \eqref{eq:sketch-1}, we have that
\begin{align}
\widetilde{\val}_\gD^{(\lambda)}(\pi) = \E_{X \sim \gD}  \sqrt{\max_{\vx \in X}\{\vx^\top (\lambda \mI +\info_{\gD}(\pi))^{-1} \vx\}} \leq O(\sqrt{d \log d}), \notag
\end{align}
proving \autoref{thm:learn-design}.

\subsection{Finding the Core} \label{sec:core}
We now present our algorithm ({\sc CoreIdentification}, \autoref{algo:core}) to find the core, and prove the following lemma on its guarantee.

\begin{algorithm}[t]
\caption{{\sc CoreIdentification}}
\label{algo:core}
\KwIn{$\lambda \in (0, 1)$, and $S = \{X_1, \dots, X_\gamma\}$}
\KwOut{A core set $C \subseteq S$}
$C_1 = S$\;
\For{$\xi = 1, 2, 3, \dots$}{
\lIf{$C_\xi$ satisfies \eqref{eq:core3}}{%
\KwRet{$C_\xi$}\label{loc:coreret}%
}\lElse{%
$C_{\xi + 1} = \{X_i \in C_\xi : \max\limits_{\vx \in X_i} \vx^\top (\lambda \mI + \frac{1}{\gamma}\sum_{X_i \in C_\xi}\info_{X_i}(\gopt))^{-1} \vx \le (1/2)d^c\}$\label{loc:coreelse}%
}
}
\end{algorithm}

\begin{lem} \label{lem:core} Let $S = \{X_1, \dots, X_\gamma\}$ be a sequence/multi-set of context sets. \autoref{algo:core} finds a core set $C \subseteq S$ in $O(d \log \lambda^{-1})$ iterations that satisfies \eqref{eq:core2} and 
\begin{align}
    \max_{X_i \in C} \max_{\vx \in X_i} \vx^\top (\lambda \mI + \frac{1}{\gamma}\sum_{X_i \in C}\info_{X_i}(\gopt))^{-1} \vx \le d^c. \label{eq:core3}
\end{align}
\end{lem}

We remark that \eqref{eq:core3} implies \eqref{eq:core1}, because $\frac{1}{\gamma}\sum_{X_i \in C}\info_{X_i}(\gopt) \preccurlyeq \info_{\Unif(C)}(\gopt)$. 

\begin{proof} 
 For any iteration $\xi$, we denote 
\begin{align}
     \mJ_\xi = (\lambda \mI +\frac{1}{\gamma}\sum_{X_i \in C_\xi}\info_{X_i}(\gopt))^{-1}. \notag 
\end{align}
We first claim that, for each $\xi$, either (a) $C_{\xi + 1}$ satisfies \eqref{eq:core3} (and thus the algorithm returns), or (b) $\det \mJ_{\xi + 1} \ge 2 \det \mJ_{\xi}$. To see this, suppose that (a) does not hold. In this case, we have that there exists $X_i \in C_{\xi + 1}$ and $\vx_i \in X_i$, such that 
\begin{align}
    \vx_i^\top \mJ_{\xi + 1} \vx > d^c. \label{eq:core-7000}
\end{align}
Since $X_i \in C_{\xi + 1}$, by \autoref{loc:coreelse} of \autoref{algo:core}, we know that 
\begin{align}
    \vx_i^\top \mJ_\xi \vx \le \frac{1}{2}d^c. \label{eq:core-7100}
\end{align}
Dividing \eqref{eq:core-7000} by \eqref{eq:core-7100}, together with \autoref{lem:comp}, we find that 
\begin{align}
   \frac{\det \mJ_{\xi + 1}}{\det \mJ_{\xi}} \ge \frac{\vx_i^\top \mJ_{\xi + 1} \vx}{\vx_i^\top \mJ_\xi \vx} > 2, \label{eq:core-7200}
\end{align}
proving the claim.

Now we prove the lemma.  First, we prove that the algorithm returns after at most $O(d \log \lambda^{-1})$ iterations. Note that $\mJ_1 \succcurlyeq (1 + \lambda)^{-1}\mI$. Furthermore, for every iteration $\xi$, we have that $\mJ_\xi \preccurlyeq \lambda^{-1} \mI$. Together with $\lambda < 1$, we have that $\det \mJ_\xi \le (2\lambda)^{-d} \det \mJ_1$. By the claim established above in \eqref{eq:core-7200}, we have that $\det \mJ_\xi \ge 2^{\xi - 1} \det \mJ_1$ so long as the algorithm does not return at iteration $\xi$. Thus we conclude that $\xi \le O(d \log \lambda^{-1})$ when the algorithm returns. 

Let $C = C_\xi$ be the returned set. We also need to show that $C$ satisfies \eqref{eq:core2}.  We claim that for each iteration $j$,
\begin{align}
    \abs{C_j \setminus C_{j + 1}} \le d^{2-c} \gamma, \label{eq:core-7300}
\end{align}
which implies \eqref{eq:core2}, because 
\begin{align}
    \frac{\abs{C}}{\gamma} = \frac{\gamma - \abs{C_1 \setminus C_\xi}}{\gamma} = 1 - \sum_{j = 1}^{\xi - 1} \frac{\abs{C_j \setminus C_{j + 1}}}{\gamma} \ge 1 - \sum_{j = 1}^{\xi - 1} d^{2-c} \ge 1 - d^{2-c} \cdot O(d \log \lambda^{-1}),  \notag 
\end{align}
where the first inequality uses \eqref{eq:core-7300} and the second inequality uses $\xi \le O(d \log \lambda^{-1})$. 

Finally, we prove \eqref{eq:core-7300}. We have that
\begin{align}
    \abs{C_j \setminus C_{j + 1}} &= \abs{\{X_i \in C_j : \max\limits_{\vx \in X_i} \vx^\top \mJ_j \vx > (1/2)d^c\}} \notag \\ 
    &= \gamma \cdot \Pr_{X_i \sim \Unif(S)}[\max_{\vx \in X_i} \{\ind[X_i \in C]\vx^\top \mJ_j \vx \}> (1/2)d^c] \notag  \\ 
    &\le \gamma \cdot 2 d^{-c} \E_{X_i \sim \Unif(S)} \max_{\vx \in X_i} \{ \ind[X_i \in C]\vx^\top \mJ_j \vx\}\le 2 \gamma d^{2-c}, \notag 
\end{align}
where the first inequality is by Markov's inequality and the second inequality uses  the variation bound for $\gopt$ (\autoref{lem:variation-unif-optg}) on the distribution $\gD = \ind[X \in C] \cdot X$, where $X \sim \Unif(S)$.
\end{proof}

\subsection{Uniform Concentration Lemmas} \label{sec:uniform-con}

Fix  $\lambda < 1$. We define the following set of positive semi-definite matrices
\begin{align}
\mathfrak{W} \defeq \{\mW \in \sR^{d \times d} \mid \lambda \mI \preccurlyeq \mW \preccurlyeq (1 + \lambda) \mI\}. \label{eq:core-psd-1}
\end{align}

Let $X_1, \dots, X_\gamma$ be a sequence of sets of context vectors with norm at most $1$. For any positive definite matrix $\mW \in \sR^{d \times d}$, we define the following functions.
\begin{align}
f(\mW) &\defeq \frac{1}{\gamma} \sum_{i = 1}^\gamma \varphi_{4d^c}(\max_{\vx \in X_i}\{\vx^\top \mW^{-1} \vx\}) \cdot \max_{\vx \in X_i}\{\vx^\top \mW^{-1} \vx\}, \notag \\
g(\mW) &\defeq  \frac1\gamma \sum_{i=1}^{\gamma}\varphi_{4d^c}(\max_{\vx \in X_i}\{\vx^\top \mW^{-1} \vx\}). \notag
\end{align}
\begin{lem} \label{lem:func-lip-1}
For any positive $\lambda < 1$, $f(\mW)$ and $g(\mW)$ are $2\lambda^{-3}$-Lipschitz (in terms of $2$-norm $\|\cdot\|$) in the range $\mathfrak{W}$.
\end{lem}
\begin{proof}
By \autoref{lem:matinvlip} (and that the context vectors have norm at most $1$), for any $X_i$, the function $\max_{\vx \in X_i}\{\vx^\top \mW^{-1} \vx\}$ is $\lambda^{-2}$-Lipschitz with respect to $\mW \in \mathfrak{W}$. Therefore, $\varphi_{4d^c}(\max_{\vx \in X_i}\{\vx^\top \mW^{-1} \vx\})$ is also $\lambda^{-2}$-Lipschitz with respect to $\mW$. Since $\varphi_{4d^c}(\cdot) \in (0, 1)$ and $\max_{\vx \in X_i}\{\vx^\top \mW^{-1} \vx\} \in (0, \lambda^{-1})$, we have that $\varphi_{4d^c}(\max_{\vx \in X_i}\{\vx^\top \mW^{-1} \vx\}) \cdot \max_{\vx \in X_i}\{\vx^\top \mW^{-1} \vx\}$ is $(\lambda^{-3} + \lambda^{-2})$-Lipschitz, and this proves the lemma.
\end{proof}

We now present our first uniform concentration lemma.
\begin{lem}[The first uniform concentration lemma] \label{lem:uniform-con-1} Let $X_1, \dots, X_\gamma \sim \gD$ be {\it i.i.d.}\ drawn from the distribution $\gD$. We have the following concentration properties,
\begin{align}
    \Pr[\sup_{\mW \in \mathfrak{W}} \{\E f(\mW) - f(\mW)\}\le d] &\ge 1 - \exp(O(d^2 \log (d\lambda^{-1})) - \gamma d^{2-2c}/128), \label{eq:core-eve2} \\ 
    \Pr[\sup_{\mW \in \mathfrak{W}} \{\E g(\mW) - g(\mW)\} \le d^{-1}] &\ge 1 - \exp(O(d^2 \log (d\lambda^{-1})) - \gamma d^{-2}/2). \label{eq:core-eve3}
\end{align}
\end{lem}

\begin{proof} Let $\mathfrak{V}_\varepsilon \subseteq \mathfrak{W}$ be an $\varepsilon$-covering of $\mathfrak{W}$ so that for any $\mM \in \mathfrak{W}$, there exists $\mN = \mN_\varepsilon(\mM) \in \mathfrak{V}_\varepsilon$ satisfying $\norm{\mM - \mN} \le \varepsilon$. 

For \eqref{eq:core-eve2}, we first consider a fixed matrix $\mN \in \mathfrak{V}_\varepsilon$. For $i \in [\gamma]$, let 
\begin{align}
Y_i = \varphi_{4d^c}(\max_{\vx \in X_i}\{\vx^\top \mW^{-1} \vx\}) \cdot \max_{\vx \in X_i}\{\vx^\top \mW^{-1} \vx\}. \notag 
\end{align}
Then $\{Y_i\}$ are independent and bounded as $\abs{Y_i} \le 8 d^c$ almost surely. Using \autoref{lem:hoeffding} with $\delta = \frac{d}{2}$ and $R = 8 d^c$, we have 
\begin{align}
    \Pr[\E f(\mN) - f(\mN) \le d/2] \ge 1 - 2\exp(-\frac{2 \gamma \delta^2}{R^2}) = 1 - 2\exp(-\gamma d^{2-2c}/128)). \notag 
\end{align}
Next we consider all $\mN \in \mathfrak{V}_\varepsilon$. Using a union bound, we have that
\begin{align}
\Pr[\max_{\mN \in \mathfrak{V}_\varepsilon}\{\E f(\mN) - f(\mN)\} \le d/2] \ge 1 - \abs{\mathfrak{V}_\varepsilon} \cdot 2 \exp(-\gamma d^{2-2c}/128).  \label{eq:uniform-con-1-b}
\end{align}
Finally, we choose $\varepsilon = \lambda^3 d/ 4$.  By the Lipschitzness of $f(\mW)$ in \autoref{lem:func-lip-1}, we have that 
\begin{align}
    \abs{f(\mW) - f(\mN_\varepsilon(\mW))} \le 2\lambda^{-3} \norm{\mW - \mN_{\varepsilon}(\mW)} \le  2\lambda^{-3} \varepsilon = d / 2. \notag 
\end{align}
Therefore, using \eqref{eq:uniform-con-1-b}, we have that
\begin{align}
\Pr[\sup_{\mW \in \mathfrak{W}} \{\E f(\mW) - f(\mW)\}\le d] &\ge \Pr[d/2 + \sup_{\mW \in \mathfrak{V}_\varepsilon} \{\E f(\mW) - f(\mW)\}\le d] \notag \\ 
&\ge 1 - \abs{\mathfrak{V}_\varepsilon} \cdot 2 \exp(-\gamma d^{2-2c}/128)) \notag \\ 
&\ge 1 - \exp(O(d^2 \log (d\lambda^{-1})) - \gamma d^{2-2c}/128), \notag 
\end{align}
where the last inequality uses the covering number bound in \autoref{lem:matcover}. 

For \eqref{eq:core-eve3}, we can prove it similarly as \eqref{eq:core-eve2}. The only differences are that 1) we need to apply \autoref{lem:hoeffding} with $\delta = 1/(2d)$ and $R = 1$, and 2) we need to choose $\varepsilon = \lambda^3 / (4d)$.
\end{proof}

We define the policy class $\Pi$ by 
\begin{align}
 \Pi &\defeq \{\pmsp_\gM \mid \gM = \{(p_j, \mM_j)\}_{j = 1}^n, p_j \ge 0, p_1 + \cdots + p_n = 1, \mM_j \in \mathfrak{M}, n \leq 4 d \log d\}, \notag  \\ 
   \text{where}~\mathfrak{M} &\defeq \{\mM \in \sR^{d \times d} \mid d^{-1} \mI \preccurlyeq \mM \preccurlyeq \lambda^{-1} \mI \}, \notag  
\end{align}
and we define the following (matrix-valued) function on $\mW \in \mathfrak{W}$ and $\pi = \pmsp_\gM\in \Pi$,
\begin{align}
\mF(\mW, \pmsp_\gM) = \mF(\mW, \gM) \defeq \frac{1}{\gamma}\sum_{i = 1}^\gamma \varphi_{2d^c}(\max_{\vx \in X_i} \{\vx^\top \mW^{-1} \vx\}) \cdot \mW^{-1/2} \info_{X_i}(\pmsp_\gM) \mW^{-1/2}.\label{eq:def-mF}
\end{align}

\begin{lem} \label{lem:func-mF-smooth} We claim the following smoothness properties of the function $\mF(\mW, \gM)$ on its parameters,
\begin{enumerate}
    \item $\mF(\cdot,\cdot)$ is $3\lambda^{-3}$-Lipschitz with respect to $\mW$;
    \item $\mF(\cdot,\cdot)$ is $\lambda^{-2}$-Lipschitz with respect to each $p_j$;
    \item for any two parameters 
    \begin{align}
        \gM = \{(p_j, \mM_j)\}_{j = 1}^n, \gM' = \{(p_j, \mM'_j)\}_{i = 1}^n, \quad \mathrm{such~that} ~ \max_{1 \le j \le n} \norm{\mM_j - \mM'_j} \le 1/R, \notag 
    \end{align}
    where $R \ge 100\lambda^{-1} \cdot  d \log K$, we have 
    \begin{align}
       (1-\lambda/30)\mF(\mW, \gM') \preccurlyeq \mF(\mW, \gM) \preccurlyeq (1-\lambda/30)^{-1} \mF(\mW, \gM'), \label{eq:func-mF-3}
    \end{align}
    which further implies (since $ \mF(\mW, \gM')\preccurlyeq \lambda^{-1} \mI$),
    \begin{align}
        \mF(\mW, \gM') - \frac1{30} \cdot \mI \preccurlyeq \mF(\mW, \gM) \preccurlyeq  \mF(\mW, \gM') + \frac1{20} \cdot \mI. \notag 
    \end{align}
\end{enumerate}
\end{lem}

\begin{proof} For item (a), we note that $\norm{\mW^{-1/2}} \le \lambda^{-1/2}$, that $\mW^{-1/2}$ is $\lambda^{-3/2}$-Lipschitz with respect to $\mW$ by \autoref{lem:matinvlip}, and that $\norm{\info_{X_i}(\pi)} \le \Tr(\info_{X_i}(\pi)) \le 1$. Also, by the proof of \autoref{lem:func-lip-1}, we have that $\varphi_{2d^c}(\max_{\vx \in X_i}\{\vx^\top \mW^{-1} \vx\})$ is  $\lambda^{-2}$-Lipschitz with respect to $\mW$, that $\varphi_{2d^c}(\cdot) \in (0, 1)$, and that $\max_{\vx \in X_i}\{\vx^\top \mW^{-1} \vx\} \in (0, \lambda^{-1})$. Therefore, we can prove this item.

For item (b), we note that $\info_{X_i}(\pmsp_\gM)$ is $1$-Lipschitz in each $p_j$ and we conclude by noting that $\norm{\mW^{-1/2}}^2 \le \lambda^{-1}$ and  $\max_{\vx \in X_i}\{\vx^\top \mW^{-1} \vx\} \in (0, \lambda^{-1})$.

For item (c), since a mixed-softmax policy is a mixture of softmax policies, in the remaining proof, we first analyze the information matrix of the softmax policies $\psp_{\mM_j}, \psp_{\mM'_j}$, and then analyze that of the mixed-softmax policies $\pmsp_\gM, \pmsp_{\gM'}$. 
Since $\norm{\mM_j - \mM'_j} \le 1/R$, we have 
\begin{align}
    \mM'_j =  \mM_j  + (\mM'_j - \mM_j) \succcurlyeq  \mM_j  -  \frac{\mI}{R} \succcurlyeq (1 - \frac{d}{R}) \mM_j, \label{eq:rounding-3500}
\end{align}
where the second inequality in \eqref{eq:rounding-3500}  uses $\mM_j \succcurlyeq d^{-1}\mI$. Similarly, we can show 
\begin{align}
    \mM'_j =  \mM_j  + (\mM'_j - \mM_j) \preccurlyeq  \mM_j  +  \frac{\mI}{R} \preccurlyeq (1 + \frac{d}{R}) \mM_j. \notag 
\end{align}

Recall that $\alpha = \ln K$. For the softmax policy $\psp_{\mM'_j}$ and any vector $\vx \in \sR^d$, we have that
\begin{align}
    (\vx^\top \mM'_j \vx)^\alpha &\ge (1 - \frac{d}{R})^\alpha (\vx^\top \mM_j \vx)^\alpha = (1 - \frac{d}{R})^{\ln K} (\vx^\top \mM_j \vx)^\alpha \nonumber  \\
    &\ge (1 - \frac{1}{100 \lambda^{-1} \cdot \ln K})^{\ln K} (\vx^\top \mM_j \vx)^\alpha \ge (1-\lambda/100)(\vx^\top \mM_j \vx)^\alpha.
    \label{eq:rounding-4500}
\end{align}
Similarly, we have 
\begin{align}
    (\vx^\top \mM'_j \vx)^\alpha &\le (1 + \frac{d}{R})^\alpha (\vx^\top \mM_j \vx)^\alpha = (1 + \frac{d}{R})^{\ln K} (\vx^\top \mM_j \vx)^\alpha \nonumber  \\
    &\le (1 + \frac{1}{100\lambda^{-1} \cdot \ln K})^{\ln K} (\vx^\top \mM_j \vx)^\alpha \le (1 + \lambda/50) (\vx^\top \mM_j \vx)^\alpha. \label{eq:rounding-5500}
\end{align}
Therefore, for any context set $X$, we have 
\begin{align}
    \Pr[\psp_{\mM'_j}(X) = \vx]& = \frac{(\vx^\top \mM'_j \vx)^\alpha}{\sum_{\vx \in X} (\vx^\top \mM'_j \vx)^\alpha} \nonumber \\
    &\ge \frac{(1 - \lambda / 100)(\vx^\top \mM_j \vx)^\alpha}{\sum_{\vx \in X} (1 + \lambda / 50) (\vx^\top \mM_j \vx)^\alpha} \geq (1 - \lambda / 30)\Pr[\psp_{\mM_j}(X) = \vx], \notag
\end{align}
where the first inequality uses \eqref{eq:rounding-4500} and \eqref{eq:rounding-5500}. As a direct corollary, for any context set $X$, we have 
\begin{align}
    \info_X(\psp_{\mM'_j}) &=  \sum_{\vx \in X} \vx \vx^\top \cdot \Pr[\psp_{\mM'_j}(X) = \vx] \nonumber \\
    &\succcurlyeq (1 - \lambda / 30)\sum_{\vx \in X} \vx \vx^\top \cdot \Pr[\psp_{\mM_j}(X) = \vx] = (1 - \lambda / 30) \info_{X}(\psp_{\mM_j}). 
    \label{eq:rounding-6000}
\end{align}
Therefore, for the mixed-softmax policy, we have 
\begin{align}
    &\info_X(\pmsp_{\gM'}) = \frac{1}{2} \info_X(\gopt) + \sum_{i = 1}^n \frac{p_j}{2} \info_X(\psp_{\mM'_j}) \nonumber \\
    &\qquad\qquad\qquad\qquad\succcurlyeq (1 - \lambda / 30)\left(\frac{1}{2} \info_X(\gopt) + \sum_{i = 1}^n \frac{p_j}{2} \info_X(\psp_{\mM_j}) \right) = (1 - \lambda / 30)\info_{X}(\pmsp_{\gM}), 
 \label{eq:rounding-6500}
\end{align}
where the inequality uses \eqref{eq:rounding-6000}. Swapping $\mM_j$ and $\mM'_j$, we find 
\begin{align}
\info_X(\pmsp_{\gM'}) \preccurlyeq (1-\lambda/30)^{-1}\info_{X}(\pmsp_{\gM}). \label{eq:rounding-7000}
\end{align}
Setting $X = X_i$ in \eqref{eq:rounding-6500} and \eqref{eq:rounding-7000}, we establish \eqref{eq:func-mF-3}.
\end{proof}

We now present our second uniform concentration lemma.

\begin{lem}[The second uniform concentration lemma]\label{lem:uniform-con-2} Fix $\lambda <1$. For any  mixed-softmax policy
$\pmsp_{\gM}$ with $\gM \in \mathfrak{M}$, and any positive semi-definite matrix $\mW \in \mathfrak{W}$ (defined in \eqref{eq:core-psd-1}), we define the random function
\begin{align}
\mG(\mW, \gM) \defeq \mF(\mW, \gM) - \E \mF(\mW, \gM), \notag
\end{align}
where the random function $\mF$ is defined in \eqref{eq:def-mF}, and the randomness is from the independent samples $X_1, \dots, X_\gamma \sim \gD$. We then have that 
\begin{align}
    \Pr[\sup_{\mW \in \mathfrak{W}} \sup_{\pi \in \Pi} \norm{\mG(\mW, \gM)} \le \frac14] &\ge 1 -  \exp(O(d^3 \log d \log (d \lambda^{-1})) - \gamma d^{-2c}  \cdot 2^{-16}), \label{eq:core-eve1}
\end{align}
\end{lem}

\begin{proof} First, we consider fixed $\mW, \gM$. Using \autoref{lem:ourmhoeffding} with $\delta = 1/32$ and $R = 2 d^c$, we have 
\begin{align}
    \Pr[\norm{\mG(\mW, \gM)} \le \frac{1}{32}] \ge 1 - 2d \exp(- \frac{\gamma \delta^2}{8 R^2 + 4 \delta R / 3}) \ge 1 - 2d \exp(-\gamma d^{-2c}\cdot 2^{-16}). \notag 
\end{align}

Second, we define the covering. Let 
\begin{align}
    \mathfrak{V} = \gN(\mathfrak{W}, \lambda^{3} / 120, \norm{\cdot}), ~~  \mathfrak{P} = \gN([0, 1], \lambda^2 / (40 \cdot 4d\log d), \abs{\cdot}),  ~~ \mathfrak{N} = \gN(\mathfrak{M}, 1 / (100 \lambda^{-1}\cdot d \ln K), \norm{\cdot}), \notag 
\end{align}
    and let 
\begin{align}
    \mathfrak{S} = \{\pmsp_{\gM} \in \Pi \mid \gM = \{(p_j = q_j/(q_1 + \cdots + q_n), \mM_i)\}_{j = 1}^n, q_j \in \mathfrak{P}, \mM_i \in \mathfrak{N}\}. \notag 
\end{align}
We have 
\begin{align}
&\quad\Pr[\max_{\mW \in \mathfrak{V}} \max_{\pi \in \mathfrak{S}} \norm{\mF(\mW, \pi) - \E \mF(\mW, \pi)} \le \frac{1}{32}] \label{eq:core-8000} \\ 
&\ge 1 - 2d \abs{\mathfrak{V}}\abs{\mathfrak{S}} \exp(-\gamma d^{-2c} \cdot 2^{-16}) \notag \\ 
&\ge 1 - 2d \abs{\mathfrak{V}}\abs{\mathfrak{P}}^n \abs{\mathfrak{N}}^n \exp(-\gamma d^{-2c}  \cdot 2^{-16}) \notag \\ 
&\ge 1 - \exp(\log d + O(d^2 \log(d\lambda^{-1})) + n \cdot \left[ O(d \log \lambda^{-1}) +  O\left(d^2 \log(d\lambda^{-1}d \log K) \right) \right]  - \gamma d^{-2c} \cdot 2^{-16})  \notag \\ 
&\ge 1 - \exp(O(d^3 \log d \log (d \lambda^{-1})) - \gamma d^{-2c}  \cdot 2^{-16}), \notag 
\end{align}
where the last inequality uses that $\log K \le O(d)$.

Finally, we invoke the smoothness results from \autoref{lem:func-mF-smooth}. Note that when the event in \eqref{eq:core-8000} holds, by item (c) of \autoref{lem:func-mF-smooth}, we have 
\begin{align}
    \max_{\mW \in \mathfrak{V}} \max_{\{p_j = q_j / (q_1+\dots + q_n) : q_j \in \mathfrak{P}\}} \sup_{\{\mM_i \in \mathfrak{M}\}}  \norm{\mG(\mW, \pi)} \le  \frac{1}{32} + \frac{1}{20} \cdot 2 \leq \frac{1}{7}. \notag 
\end{align}
Note that the Lipschitz constant of $\mG(\cdot,\cdot)$ is at most double of that of $\mF(\cdot,\cdot)$. By item (b)  of \autoref{lem:func-mF-smooth}, we have 
\begin{align}
    \max_{\mW \in \mathfrak{V}} \sup_{\pi \in \Pi}  \norm{\mG(\mW, \pi)} \le \frac{1}{7} + 2\lambda^{-2} \cdot 4d\log d \cdot \frac{\lambda^2}{40 \cdot 4d\log d} = \frac{1}{5}. \notag 
\end{align}
Finally, by item (a) of \autoref{lem:func-mF-smooth}, we have that
\[
    \sup_{\mW \in \mathfrak{W}} \sup_{\pi \in \Pi}  \norm{\mG(\mW, \pi)} \le \frac{1}{5} + 2 \cdot 3\lambda^{-3} \cdot \frac{\lambda^{3}}{120} = \frac{1}{4}.  \qedhere
\]
\end{proof}

\section{Putting Everything Together: the Optimal Batch Algorithm}

 \label{sec:blinucbdg}

\begin{algorithm}[t]
\caption{\BatchedLinUCBDG}
\label{algo:linucbdg}
$M = \lceil \log \log T \rceil + 1, \alpha \gets 10\sqrt{ \ln \frac{2dKT}{\delta}}, \pi^0 = \gopt$, $\gT = \{\gT_1, \gT_2, \dots, \gT_M\}$, where $\gT_0 = 0$, $\gT_1 = \sqrt{T}$, $\gT_2 = 2\sqrt{T}$,  and $\gT_i = T^{1 - 2^{-(i-1)}}$ for $i \in \{3, \dots, M-1\}, \gT_M = T$\;
\For{$k \gets 1, 2, \dots, M$}{
\For{$t \gets\gT_{k - 1} + 1, \gT_{k - 1} + 2, \dots, \gT_k$}{
$A_{t}^{(0)} \gets [K], \hat r_{ti}^{(0)} \gets 0, \omega_{ti}^{(0)} \gets 1$\;
\For(\Comment*[f]{Eliminate}){$\kappa \gets 1, 2, \dots, k - 1$}{
$\forall i \in A_{t}^{(\kappa-1)} : \hat r_{ti}^{(\kappa)} \gets  \vx_{t i}^\top \hat \vtheta_{\kappa}, \omega_{t i}^{(\kappa)} \gets \alpha \sqrt{\vx_{t i}^\top \mLambda_{\kappa}^{-1} \vx_{t i}}$\;
$ A_{t }^{(\kappa)} \gets \{i \in  A_{t }^{(\kappa-1)} \mid \hat r_{ti}^{(\kappa)} +   \omega_{ti}^{(\kappa)} \ge \hat r_{tj}^{(\kappa)} -  \omega_{tj}^{(\kappa)},  \forall j \in  A_{t }^{(\kappa-1)}\}$\; 
}
$A_t \gets A_{t}^{(k - 1)}$\; 
Select $i_t$ such that $\vx_{t,i_t} \sim \pi_{k-1}(\{\vx_{t,i} : i \in A_t\})$, play arm $i_t$, and receive reward $r_t$\;
    $\vx_t \gets \vx_{t, i_t}$\;
}
Evenly divide $\{\gT_{k-1} + 1, \dots, \gT_k\}$ into two sets $\gA, \gB$\;\label{line:blinucbdg-learn-policy-pre}
 $\lambda \gets 32 \ln (2dT/\delta), \mLambda_{k} \gets \lambda \mI + \sum_{\tau \in \gA} \vx_\tau \vx_\tau^\top$, $\vxi_{k} \gets \sum_{\tau \in \gA} r_\tau \vx_\tau, \hat \vtheta_{k} \gets \mLambda_{k}^{-1} \vxi_k$\;
 \For{$\tau \in \gB$}{
    $\forall i \in A_{\tau}^{(k-1)} : \hat r_{ti}^{(k)} \gets  \vx_{t i}^\top \hat \vtheta_{k}, \omega_{t i}^{(k)} \gets \alpha \sqrt{\vx_{t i}^\top \mLambda_{k}^{-1} \vx_{t i}}$\;
    $ A_{t }^{(k)} \gets \{i \in  A_{t }^{(k-1)} \mid \hat r_{ti}^{(k)} +   \omega_{ti}^{(k)} \ge \hat r_{tj}^{(k)} -  \omega_{tj}^{(k)},  \forall j \in  A_{t }^{(k-1)}\}$\; 

 }
 Use the context sets $S = \{\{\vx_{\tau, a} \mid a \in A_\tau^{(k)}\}\}_{\tau \in \gB}$ and $\lambda = 1/T$ as the input of \autoref{algo:learn-design} and learn the sample policy $\pi_k$\; \label{line:blinucbdg-learn-policy}
}
\end{algorithm}

Our final algorithm with $O(\log \log T)$ static-grid batches and optimal minimax expected regret (up to $\mathrm{poly} \log T$ factors) is presented in \autoref{algo:linucbdg}. Compared with \BatchedLinUCB and \BatchedLinUCBKW, the main difference here is the addition of from \autoref{line:blinucbdg-learn-policy-pre} to \autoref{line:blinucbdg-learn-policy}, which not only learns the new estimate $\hat\vtheta_k$, but also the new sample policy $\pi_k$. Learning of the two objects are done through disjoint sets of samples ($\gA$ and $\gB$). This is because that $\gD_k$ depends on $\hat\vtheta_k$ (which is learned from $\gA$) and we have to make $\gB$ disjoint from $\gA$ so as to ensure elements in $S$ are independently sampled from $\gD_k$.

The following theorem bounds the expected regret of \autoref{algo:linucbdg}. 
\begin{thm} \label{thm:blinucbdg}
Assume that $T \leq \exp(d)$ and $T \geq \Omega(d^{32} \log^4 d \log^2 (\delta^{-1}))$. With probability at least $(1 - \delta)$, the expected regret of \autoref{algo:linucbdg} is bounded as 
\begin{align}
    R^T_{\text{\BatchedLinUCBDG}} \le O(\sqrt{d T \log d \log(dKT/\delta)} \times \log \log T). \notag 
\end{align}
\end{thm}

Note that the assumption that $T \leq \exp(d)$ is not restrictive since otherwise we have $\log T \geq \Omega(d)$ and \BatchedLinUCBKW (\autoref{thm:BatchedLinUCBKW}) already achieves the minimax optimal regret up to $\mathrm{poly}\log T$ factors. We also note that the $K$ in the regret bound can be replaced by $\min\{K, d \log T\}$ by a simple $\varepsilon$-net argument, so that our regret bound becomes minimax-optimal for all $K$ (up to $\mathrm{poly}\log T$ factors).

We finally remark that we make no effort in optimizing the exponent in the constraint that $T \geq d^{O(1)}$. Some simple tricks may significantly reduce this exponent constant. For example, first running a revised version of \BatchedLinUCBKW till time $\sqrt{T/d}$ and then switch to \BatchedLinUCBDG would reduce the exponent to $17$. A more careful analysis in the concentration lemmas in \autoref{sec:learn-design} may further substantially optimize the constant.

We now provide the proof of \autoref{thm:blinucbdg}.
\begin{proof}[Proof of \autoref{thm:blinucbdg}.]
We adopt the notations in \autoref{sec:blinucb}. Conditioned on the batches $1, 2, \dots, k-1$, we can bound the expected regret incurred in batch $k$ similarly as \eqref{eq:thm-blinucb-40}, and have that with probability at least $(1 - \delta \gT_k/T^2)$,
\begin{align}
R_k \leq  4\alpha \gT_k \times \E_{X \sim \gD_{k-1}}\max_{\vx \in X} \sqrt{\vx^\top \mLambda_{k-1}^{-1} \vx}.  \label{eq:blinucbdg-rk}
\end{align}

Furthermore, similar to \autoref{lem:exploration}, we can show that for each batch $k$ ($k < M$), with probability $(1 - \delta / T^2)$, we have that 
\begin{align}
\mLambda_k \succcurlyeq \frac{\gT_{k}}{32} \left(\frac{ \ln T}{ \gT_{k}} \mI + \E_{X \sim \gD_{k-1}} \E_{\vx \sim \pi_{k-1}(X)}[\vx \vx^\top]\right) \succcurlyeq \frac{\gT_k}{32} \left(T^{-1} \cdot \mI + \info_{\gD_{k-1}}(\pi_{k-1})\right). \label{eq:lem-exploration-new}
\end{align}
Note that compared with \eqref{eq:lem-exploration}, \eqref{eq:lem-exploration-new} has a worse constant $32$ since $\gA$ only contains half of the samples.

For each $k < M$, note that at \autoref{line:blinucbdg-learn-policy}, $S = \{\{\vx_{\tau, a} \mid a \in A_\tau^{(k)}\}\}_{\tau \in \gB}$ contains {\it i.i.d.}~samples from $\gD_{k}$, and $|S| \geq |\gT_k - \gT_{k-1}| /2 \geq \sqrt{T} / 4$. By \autoref{thm:learn-design}, we have that with probability $(1 - \exp(O(d^4 \log^2 d) - \sqrt{T} d^{-12} \cdot 2^{-18}) \geq 1 - \delta/T^2$ (since $T \geq \Omega(d^{32} \log^4 d \log^2 (\delta^{-1}))$, it holds that
\begin{align}
\widetilde{\val}_{\gD_k}^{(1/T)}(\pi_k) \leq O(\sqrt{d\log d}). \label{eq:blinucbdg-val}
\end{align}

The expected regret incurred during batch $1$ and batch $2$ is at most $2\sqrt{T}$. For any $k\geq 3$, assuming \eqref{eq:blinucbdg-rk} holds for batch $k$, and  \eqref{eq:lem-exploration-new} and \eqref{eq:blinucbdg-val} hold for batch $(k-1)$, we have that
\begin{align}
R_k &\le 4\alpha \gT_k  \E_{X \sim \gD_{k-1}}\max_{\vx \in X} \sqrt{\vx^\top \mLambda_{k-1}^{-1} \vx}
\leq \frac{4\sqrt{32}\alpha \gT_k  }{\sqrt{\gT_{k-1}}}  \E_{X \sim \gD_{k-1}}\max_{\vx \in X} \sqrt{\vx^\top \left(T^{-1}  \mI + \info_{\gD_{k-2}}(\pi_{k-2})\right)^{-1} \vx} \nonumber  \\
&\leq 32 \alpha \sqrt{T} \cdot \E_{X \sim \gD_{k-2}}\max_{\vx \in X} \sqrt{\vx^\top \left(T^{-1}  \mI + \info_{\gD_{k-2}}(\pi_{k-2})\right)^{-1} \vx} \label{eq:blinucbdg-10} \\
&\leq 32 \alpha \sqrt{T} \cdot \widetilde{\val}^{(1/T)}_{\gD_{k-1}}(\pi_{k-1}) \leq O(\sqrt{dT \log d \log (dKT/\delta)}), \notag 
\end{align}
where \eqref{eq:blinucbdg-10} is because that $X\sim \gD_{k-1}$ can be sampled via first drawing $X' \sim \gD_{k-2}$, then performing one-step elimination on $X'$, and getting $X \subseteq X'$.

Finally, collecting the failure probabilities for all $O(\log \log T)$ batches, we prove the desired regret bound.
\end{proof}

\section{Rarely Switching Algorithm for Adversarial Contexts and $\ln K \leq o(d)$} \label{sec:rarely-switch-suplinucb}

\citet{abbasi2011improved} showed an algorithm for adversarial contexts that achieves $d\sqrt{T} \times \mathrm{poly}\log T$ regret for any $K$. The authors also propose a special doubling trick that only updates the policy when the determinant of the corresponding information matrix (i.e., $\lambda \mI + \sum_{t} \vx_{t, i_t} \vx_{t, i_t}^\top$) doubles. Using this trick, their algorithm only uses $O(d \log T)$ policy switches, while still achieving the same order of regret. However, when $\log K \ll d$, there is a gap between the regret of their algorithm and the target minimax-optimal regret $\sqrt{T \min\{d, \log K\}} \times \mathrm{poly}\log T$. In this section, we propose an algorithm to close this gap, while still maintaining a small number of policy switches.

\paragraph{The Natural Approach and its Limitation.} The most natural approach is to apply the determinant-based doubling trick to the minimax-optimal algorithms for fewer number of arms, such as {\sc SupLinUCB} \citep{chu2011contextual} and {\sc SupLinRel} \citep{auer2002using}. However, a direct implementation of such an approach would lead to $O(d \log^2 T)$ policy switches. The reason is that, to replace an $\sqrt{d}$ factor by the $\sqrt{\log K}$ factor in the algorithm by \citet{abbasi2011improved}, the state-of-the-art concentration inequalities have to crucially rely on the statistical independence between the noises and the context vectors of the played arms, which is not true in the plain {\sc LinUCB} algorithm (and the {\sc OFUL} algorithm in \citep{abbasi2011improved}). In contrast, the concentration inequality used in \citep{abbasi2011improved} does not require such strong independence, but loses a $\sqrt{d}$ factor when $K$ is small.

To ensure the independence, \citet{auer2002finite} and \citet{chu2011contextual} came up with a more sophisticated layering trick, where each time step is assigned to one of the layers. The layers form a hierarchy and the observations from (the time steps in) each layer give more and more accurate estimates for the mean rewards, as the level of the layer increases. Meanwhile, it is possible to ensure the independence between the observations and the context vectors within each layer, so that the more accurate concentration inequality (e.g., \autoref{lem:conf} in this paper) can be applied. There are $\Theta(\log T)$ layers in {\sc SupLinUCB} and {\sc SupLinRel}, where each layer maintains a separate information matrix for the estimation. Therefore, if we  directly apply the determinant-based doubling trick, there will be $O(d \log T)$ updates in each layer, leading to $O(d \log^2 T)$ policy updates in total.

\paragraph{Our Approach.} Our approach is a simple combination of both types of algorithms mentioned above. Note that the estimation accuracy of the layers in {\sc SupLinUCB} and {\sc SupLinRel}  starts from $\Omega(1)$ for the first layer, and halves as the level of the layer increases. Therefore, it takes $\Theta(\log T)$ levels to reach the sufficient accuracy of $\sqrt{d/T}$. We also note that by the detailed analysis, if the accuracy provided by a layer is $\varpi$, the regret incurred by the layer can be roughly bounded by $d/\varpi$ (up to poly-logarithmic factors).

To reduce the number of layers, in our algorithm, we introduce a special layer, namely \emph{layer $0$}, which helps to bootstrap the accuracy parameters. More precisely, at layer $0$, we use the concentration inequality by \citet{abbasi2011improved} (\autoref{lem:abbasi-ball}). Since such an inequality is not as efficient as \autoref{lem:analysis-lr}, the regret incurred by layer $0$ can only be bounded by $d^2/\varpi_0$ (up to poly-logarithmic factors), where $\varpi_0$ is the accuracy parameter for layer $0$. However, since the inequality does not rely on the strong independence assumption, instead of starting from the $\Omega(1)$ accuracy, we may directly set $\varpi_0 = d^{1.5}/\sqrt{T}$, a much smaller value, while the incurred regret is still as desired. From layer $1$, we go back to the normal layer settings as {\sc SupLinUCB} and {\sc SupLinRel}, and set $\varpi_{\kappa} = \varpi_{\kappa}/2$ for $\kappa = 1, 2,  \ldots$. Since the target accuracy is $d/\sqrt{T}$, we now only need $\kappa_0 = O(\mathrm{log}(\varpi_0 / (d/\sqrt T))) = O(\log d)$ layers to achieve the minimax-optimal regret. Together with the determinant-based doubling trick, our algorithm uses only $O(d \log d \log T)$ policy switches.

Our \SupLinUCB algorithm is formally presented in \autoref{algo:suplinucb}. Note that the key difference from {\sc SupLinUCB} is at \autoref{loc:suplinucb-1} and \autoref{loc:suplinucb-3}, where $\alpha_0$ and $\varpi_0$ are specially set. Also, at \autoref{loc:layerzero}, a special elimination rule for layer $0$ is implemented, which is different from the elimination rules for the rest of the layers at \autoref{loc:layernext}. We next formally analyze the algorithm.

\begin{algorithm}[t]
\caption{\SupLinUCB}
\label{algo:suplinucb}
$\kappa_0 \gets \lceil \log d  \rceil, \alpha_0 \gets 2 \sqrt{d \ln(2T / \delta)}, \forall \kappa \in \{1, 2, \dots, \kappa_0\}: \alpha_\kappa \gets 10\sqrt{ \ln(2dKT /\delta)}$\;\label{loc:suplinucb-1}
$\forall \kappa \in \{0, 1, 2, \dots, \kappa_0\}: \mLambda_{\kappa 0} \gets  \mI, \vxi_{\kappa 0} \gets \vzero, \zeta_\kappa \gets 0$\;
$ \varpi_0 \gets d^{1.5} / \sqrt{T}, \forall \kappa \in \{1, 2, \dots, \kappa_0\}: \varpi_\kappa \gets \varpi_{\kappa - 1} / 2$\;\label{loc:suplinucb-3}
\For{$t \gets 1, 2, \dots, T$}{
\For{$\kappa \gets 0, 1, \dots, \kappa_0$}{
$\hat{\vtheta}_\kappa  \gets \mLambda_{\kappa, \zeta_\kappa}^{-1} \vxi_{\kappa, \zeta_\kappa}, \forall i \in [K]: \hat{r}_{ti}^{(\kappa )} \gets \vx_{ti}^\top \hat{\vtheta}_\kappa, \omega_{ti}^{(\kappa )} \gets \alpha_\kappa  \sqrt{\vx_{ti}^\top \mLambda_{\kappa, \zeta_\kappa}^{-1} \vx_{ti}}$\;
}
$A_t^{(0)} \gets \{i \in [K] \mid \hat r_{ti}^{(0)} +   \omega_{ti}^{(0)} \ge \hat r_{tj}^{(0)} -  \omega_{tj}^{(0)},  \forall j \in [K]\}$\; \label{loc:layerzero}
\For{$\kappa \gets 0, 1, \dots, \kappa_0$}{
\If{$\kappa = \kappa_0$}{
 select any $i_t \in A_{t}^{(\kappa)}$ and set $\kappa_t \gets \kappa$, break\;
 }\ElseIf{$\omega_{ti}^{(\kappa)} \le \varpi_\kappa$ for all $i \in A_t^{(\kappa)}$ \label{loc:layercond}}{
 $A_t^{(\kappa + 1)} \gets \{ i\in A_t^{(\kappa)} \mid \hat r_{ti}^{(\kappa)} \ge \max_{j \in A_{t}^{(\kappa)}} \hat r_{tj}^{(\kappa)} - 2\varpi_\kappa\}$\;  \label{loc:layernext}
 }\Else{
  select $i_t \gets \argmax_{i \in A_t^{(\kappa)}} \omega_{t, i}^{(0)}$ and set $\kappa_t \gets \kappa$, break\; \label{loc:suplinucb-kappa-eq-0}
 }
}
play arm $i_t$ and receive reward $r_t$\;
$\mLambda_{\kappa t} \gets \mLambda_{\kappa, t-1} + \vx_{t, i_t} \vx_{t, i_t}^\top, \vxi_{\kappa t} \gets \vxi_{\kappa, t-1} + r_t \vx_{t, i_t}$ for $\kappa = \kappa_t$ and $\mLambda_{\kappa t} \gets \mLambda_{\kappa, t-1}, \vxi_{\kappa t} \gets \vxi_{\kappa, t-1}$ for all $\kappa \ne \kappa_t$\; \label{loc:beforeupd}
\textbf{forall} $\kappa \in \{0, 1, 2, \dots, \kappa_0\}$: \lIf{$\det \mLambda_{\kappa t} \ge C \det \mLambda_{\kappa, \zeta_\kappa}$}{$\zeta_\kappa \gets t$} \label{loc:suplinucb-upd}
}
\end{algorithm}

\begin{thm} \label{thm:suplinucb} For any $C \ge 2$, the number of policy switches made by \autoref{algo:suplinucb} is at most  $O(d \log d \log T / \log C)$; with probability $(1-2\delta)$, the expected regret of the algorithm is at most  ${O}(C \sqrt{d T} \log d \log T \log(d K T / \delta))$. 
\end{thm}

\begin{proof}
We first upper bound the number of policy switches. Note that for each $\kappa$, we have that $\ln \det \mLambda_{\kappa, T} \le O(d \log T)$,  and $\ln \det \mLambda_{\kappa, 0} = 0$; therefore, $\zeta_\kappa$ is updated at \autoref{loc:suplinucb-upd} by at most $O(d \log T / \log C)$ times. Since the learning policy is completely decided by $\{\mLambda_{\kappa, \zeta_\kappa}, \vxi_{\kappa, \zeta_\kappa}\}_{\kappa=0}^{\kappa_0}$, we conclude that the policy changes by at most $(\kappa_0 + 1) \cdot O(d \log T / \log C) = O(d \log d \log T /\log C)$ times.

We next prove the regret of \autoref{algo:suplinucb}. Note that when the event specified in \autoref{lem:abbasi-ball} holds (which happens with probability at least $1 - \delta$), for any $i \in [K]$ and any time step $t \in [T]$, we have 
\begin{align*}
    \abs{\vx_{ti}^\top \hat{\vtheta}_0 - \vx_{ti}^\top \vtheta} \le \omega_{ti}^{(0)},
\end{align*}
where we use \eqref{eq:abbasi-event-2} and that $\lambda = 1$ in \autoref{algo:suplinucb}.

We define $\Psi_{t,\kappa} \defeq \{\tau \leq t \mid \kappa_\tau = \kappa\}$ to be the set of the time steps assigned to layer $\kappa$ at or before time step $t$. Similar to Lemma 14 in \citep{auer2002using} and Lemma 4 in \citep{chu2011contextual}, we claim that for each $\kappa \geq 1$ and each time $t$, conditioned on any fixed $\Psi_{t-1, \kappa}$, the corresponding noises $\{r_{\tau} - \vx_{\tau, i_\tau}^\top \vtheta \mid \tau \in \Psi_{t - 1, \kappa}\}$ are independent sub-Gaussian random variables with variance proxy $1$. This is because for $\kappa \geq 1$, $\Psi_{t - 1, \kappa}$ only depends on $\{\omega_{\tau,i}^{(\kappa')}\mid \tau < t, \kappa' \leq \kappa, i \in [K]\}$ and $\{\hat{r}_{\tau,i}^{(\kappa')}\mid \tau < t, \kappa' < \kappa, i \in [K]\}$. While $\{\omega_{\tau,i}^{(\kappa')}\mid \tau < t, \kappa' \leq \kappa, i \in [K]\}$ only depends on the context vectors which are independent from the noises, $\{\hat{r}_{\tau,i}^{(\kappa')}\mid \tau < t, \kappa' < \kappa, i \in [K]\}$ depends on the context vectors and the noises generated from time steps in $\Psi_{t-1, 0} \cup \dots \cup \Psi_{t-1,\kappa-1}$, which is disjoint from $\Psi_{t-1, \kappa}$. Thus, the procedure for generating $\Psi_{t-1,\kappa}$ does not use the noises in the time steps in $\Psi_{t-1,\kappa}$, and therefore the noises are independent sub-Gaussian random variables even when conditioned on $\Psi_{t-1,\kappa}$. Given this statistical independence property, by \autoref{lem:conf}, we have that with probability at least $1 - \delta$, for any $i \in [K]$, any time step $t \in [T]$ and any $\kappa \in \{1, 2, \dots, \kappa_0\}$, it holds that 
\begin{align*}
    \abs{\vx_{ti}^\top \hat{\vtheta}_\kappa - \vx_{ti}^\top \vtheta} \le \omega_{ti}^{(\kappa)}.
\end{align*}
Now, summarizing the discussion above, we define the desired event  
\begin{align*}
    E \defeq \{\forall i \in [K], t \in [T], \kappa \in \{0\} \cup [\kappa_0]:  \abs{\vx_{ti}^\top \hat{\vtheta}_\kappa - \vx_{ti}^\top \vtheta} \le \omega_{ti}^{(\kappa)}\}, 
\end{align*} 
and have that $\Pr[E] \ge 1 - 2 \delta$. Below we will upper bound the expected regret incurred by the algorithm when conditioned on $E$.


For each layer $\kappa \in \{0, 1, 2, \dots, \kappa_0\}$, we define the regret incurred during time steps that are assigned to layer $\kappa$ as
\begin{align*}
    R_\kappa \defeq \sum_{t =1}^T  \ind[\kappa_t = \kappa] \cdot (\max_{i \in [K]} \vx_{ti}^\top \vtheta - \vx_{t, i_t}^\top \vtheta),
\end{align*}
Since each time step will be assigned to exactly one layer $\kappa \in \{0, 1, \dots, \kappa_0\}$, the total regret is
\begin{align}
    R^T  = \sum_{\kappa = 0}^{\kappa_0} \E[R_\kappa]. \label{eq:suplinucb-reg-500}
\end{align}

We will use the following lemmas.
\begin{lem} \label{lem:layersize} We have the following bounds for the size of each layer,
\begin{align*}
    \abs{\Psi_{T,\kappa}} \le \begin{cases}
    8 C (T / d) \ln T \ln(2 T / \delta), & \kappa = 0, \\
      200 C \cdot 4^{\kappa} (T / d^2)  \ln T \ln(2 d K T / \delta), & 1 \le \kappa \le \kappa_0 - 1, \\ 
    T, & \kappa = \kappa_0. 
    \end{cases}
\end{align*}
\end{lem}
\begin{lem} \label{lem:suplinucb-layer-regret} When the event $E$ happens, for any $t \in [T], \kappa \in \{1, 2, \dots, \kappa_0\}$, it holds that
\begin{align*}
     \ind[\kappa_t = 0] \cdot (\max_{i \in [K]} \vx_{ti}^\top \vtheta - \vx_{t, i_t}^\top \vtheta)&\le 4 \omega_{t, i_t}^{(0)}, \\ 
    \ind[\kappa_t = \kappa] \cdot (\max_{i \in [K]} \vx_{ti}^\top \vtheta - \vx_{t, i_t}^\top \vtheta) &\le 8 \varpi_{\kappa}.
\end{align*}
\end{lem}

Now, in light of \eqref{eq:suplinucb-reg-500}, we upper bound each $R_\kappa$ (conditioned on $E$). For $R_0$, we have  that
\begin{align}
    R_0 &= \sum_{t  \in \Psi_{T,0}} (\max_{i \in [K]} \vx_{ti}^\top \vtheta - \vx_{t, i_t}^\top \vtheta)  \le \sum_{t  \in \Psi_{T,0}}  4 \omega^{(0)}_{t, i_t}  = \sum_{t  \in \Psi_{T,0}} 4 \alpha_0 \sqrt{\vx_{t, i_t}^\top \mLambda_{0, \zeta_0}^{-1} \vx_{t, i_t}} \label{eq:suplinucb-reg-1000} \\
    & \le \sum_{t  \in \Psi_{T,0}} 4 \alpha_0 \sqrt{C \vx_{t, i_t}^\top \mLambda_{0, t - 1}^{-1} \vx_{t, i_t}} \le  4 \alpha_0 \sqrt{C \abs{\Psi_{T,0}} \sum_{t  \in \Psi_{T,0}}  \vx_{t, i_t}^\top \mLambda_{0, t - 1}^{-1} \vx_{t, i_t}} \label{eq:suplinucb-reg-1100} \\ 
    &\le 4 \alpha_0 \sqrt{2 C d \abs{\Psi_{T,0}} \ln T} \le 16 C \sqrt{dT} \ln T \ln(2 T / \delta), \label{eq:suplinucb-reg-1200}
\end{align}
where the inequality in \eqref{eq:suplinucb-reg-1000} uses \autoref{lem:suplinucb-layer-regret}, the first inequality in \eqref{eq:suplinucb-reg-1100} is due to \autoref{lem:comp} and the update rule at \autoref{loc:suplinucb-upd}, the second inequality in \eqref{eq:suplinucb-reg-1100} uses Cauchy-Schwarz, the first inequality in \eqref{eq:suplinucb-reg-1200} uses the elliptical potential lemma (\autoref{lem:ellip}), and the second inequality in \eqref{eq:suplinucb-reg-1200} uses \autoref{lem:layersize}.

For $\kappa \in \{1, 2, \dots \kappa_0 - 1\}$, we have that
\begin{align}
    R_\kappa &= \sum_{t  \in \Psi_{T,\kappa}}  (\max_{i \in [K]} \vx_{ti}^\top \vtheta - \vx_{t, i_t}^\top \vtheta)  \leq \sum_{t  \in \Psi_{T,\kappa}}  8 \varpi_{\kappa}  = 8 \varpi_{\kappa} \abs{\Psi_{T,\kappa}} \label{eq:suplinucb-reg-2100} \\ 
    &\le 1600 C\cdot 2^{\kappa} \cdot \frac{d^{1.5}}{\sqrt T} \cdot \frac{T}{d^2} \ln T \ln(2 d KT / \delta) \le 3200 C \sqrt{dT} \ln T \ln(2 d KT / \delta),   \label{eq:suplinucb-reg-2200}
\end{align}
where the inequality in \eqref{eq:suplinucb-reg-2100} uses \autoref{lem:suplinucb-layer-regret}, the  first inequality in \eqref{eq:suplinucb-reg-2200} uses \autoref{lem:layersize} and that $\varpi_{\kappa} = 2^{-\kappa} d^{1.5} / \sqrt T$. 

For $k = \kappa_0$, we have that
\begin{align}
    R_{\kappa_0} =\sum_{t  \in \Psi_{T,\kappa_0}}  (\max_{i \in [K]} \vx_{ti}^\top \vtheta - \vx_{t, i_t}^\top \vtheta) \le 8 \varpi_{\kappa_0} \abs{\Psi_{\kappa_0}} \le 8 \varpi_{\kappa_0} T \le 8 \times \frac{d^{0.5}}{\sqrt T} \cdot T = 8 \sqrt{dT}, \label{eq:suplinucb-reg-3000}
\end{align}
where the first inequality uses \autoref{lem:suplinucb-layer-regret}, the second inequality uses \autoref{lem:layersize}, the third inequality uses that  $\varpi_{\kappa_0} = 2^{-\kappa_0} d^{1.5} / \sqrt T$ and that $2^{\kappa_0} \ge d$. 

We prove the theorem by plugging \eqref{eq:suplinucb-reg-1200}, \eqref{eq:suplinucb-reg-2200}, and \eqref{eq:suplinucb-reg-3000} back to \eqref{eq:suplinucb-reg-500}.
\end{proof}

\subsection{Proof of \autoref{lem:layersize}}
 The third bound $\abs{\Psi_{T,\kappa_0}} \le T$ is self-evident, so we only prove the first two bounds. 
By the elliptical potential lemma (\autoref{lem:ellip}), for every $\kappa \in \{0,1,2,\dots, \kappa_0\}$, we have that
\begin{align*}
\sum_{t \in \Psi_{T,\kappa}} \vx_{t, i_t}^\top \mLambda_{\kappa, t - 1}^{-1} \vx_{t, i_t}  \le 2 d \ln T, 
\end{align*}
which, together with the Cauchy-Schwarz inequality, implies that
\begin{align*}
\sum_{t \in \Psi_{T, \kappa}} \sqrt{\vx_{t, i_t}^\top \mLambda_{\kappa, t - 1}^{-1} \vx_{t, i_t}}  \le \sqrt{2 d \abs{\Psi_{T,\kappa}} \ln T}.  
\end{align*}

In the following, we use $\zeta_{\kappa,t}$ to denote the value of $\zeta_\kappa$ at \autoref{loc:beforeupd} of \autoref{algo:suplinucb} during time step $t$. Note that we have $\zeta_{\kappa,t} \le t - 1$ for every $\kappa$.

By our update rule (\autoref{loc:suplinucb-upd}), we have that $\det \mLambda_{\kappa,t - 1} \le C\det \mLambda_{\kappa,\zeta_{\kappa,t}} $ for every $t \in [T]$. Therefore, for each $t \in [T]$ and $\kappa \in \{0,1, \dots \kappa_0 - 1\}$ such that $\kappa_t = \kappa$, together with \autoref{lem:comp}, we have that
\begin{align*}
    \sqrt{C \vx_{t, i_t}^\top \mLambda_{\kappa, t - 1}^{-1} \vx_{t, i_t}} \ge \sqrt{\vx_{t, i_t}^\top \mLambda_{\kappa, \zeta_{\kappa,t}}^{-1} \vx_{t, i_t}} \ge \frac{\varpi_\kappa}{\alpha_\kappa}, 
\end{align*}
where the last inequality is by \autoref{loc:suplinucb-kappa-eq-0}. Therefore, for each $\kappa \in \{0,1, \dots \kappa_0 - 1\}$, we have that
\begin{align*}
    \sqrt{2 C d \abs{\Psi_{T, \kappa}} \ln T} \ge \abs{\Psi_{T,\kappa}} \cdot \frac{\varpi_\kappa}{\alpha_\kappa},
\end{align*}
which implies that (for $\kappa \in \{1, 2, \dots, \kappa_0 -1\}$)
\begin{align*}
    \abs{\Psi_{T,\kappa}} \le  (\frac{\alpha_{\kappa} \sqrt{2 C d \ln T}}{\varpi_\kappa})^2 \le \frac{200 C d \ln T \ln (2 d K T / \delta)}{4^{-\kappa} d^3 / T } =200 C \cdot 4^{\kappa} (T / d^2)  \ln T \ln(2 d K T / \delta), 
\end{align*}
and
\begin{align*}
    \abs{\Psi_{T,0}} \le (\frac{\alpha_0 \sqrt{2 C d \ln T}}{\varpi_0})^2 \le \frac{8 C d^2 \ln T \ln(2 T / \delta)}{d^3 / T } = 8 C (T / d) \ln T \ln(2 T / \delta). 
\end{align*}

\subsection{Proof of \autoref{lem:suplinucb-layer-regret}}

\autoref{lem:suplinucb-layer-regret} is a direct corollary of the following two lemmas, which we prove separately in this subsection. 

\begin{lem} \label{lem:layer} When the event $E$ happens, for any $t \in [T], \kappa \in \{1, 2, \dots, \kappa_0\}$, it holds that
\begin{align*}
    \ind[\kappa_t = 0] \cdot (\max_{i \in A_t^{(0)}} \vx_{ti}^\top \vtheta - \vx_{t, i_t}^\top \vtheta) &\le 4 \omega_{t, i_t}^{(0)}, \\ 
     \ind[\kappa_t = \kappa] \cdot (\max_{i \in A_t^{(\kappa)}} \vx_{ti}^\top \vtheta - \vx_{t, i_t}^\top \vtheta) &\le 8 \varpi_{\kappa}.
\end{align*}
\end{lem}

\begin{lem} \label{lem:suplinucb-elim} When the event $E$ happens, for any $t \in [T]$, and all $\kappa$ such that $\kappa \leq \kappa_t$, we have that
\begin{align*}
      \max_{i \in [K]} \vx_{ti}^\top \vtheta =  \max_{i \in A_t^{(\kappa)}} \vx_{ti}^\top \vtheta. 
\end{align*}
\end{lem}

\begin{proof}[Proof of \autoref{lem:layer}] For the inequality, assuming that $\kappa_t = 0$, we have $i_t = \argmax_{i \in A_t^{(0)}} \omega_{t i}^{(0)}$. Then we have that
\begin{align*}
    \max_{i \in A_t^{(0)}} \vx_{ti}^\top \vtheta - \vx_{t, i_t}^\top \vtheta &\le \max_{i \in A_t^{(0)}} \vx_{ti}^\top \vtheta - \min_{i \in A_t^{(0)}} \vx_{t i}^\top \vtheta \\ 
    &\le \max_{i \in A_t^{(0)}}\{\vx_{ti}^\top \hat{\vtheta}_{0} + \omega_{ti}^{(0)}\} - \min_{i \in A_t^{(0)}}\{\vx_{t i}^\top \hat{\vtheta}_0 - \omega_{ti}^{(0)}\} \\ 
    &\le  4 \max_{i \in A_t^{(0)}} \omega_{ti}^{(0)} = 4 \omega_{t, i_t}^{(0)},
\end{align*}
where the second inequality is because of event $E$ and the third inequality  follows from the elimination rule at \autoref{loc:layerzero}, and the last equality is due to \autoref{loc:suplinucb-kappa-eq-0}.

For the second inequality, assuming that $\kappa_t = \kappa \geq 1$, we have that
\begin{align*}
    \max_{i \in A_t^{(\kappa)}} \vx_{ti}^\top \vtheta - \vx_{t, i_t}^\top \vtheta &\le \max_{i \in A_t^{(\kappa-1)}}\{\vx_{ti}^\top \hat{\vtheta}_{\kappa-1} + \omega_{ti}^{(\kappa-1)}\} - \min_{i \in A_t^{(\kappa-1)}}\{\vx_{t i}^\top \hat{\vtheta}_{\kappa-1} - \omega_{ti}^{(\kappa-1)}\} \\ 
    &\le2 \max_{i \in A_t^{(\kappa-1)}} \omega_{ti}^{(\kappa-1)} + \max_{i \in A_{t}^{(\kappa-1)}} \hat r_{ti}^{(\kappa-1)}  - \min_{i \in A_t^{(\kappa-1)}} \hat r_{ti}^{(\kappa-1)} \\ 
    &\le2 \max_{i \in A_t^{(\kappa-1)}} \omega_{ti}^{(\kappa-1)} + 2 \varpi_{\kappa-1} \leq 4\varpi_{\kappa - 1} = 8 \varpi_\kappa,
\end{align*}
where the first inequality is by the event $E$, the third inequality follows from  that $\min_{i \in A_{t}^{(\kappa-1)}} \hat{r}_{ti}^{(\kappa-1)} \ge \max_{i \in A_{t}^{(\kappa-1)}} \hat r_{ti}^{(\kappa-1)} - 2 \varpi_{\kappa-1}$ as implied by \autoref{loc:layernext} of \autoref{algo:suplinucb}, the  last inequality is because the condition at \autoref{loc:layercond} was met at  iteration $\kappa - 1$ (since otherwise the loop should have terminated at  iteration $\kappa - 1$).
\end{proof}

\begin{proof}[Proof of \autoref{lem:suplinucb-elim}] For any time step $t \in [T]$, note that when the event $E$ holds, by the elimination rule at \autoref{loc:layerzero}, we have that
\begin{align*}
\max_{i \in [K]} \vx_{ti}^\top \vtheta =  \max_{i \in A_t^{(0)}} \vx_{ti}^\top \vtheta.
\end{align*}
Also, for each $\kappa < \kappa_t$, by the elimination rule at \autoref{loc:layernext}, we have that
\begin{align*}
 \max_{i \in A_t^{(\kappa)}} \vx_{ti}^\top \vtheta =  \max_{i \in A_t^{(\kappa+1)}} \vx_{ti}^\top \vtheta,
\end{align*}
Applying the equality iteratively for $\kappa = 0, 1, 2, \dots, \kappa_t - 1$, and we prove the lemma.
\end{proof}

\section{Lower Bounds for Adversarial Contexts} \label{sec:lb}

In this section, we prove the following lower bound for the number of policy switches in the adversarial context setting. 

\begin{thm}\label{thm:lb}
Let $K = 2$, for any even number of dimensions $d \geq 2$, and $T$ greater than a sufficiently large constant times $d$, suppose the  expected number of policy switches made by the learner is at most $M$ ($20d \leq M \leq (d \ln T)/ 48$), then there exists a bandit instance such that the learner's expected regret on the instance is at least $\sqrt{dT} \times \frac{1}{32d}\left(\frac{2T}{d}\right)^{1/(16M/d + 2)}$.
\end{thm}

\autoref{thm:lb} shows that, even for $K = 2$, when $T \geq d^2$, in order to achieve $\sqrt{dT} \times \mathrm{poly}\log T$ regret, $M$ has to be $\Omega(d \log T/\log (d\log T))$. Note that on the upper bound side, our \autoref{algo:suplinucb} 
achieves $C$ times the target minimax-optimal regret (up to $\mathrm{poly} \log T$ factors) with $O((d \log d\log T) / \log C)$ policy switches, and our \autoref{thm:lb} shows that $\Omega(\frac{d\log T}{\log C + \log (d \log T)})$ policy switches are needed, almost matching the upper bound for every $C$.

To prove \autoref{thm:lb}, we first prove the lower bound in the special case of $d = 2$ in \autoref{sec:lb-2-d}. Then, in \autoref{sec:lb-general}, we prove the theorem for general $d$ using the special case as a building block.

\subsection{Lower Bound for Constant-Dimension Special Case} \label{sec:lb-2-d}

\begin{lem} \label{lem:lb} 
When $K=d=2$, for sufficiently large $T$, suppose the expected number of policy switches made by the learner is at most $M$ ($40 \leq M \leq (\ln T)/24$), then there exist a bandit instance such that the learner's expected regret on the instance is at least $T^{1/2+1/(8M+2)}/32$.
\end{lem}

To prove \autoref{lem:lb}, we will construct a class of bandit problem instances $\mathfrak{B} = \{B^{(\vu)}\}$, where each instance $B^{(\vu)}$ is parameterized by $\vu \in \{\pm 1\}^{L}$ and $L = 4M$. For any fixed learner with no more than $M$ policy switches, we will show that the regret averaged over the $2^L$ instances in $\mathfrak{B}$ is large, and therefore there exists at least one instance in $\mathfrak{B}$ that is bad for the learner.

For each $B^{(\vu)}$, we assume that the noises are independent centered Gaussian with variance $1$. We also need to define the hidden vector $\vtheta^{(\vu)}$ and the context vectors $\{\vx_{t,1}^{(\vu)}, \vx_{t,2}^{(\vu)}\}_{t = 1}^{T}$ (where, in our formal definition of linear bandits, $\gD_t$ is the deterministic distribution supported on $\{\vx_{t,1}^{(\vu)}, \vx_{t,2}^{(\vu)}\}$ for every $t$). Before defining $\mathfrak{B}$, we first divide the time steps into \emph{stages}, and define a few helpful notations.

\paragraph{Stages.} We uniformly divide the $T$ time steps into $L$ stages. Let $t_j \defeq \lceil jT/L \rceil$ for all $j \in \{0, 1, 2, \dots, L\}$. The $j$-th stage consists of the time steps in the range $(t_{j-1}, t_j]$. 

\paragraph{Additional Notations.} Let $\upsilon \defeq \sqrt{T}^{-1/(L+1)}$. Note that $\upsilon \leq 1/10$ since $L \leq (\ln T) / 6$. For each $\vu = (u_1, u_2, \dots, u_L) \in \{\pm 1\}^L$ and each $j \in \{ 1, 2, \dots, L\}$, we define the map $\psi_j(\vu) \defeq 1/2 + \sum_{i=1}^{j} u_i \cdot \upsilon^i$ that sends the sequence to the decimal. We have that $\psi_j(\vu) \leq 2/3$ since $\upsilon \leq 1/10$. For convenience, we also define $\psi(\vu) \defeq \psi_L(\vu)$. For each $j$, we also define $z_j \defeq \upsilon^{-(j+1)}/\sqrt{T} \leq 1$. 

\paragraph{Bandit Instances.} We now define $B^{(\vu)}$ for each $\vu \in \{\pm 1\}^{L}$. For the hidden vector, we let $\vtheta^{(\vu)} \defeq (\psi(\vu), \frac23 )^\top$. For every stage $j$, and every time step $t$ during stage $j$, we set the context vectors by $\vx_{t, 1}^{(\vu)} \defeq (z_j, 0)^\top$ and $\vx_{t, 2}^{(\vu)} \defeq (0, \frac{3}{2} z_j \cdot \psi_{j-1}(\vu))^\top$. One can easily verify that the norms of all vectors are upper bounded by $1$. 

\medskip

We now start analyzing the constructed instances.

\paragraph{Suboptimal Action and its Regret.} Since there are only two candidate actions during each time step, we refer to the one with smaller expected reward as the \emph{suboptimal action}. The following lemma lower bounds the expected regret incurred by playing a suboptimal action.

\begin{lem}\label{lem:subopt-action}
For any instance $B^{(\vu)}$, and any time step $t$, the regret incurred by playing the suboptimal action at time step $t$ is at least $ \upsilon^{-1}/(2\sqrt{T})$.
\end{lem}

\begin{proof}
Suppose that time step $t$ is in stage $j$. The  regret incurred by the suboptimal action is 
\begin{align}
\abs{(z_j,0)^\top \vtheta ^{(\vu)}-(0,\frac32 z_j \cdot \psi_{j-1}(\vu) )^\top \vtheta ^{(\vu)}} &= z_j \cdot \abs{\psi(\vu)-\psi_{j-1}(\vu)} \nonumber \\
&\ge z_j\cdot\left(\upsilon^j - \sum_{i=j+1}^{+\infty} \upsilon^i\right) \geq \frac{ z_j \upsilon^j}{2} \geq \frac{\upsilon^{-1}}{2\sqrt{T}}.\nonumber \qedhere
\end{align}
\end{proof}

\paragraph{The Regret of a Rarely Switching Learner.} For any learner who switches the policy for at most $M$ times, let $F_j$ be the event that the policy is not switched during stage $j$. Let $E_j$ be the event that the learner's policy $\chi_t$ places greater or equal to $1/2$ probability mass on the suboptimal action at time $t$, where $t = t_{j-1}+1$ is the first time step of stage $j$.  By \autoref{lem:subopt-action}, the expected regret of the learner for bandit instance $B^{(\vu)}$ can be lower bounded by
\begin{align}\label{eq:lb-regret-rarely-switch-first-step}
\E_{B^{(\vu)}}[R^T] \geq \sum_{j=1}^L \sum_{t=t_{j-1}+1}^{t_j} \Pr_{B^{(\vu)}} [E_j \cap F_j] \cdot \frac{1}{2} \cdot \frac{\upsilon^{-1}}{2\sqrt{T}} = \frac{1}{4\upsilon\sqrt{T}} \cdot (T/L) \cdot \sum_{j=1}^L \Pr_{B^{(\vu)}} [E_j \cap F_j],
\end{align}
where $\E_{B^{(\vu)}}[\cdot]$ denotes the expectation taken over the probability distribution induced by the learner and the bandit instance $B^{(\vu)}$ (and we similarly define $\Pr_{B^{(\vu)}}[\cdot]$). Let $p_j^{(\vu)} \defeq \Pr_{B^{(\vu)}} [E_j]$, continuing with \eqref{eq:lb-regret-rarely-switch-first-step}, we have 
\begin{align}
\E_{B^{(\vu)}}[R^T] \geq \frac{\sqrt{T}}{4\upsilon L} \sum_{j=1}^L \left(\Pr_{B^{(\vu)}} [E_j] - \Pr_{B^{(\vu)}} [\overline{F_j}]\right) \geq  \frac{\sqrt{T}}{4\upsilon L}\left( \sum_{j=1}^L p_j^{(\vu)} - M\right), \label{eq:lb-regret-rarely-switch-first-step-a}
\end{align}
where $\overline{F_j}$ denotes the complement event of $F_j$ and the last inequality is because that the learner can switch in at most $M$ stages (in expectation).

\paragraph{Probability of Playing a Suboptimal Action.} By the discussion above, to lower bound the regret, we need to lower bound $p_j^{(\vu)}$. We first prove the following lemma.
\begin{lem}\label{lem:lb-prob-difference-bound}
Consider any $\vu = (u_1, u_2, \dots, u_L) \in \{\pm 1\}^L$ and $\vu' = (u_1', u_2', \dots, u_L') \in \{\pm 1\}^L$. Suppose $\vu \neq \vu'$, let $j$ be the smallest index such that $u_j \neq u_j'$. For any event $E$,  we have  
\begin{align}
\abs{\Pr_{B^{(\vu)}}[E] - \Pr_{B^{(\vu')}}[E] } \leq 0.25 .
\end{align}
\end{lem}
\begin{proof}
Let $t = t_{j-1}$ be the last time step before stage $j$.
We will consider the sample space $\Omega_t$ that consists of the trajectories $(i_1, r_1, \dots, i_t, r_t)$ and the internal randomness source $\vs$ used by the learner. Now consider two probability distributions $D$ and $D'$ over $\Omega_t$, where $D$ is induced by the learner and the instance $B^{(\vu)}$, and $D'$ is induced by the learner and $B^{(\vu')}$. We will show that
\begin{align}\label{eq:lb-KL-single-stage-bound}
\mathrm{KL}(D \Vert D') \leq 0.1,
\end{align}
where $\mathrm{KL}(\cdot \Vert \cdot)$ denotes the Kullback--Leibler (KL) divergence between the two distributions, so that we can prove the lemma by invoking Pinsker's inequality (\autoref{lem:pinsker}).

We now prove \eqref{eq:lb-KL-single-stage-bound}. Fix any $\vs$, let $D_\vs$ be $D$ conditioned on $\vs$ and let $D'_\vs$ be $D'$ conditioned on $\vs$. Since $D$ and $D'$ share the same marginal distribution on $\vs$, to prove \eqref{eq:lb-KL-single-stage-bound}, we only need to show
\begin{align} \label{eq-lb-prob-difference-bound-1}
\mathrm{KL}(D_\vs \Vert D'_\vs) \leq 0.1.
\end{align}
Let $q(i_1, r_1, \dots, i_t, r_t)$ and $q'(i_1, r_1, \dots, i_t, r_t)$ be the probability density functions for $D$ and $D'$ respectively. We have that
\begin{align}
q(i_1, r_1, \dots, i_t, r_t) &= \prod_{\tau=1}^{t} \ind[i_\tau = i_\tau(i_1, r_1, \dots, i_{\tau-1}, r_{\tau-1}; \vs)] \cdot q^{(\vu)}(r_\tau | i_\tau),  \label{eq-lb-prob-difference-bound-10} \\
\text{and}\qquad q'(i_1, r_1, \dots, i_t, r_t) &= \prod_{\tau=1}^{t} \ind[i_\tau = i_\tau(i_1, r_1, \dots, i_{\tau-1}, r_{\tau-1}; \vs)] \cdot q^{(\vu')}(r_\tau | i_\tau),  \label{eq-lb-prob-difference-bound-20}
\end{align}
where $i_\tau(i_1, r_1, \dots, i_{\tau-1}, r_{\tau-1}; \vs)$ is the deterministic decision of the learner at time $\tau$ given the trajectory $(i_1, r_1, \dots, i_{\tau-1}, r_{\tau-1})$ and the learner's internal randomness source $\vs$, and $q^{(\vu)}(r_\tau \vert i_\tau)$ is the probability density function for the reward at time $t$, if playing action $i_\tau$ in instance $B^{(\vu)}$.

Since the second dimensions of $\vtheta^{(\vu)}$ and $\vtheta^{(\vu')}$ are the same, the difference of the mean reward at any time step in stage $j' < j$ for the same action in $B^{(\vu)}$ and $B^{(\vu')}$ is either $|\psi(\vu) z_{j'} - \psi(\vu') z_{j'}|$ (if the first action is played) or $0$ (if the second action is played, since $\psi_{j-1}(\vu) = \psi_{j-1}(\vu')$). Since the rewards are Gaussian with variance $1$, and the KL divergence between two variance-1 Gaussian variables with means $\mu_1$ and $\mu_2$ is $(\mu_1-\mu_2)^2/2$, for any $\tau \leq t$ that is in stage $j'$ and any $i_\tau \in \{1, 2\}$, we have that
\begin{align}
\mathrm{KL}\left(q^{(\vu)}(\cdot\vert i_\tau) ~\Vert~ q^{(\vu')}(\cdot\vert i_\tau)\right)  \leq z_{j'}^2 \cdot \frac{ (\psi(\vu) - \psi(\vu'))^2}{2} \leq z_{j'}^2 \cdot \frac{(4 \upsilon^{j})^2}{2} = 8 z_{j'}^2 \upsilon^{2j}, \label{eq-lb-prob-difference-bound-30}
\end{align}
where the last inequality is because that $j$ is the first index where $\vu$ and $\vu'$ differ and that $\upsilon \leq 1/10$. 
By \eqref{eq-lb-prob-difference-bound-10}, \eqref{eq-lb-prob-difference-bound-20}, \eqref{eq-lb-prob-difference-bound-30}, and the Chain Rule for KL divergence, we have that
\begin{align*}
\mathrm{KL}(D_\vs \Vert D_{\vs'}) \leq \sum_{j'=1}^{j-1} \sum_{\tau=t_{j'-1}+1}^{t_{j'}}  8 z_{j'}^2 \upsilon^{2(j+1)} & =  \sum_{j'=1}^{j-1} \sum_{\tau=t_{j'-1}+1}^{t_{j'}}  8\cdot \frac{\upsilon^{-2j'-2}}{T} \cdot \upsilon^{2j}\\
&\leq \frac{8(T/L)}{T} \cdot 2 \upsilon^{-2j} \cdot \upsilon^{2j} = \frac{16}{L} \leq 0.1,
\end{align*}
proving \eqref{eq-lb-prob-difference-bound-1}.
\end{proof}

We now bound $p_j^{(\vu)}$ by the following lemma.
\begin{lem}\label{lem:lb-paired-lowerbound}
Consider any $\vu = (u_1, u_2, \dots, u_L) \in \{\pm 1\}^L$ and $\vu' = (u_1', u_2', \dots, u_L') \in \{\pm 1\}^L$. Suppose $\vu \neq \vu'$, let $j$ be the smallest index such that $u_j \neq u_j'$. Then we have $p_j^{(\vu)}+p_j^{(\vu')}\ge 0.75$.
\end{lem}
\begin{proof}
Since $j$ is the smallest index such that $u_j \neq u_j'$, by our construction, exactly one of $\psi_j(\vu)$ and $\psi_j(\vu')$ is greater than $\psi_{j-1}(\vu)$, which means, at stage $j$, any action that is suboptimal for instance $B^{(\vu)}$ is optimal for instance $B^{(\vu')}$, and vice versa. Let $t = t_{j-1}+1$ be the first time step in stage $j$. Let $E$ be the event that the learner's policy $\chi_t$ for time step $t$ assigns at least $1/2$ probability mass to the suboptimal action for $B^{(\vu)}$. Since $B^{(\vu)}$ and $B^{(\vu')}$ share the same context vector set at time step $t$ (because $\psi_{j-1}(\vu) = \psi_{j-1}(\vu')$), the complement event $\overline{E}$ is that $\chi_t$ assigns at least $1/2$ probability mass to the suboptimal action for $B^{(\vu')}$. Invoking \autoref{lem:lb-prob-difference-bound}, we have that
\[
p_t^{(\vu)}+p_t^{(\vu')} = \Pr_{B^{(\vu)}}[E] + \Pr_{B^{(\vu')}}[\overline{E}] = 1 + \Pr_{B^{(\vu)}}[E] - \Pr_{B^{(\vu')}}[E] \geq 0.75.    \qedhere
\]
\end{proof}

\paragraph{Putting Things Together and the Average Case Analysis.} Note that 
\begin{align}
 &\frac{1}{2^L} \sum_{\vu \in  \{\pm 1\}^L}  \sum_{j=1}^{L} p_j^{(\vu)} = \frac{1}{2^L}   \sum_{j=1}^L \sum_{\vu \in  \{\pm 1\}^L} p_j^{(\vu)} \nonumber \\
 &\qquad\qquad\qquad\qquad\qquad\qquad =  \frac{1}{2^L}   \sum_{j=1}^L \sum_{\vu \in  \{\pm 1\}^L} \frac{p_j^{(\vu)} + p_j^{(\vu^{\oplus j})}}{2} \geq \frac{1}{2^L} \sum_{j=1}^{L} 2^L \cdot \frac{0.75}{2} = 0.375 L,  \label{eq:lb-avg-case}
\end{align}
where $\vu^{\oplus j}$ is the $\pm 1$ sequence derived by flipping the sign of the $j$-th element of $\vu$, and the inequality is due to \autoref{lem:lb-paired-lowerbound}. Therefore, there exists $\vu^*$ such that $\sum_{j=1}^{L} p_j^{(\vu^*)} \geq 0.375 L$. Together with \eqref{eq:lb-regret-rarely-switch-first-step-a}, we have
\begin{align}
\E_{B^{(\vu^*)}}[R^T] \geq \frac{\sqrt{T}}{4 \upsilon L} \times (0.375L - M) = \frac{\sqrt{T}}{4 \upsilon \cdot 4M} \times (0.375 \cdot 4M - M) = \frac{1}{32} \cdot T^{1/2 + 1/(8M+2)},
\end{align}
proving \autoref{lem:lb}.

\subsection{Proof of \autoref{thm:lb}  for General Dimensions} \label{sec:lb-general}
We equally divide the $T$ time steps into $d/2$  intervals. We construct the class of bandit instances, $\widetilde{\mathfrak{B}}$ from the class $\mathfrak{B}$ constructed in the proof of \autoref{lem:lb} as follows. For each $\ell \in \{1, 2, \dots, d/2\}$, we choose a bandit instance $B_\ell$ from $\mathfrak{B}$, and construct the $d$-dimensional instance $\widetilde{B}$. The hidden vector $\vtheta$ of $\widetilde{B}$ is derived by concatenating the hidden vectors of the $d/2$ smaller instances. During the $\ell$-th interval of time, we use the context vectors in $B_\ell$ in order: for each time step in the $\ell$-th interval, we put the $2$-dimensional context vectors in the corresponding time step in $B_\ell$ at the $(2\ell - 1)$-th and $2\ell$-th entries, while filling other entries with $0$. $\widetilde{\mathfrak{B}}$ will consist of all possible instances that can be constructed in this way, and we have $\abs{\widetilde{\mathfrak{B}}} = \abs{\mathfrak{B}}^{d/2}$. 

By our construction, the rewards from different time intervals are completely independent. 
Since the length of an interval is $T/(d/2)$, if we let $L = 4M/(d/2)$ and $\upsilon = \sqrt{T/(d/2)}^{-1/(L+1)}$, we can prove similarly as \autoref{lem:lb}  that
\begin{align}
\frac{1}{\abs{\widetilde{\mathfrak{B}}}} \sum_{\widetilde{B} \in \widetilde{\mathfrak{B}}}  \E_{\widetilde{B}}[R_\ell] \geq \frac{\sqrt{T/(d/2)}}{4\upsilon L} \times \left(0.375L - \frac{1}{\abs{\widetilde{\mathfrak{B}}}} \sum_{\widetilde{B} \in \widetilde{\mathfrak{B}}}  M_{\widetilde{B},\ell}\right), \nonumber
\end{align}
where $R_\ell$ is the regret incurred during the $\ell$-th interval, and $M_{\widetilde{B},\ell}$ is the expected number of policy switches made during the $\ell$-th interval, when given instance $\widetilde{B}$.
Therefore, we have that 
\begin{align}
  \frac{1}{\abs{\widetilde{\mathfrak{B}}}} \sum_{\widetilde{B} \in \widetilde{\mathfrak{B}}}  \E_{\widetilde{B}}[R^T]  & = \sum_{\ell=1}^{d/2} \frac{1}{\abs{\widetilde{\mathfrak{B}}}} \sum_{\widetilde{B} \in \widetilde{\mathfrak{B}}}  \E_{\widetilde{B}}[R_\ell] \geq \frac{ \sqrt{T/(d/2)}}{4\upsilon L} \times \left(\frac{3dL}{16} -\frac{1}{\abs{\widetilde{\mathfrak{B}}}}\sum_{\widetilde{B} \in \widetilde{\mathfrak{B}}} \sum_{\ell=1}^{d/2} M_{\widetilde{B},\ell}\right) \nonumber \\
  & \geq \frac{ \sqrt{T/(d/2)}}{4\upsilon L} \times \left(\frac{3dL}{16} - M\right) \geq \frac{\sqrt{T/(d/2)}}{4\upsilon \cdot 4M} \times \frac{M}{2} \geq \frac{\sqrt{T/d}}{32} \times \left(\frac{2T}{d}\right)^{1/(16M/d + 2)} , \nonumber
\end{align}
which means that there exists at least one instance from $\widetilde{\mathfrak{B}}$ such that the learner incurs at least $\sqrt{dT} \times \frac{1}{32d}\left(\frac{2T}{d}\right)^{1/(16M/d + 2)}$ expected regret, proving the theorem.

\section*{Acknowledgments}

We thank Yanjun Han, Zhengyuan Zhou, and Zhengqing Zhou for their valuable comments.

\bibliographystyle{plainnat}


\appendix

\section{Technical Lemmas}

\subsection{Concentration Inequalities}\label{app:concentration}

\begin{lem}[\citet{hoeffding1963probability}] \label{lem:hoeffding} Let $X_1, \dots, X_n \in [0, R]$ be independent bounded random variables. Let $\overline{X} = \frac{1}{n}\sum_{i = 1}^n X_i$ be their average. Then 
\begin{align*}
    \Pr[\abs{\overline{X} - \E \overline{X}} \ge \delta] \le 2\exp(- \frac{2 n \delta^2}{R^2}).
\end{align*}
\end{lem}

\begin{lem}[\citet{tropp2012user}, Theorem 1.1] \label{lem:mchernoff} Let $\mX_1, \dots, \mX_n$ be a sequence of independent  positive semi-definite random matrices in dimension $d$ such that $\norm{\mX_i} \le R$ almost surely (where $\norm{\cdot}$ denotes the operator norm). Let $\mX = \sum_{i = 1}^n \mX_i$ be their sum. Let $\mu_{\min} = \lambda_{\min}(\E \mX)$ and $\mu_{\max} = \lambda_{\max}(\E \mX)$. Then we have
\begin{align*}
    \Pr[\lambda_{\min}(\mX) \le (1 - \delta) \mu_{\min}] \le d \left( \frac{e^{-\delta}}{(1 - \delta)^{1 - \delta}}\right)^{\mu_{\min} / R}, &\qquad \mathrm{when}~\delta \in [0, 1], \\ 
    \Pr[\lambda_{\max}(\mX) \ge (1 + \delta) \mu_{\max}] \le d \left( \frac{e^{\delta}}{(1 + \delta)^{1 + \delta}}\right)^{\mu_{\max} / R}, &\qquad \mathrm{when}~\delta \ge 0.
\end{align*}
\end{lem}

\begin{lem} \label{lem:covcon} Suppose $\vx_1, \dots, \vx_{n} \sim \gD$ are {\it i.i.d.}~drawn from a distribution $\gD$ and $\vx_i^\top \vx_i \leq 1$ almost surely. Let $\lambda = \lambda_{\min}(\E_{\vx \sim \gD}{\vx \vx^\top}) > 0$ be the smallest eigenvalue of the co-variance matrix. We have that 
\begin{align}
\Pr[\frac{1}{n} \sum_{i = 1}^n \vx_i \vx_i^\top \succcurlyeq  \frac{1}{2} \E_{\vx \sim \gD} \vx \vx^\top] \ge 1 - d \exp(-\frac{\lambda n}{8}). 
\end{align}
\end{lem}

\begin{proof} 
Let $\mSigma = \E_{\vx \sim \gD}{\vx \vx^\top}$ and $\vy_i = \mSigma^{-1/2} \vx_i$ for all $i \in [n]$. Note that $\norm{\vy_i \vy_i^\top} \leq \lambda^{-1}$ almost surely, and that $\E_{\vy_i} \vy_i \vy_i^\top = \mI$.  Therefore, by \autoref{lem:mchernoff}, we have that
\begin{align*}
1 - d\exp(-\frac{\lambda n}{8}) &\leq \Pr[\frac{1}{n}\sum_{i=1}^{n} \vy_i\vy_i^\top \succcurlyeq \frac{1}{2} \mI]  = \Pr[\frac{1}{n}\sum_{i=1}^{n} \vy_i\vy_i^\top \succcurlyeq \frac{1}{2} \E \mSigma^{-1/2} \vx\vx^\top \mSigma^{-1/2}] \\
&= \Pr[\frac{1}{n} \sum_{i = 1}^n \vx_i \vx_i^\top \succcurlyeq  \frac{1}{2} \E_{\vx \sim \gD} \vx \vx^\top]. \qedhere
\end{align*}
\end{proof}

\begin{lem} \label{lem:conconcut}  Suppose $\vx_1, \dots, \vx_{n} \sim \gD$ are {\it i.i.d.}~drawn from a distribution $\gD$ and $\vx_i^\top \vx_i \leq 1$ almost surely. For any cutoff level $\lambda > 0$, with probability at least $(1 - 2d \exp(- \frac{\lambda n}{8}))$, we have that 
\[
3 \lambda \mI + \frac{1}{n} \sum_{i = 1}^n \vx_i \vx_i^\top \succcurlyeq \frac{1}{8}\E_{\vx \sim \gD} \vx \vx^\top.
\]
\end{lem}

\begin{proof}
Suppose $\E_{\vx \sim \gD} \vx \vx^\top  = \sum_{i=1}^d \lambda_i \vv_i \vv_i^\top$ where $\{\vv_i\}_{i=1}^d$ is a set of orthonormal basis. Let $\mP_{+} = \sum_{i=1}^d \vv_i \vv_i^\top \ind[\lambda_i \geq\lambda]$ and $\mP_{-} = \sum_{i=1}^d \vv_i \vv_i^\top \ind[\lambda_i <\lambda]$, so that $\mI = \mP_+ \mP_-$. Observe that the eigenvalues of $\E_{\vx \sim \gD} \mP_{+} \vx \vx^\top \mP_+^\top$ are greater or equal to $\lambda$ when restricted to the space spanned by the $\mP_+$. Therefore, by \autoref{lem:covcon}, we have that with probability at least $(1 - d \exp(-\lambda n/8))$,
\begin{align}
    \frac{1}{n} \sum_{i=1}^n (\mP_+\vx_i) (\mP_+\vx_i)^\top \succcurlyeq \frac{1}{2} \E_{\vx \sim \gD} \mP_{+} \vx \vx^\top \mP_+^\top. \label{eq:lem-conconcut-1}
\end{align}
Note that
\begin{align}
&\quad \frac{1}{n} \sum_{i = 1}^n \vx_i \vx_i^\top = \frac{1}{n} (\mP_+  \mP_-) \sum_{i = 1}^n \vx_i \vx_i^\top (\mP_+  \mP_-)^\top \notag \\
&= \frac{1}{2n}\sum_{i=1}^n \mP_{+} \vx_i \vx_i^\top \mP_+^\top + \frac{1}{n} \sum_{i=1}^n \left(\frac{1}{2} \mP_{+} \vx_i \vx_i^\top \mP_+^\top + \mP_{+} \vx_i \vx_i^\top \mP_-^\top + \mP_{-} \vx_i \vx_i^\top \mP_+^\top + 2 \mP_{-} \vx_i \vx_i^\top \mP_-^\top\right) \nonumber \\
& \qquad - \frac{1}{n} \mP_{-} \vx_i \vx_i^\top \mP_-^\top, \label{eq:lem-conconcut-2}
\end{align}
where the first term is $\succcurlyeq \frac{1}{4} \E_{\vx \sim \gD} \mP_{+} \vx \vx^\top \mP_+^\top$ by \eqref{eq:lem-conconcut-1}, the second term is a sum of positive semi-definite matrices, and for the third term, by \autoref{lem:mchernoff}, with probability is at least $(1 - d \exp(-\lambda n /3)$, we have that
\begin{align*}
    \frac{1}{n} \sum_{i=1}^n \mP_{-} \vx_i \vx_i^\top \mP_-^\top \preccurlyeq 2 \lambda \mI. 
\end{align*}
Therefore, continuing with \eqref{eq:lem-conconcut-2}, and collecting probabilities, we have that with probability at least $(1 - 2d \exp(-\lambda n /8))$,
\begin{align*}
     \frac{1}{n} \sum_{i = 1}^n \vx_i \vx_i^\top& \succcurlyeq \frac{1}{4} \E_{\vx \sim \gD} \mP_{+} \vx \vx^\top \mP_+^\top - 2 \lambda \mI \\
     &= \frac{1}{8} \E_{\vx \sim \gD}  \vx \vx^\top + \frac{1}{8} \E_{\vx \sim \gD}\left(\mP_{+} \vx \vx^\top \mP_+^\top - \mP_{+} \vx \vx^\top \mP_-^\top -  \mP_{-} \vx \vx^\top \mP_+^\top + \mP_{-} \vx \vx^\top \mP_-^\top \right) \nonumber \\
     & \qquad \qquad - \frac{1}{4} \E_{\vx \sim \gD}  \mP_{-} \vx \vx^\top \mP_-^\top - 2 \lambda \mI \\
     &\succcurlyeq \frac{1}{8} \E_{\vx \sim \gD}  \vx \vx^\top - 3\lambda \mI,
\end{align*}
proving the lemma.
\end{proof}

\begin{lem}[\citet{vershynin2018high}, Theorem 5.4.1] \label{lem:mbernstein} Let $\mX_1, \dots, \mX_n \in \sR^{d \times d}$ be independent symmetric random matrices, such that $\E \mX_i = \vzero$ and $\norm{\mX_i} \le R$ almost surely. Then 
\begin{align*}
    \Pr[\norm{\sum_{i = 1}^n \mX_i} \ge \delta] \le 2 n \exp(-\frac{\delta^2 / 2}{\sigma^2 + n R /3 }), \qquad \text{where} \quad \sigma^2 = \norm{\sum_{i = 1}^n \E \mX_i^2}. 
\end{align*}
\end{lem}

\begin{lem} \label{lem:ourmhoeffding} Let $\mX_1, \dots, \mX_n \in \sR^{d \times d}$ be a sequence of i.i.d.\ positive semi-definite random matrices such that $\norm{\mX_i} \le R$ almost surely. Let $\overline{\mX} = \frac{1}{n}\sum_{i = 1}^n \mX_i$ be their average. Then we have
\begin{align*}
    \Pr[\norm{\overline{\mX} - \E \overline{\mX}} > \delta] \le 2 d \exp(- \frac{n\delta^2 }{8 R^2  + 4\delta R / 3}) .
\end{align*}
\end{lem}

\begin{proof} Define $\mY_i = \frac{\mX_i - \E \overline{\mX}}{n}$. Note that $\norm{\mX_i}, \norm{\overline{\mX}} \le R$ almost surely, so $\norm{\mY_i} \le 2 R / n$ almost surely. Furthermore, we have
\begin{align*}
    \sigma^{2} = \norm{\sum_{i = 1}^n \E \mY_i^2} \le n \cdot  \frac{4 R^2}{n^2} = \frac{4 R^2}{n}. 
\end{align*}
By \autoref{lem:mbernstein}, we have 
\begin{align*}
    \Pr[\norm{\overline{\mX} - \E\overline{\mX}} \ge \delta] &= \Pr[\norm{\sum_{i = 1}^n \mY_i} \ge \delta] \\ 
    &\le 2 d \exp(- \frac{\delta^2 / 2}{\sigma^2 + 2\delta R / (3n)}) \le 2 d \exp(- \frac{n\delta^2 }{8 R^2  + 4\delta R / 3}).
\end{align*}
\end{proof}

\begin{lem}[\citet{abbasi2011improved}, Theorems 1 and 2] \label{lem:abbasi-ball} Let $\{\gF_{i}\}_{i = 0}^\infty$ be a filtration.  Let $\{\vx_{i}\}_{i = 1}^\infty$ be an $\sR^d$-valued stochastic process such that $\vx_i$ is $\gF_{i-1}$-measurable and $\norm{\vx_i} \le 1$ almost surely. Let $\{\epsilon_i\}_{i = 1}^\infty$ be a real-valued stochastic process such that $\varepsilon_i$ is $\gF_{i}$-measurable and is sub-Gaussian with variance proxy $1$ when conditioned on $\gF_{i - 1}$. Fix $\vtheta \in \sR^d$ such that $\norm{\vtheta} \le 1$. Let $\mLambda_n = \lambda \mI + \sum_{i = 1}^n \vx_i \vx_i^\top, y_i = \vx_i^\top \vtheta + \varepsilon_i$, and $\hat \vtheta_n = \mLambda_n^{-1} \sum_{i = 1}^n y_i \vx_i$. For every $\delta > 0$, we have that 
\begin{align}
    \Pr[\forall n \ge 0: \norm{\hat \vtheta_n - \vtheta}_{\mLambda_n} \le \sqrt{\lambda} + \sqrt{d \ln(\frac{1 + n  / \lambda}{ \delta})} ] \ge 1 - \delta, \label{eq:abbasi-event}
\end{align}
where we define $\norm{\vx}_{\mLambda} \defeq \sqrt{\vx^\top \mLambda \vx}$.
Furthermore, when the event specified in \eqref{eq:abbasi-event} holds, we have for every $n \ge 0$ and any vector $\vx \in \sR^d$ that 
\begin{align}
    \abs{\vx^\top (\hat \vtheta_n - \vtheta)} \le \left(\sqrt{\lambda} + \sqrt{d \ln(\frac{1 + n  / \lambda}{ \delta})}\right) \sqrt{\vx^\top \mLambda_n^{-1} \vx}. \label{eq:abbasi-event-2}
\end{align}
\end{lem}

\subsection{Tools for Matrix Operations}

\begin{lem}[\citet{abbasi2011improved}, Lemma 12] \label{lem:comp} Given two  positive semi-definite matrices $\mA$ and $\mB$. Suppose that $\mA \succcurlyeq \mB$. Then we have that  
\begin{align*}
\sup_{\vx \ne \vzero} \frac{\vx^\top \mA \vx}{\vx^\top \mB \vx}  \le \frac{\det \mA}{\det \mB}.
\end{align*}
\end{lem}

\begin{lem} \label{lem:unif} Let $\mA \succcurlyeq \vzero$ be a positive semi-definite matrix. Suppose we are given a vector $\vx \in \sR^d$ such that $\vx^\top \mA^{-1} \vx \le z$, then we have $z \mA \succcurlyeq \vx \vx^\top$.
\end{lem}

\begin{proof} Without loss of generality, assume that $\mA = \mathrm{diag}(\lambda_1, \dots, \lambda_d)$ is diagonal. Let $\vx = (x_1, \dots, x_d)^\top \in \sB^d$. For any vector $\vy = (y_1, \dots, y_d)^\top$, we have that
\begin{align*}
   \vy^\top (z \mA) \vy \ge (\vy^\top \mA \vy) (\vx^\top \mA^{-1} \vx^\top ) = \left( \sum_{i = 1}^d \lambda_i y_i^2 \right) \left( \sum_{i = 1}^d \lambda_i^{-1} x_i^2 \right) \ge \left(\sum_{i = 1}^d x_i y_i \right)^2 
   = \vy^\top (\vx \vx^\top) \vy,
\end{align*}
where the last inequality is by Cauchy-Schwarz.
\end{proof}


\begin{lem}[The matrix covering number] \label{lem:matcover} The covering number of the matrix set 
\begin{align*}
    \rmM =\{\mA \in \sR^{d \times d} \mid R_1 \mI \preccurlyeq \mA \preccurlyeq R_2 \mI\}
\end{align*}
is bounded by 
\begin{align*}
    \log \gN(\rmM, \varepsilon) \le O(d^2 \log(\max\{\abs{R_1}, \abs{R_2}\}d/\varepsilon)). 
\end{align*}
\end{lem}
\begin{proof}
Let $\widetilde{\mathfrak{M}} = \{\mA : \mA_{ij} \in [-\max\{R_1, R_2\}, \max\{\abs{R_1}, \abs{R_2}\}] \cap \{k \epsilon/(10d^2) : k \in \sZ\}\}$, and let $\mathfrak{M}$ be the projection of $\widetilde{\mathfrak{M}}$ onto $\rmM$. One can show that $\mathfrak{M}$ is an $\varepsilon$-cover of $\rmM$, and $\log \abs{\mathfrak{M}} \leq O(d^2 \log (\max\{\abs{R_1}, \abs{R_2}\}d/\epsilon))$.
\end{proof}

\begin{lem}[Lipschitzness of matrix inverse] \label{lem:matinvlip} For any two positive semi-definite matrices $\mA, \mB \succcurlyeq \lambda \mI$, we have that 
\begin{align*}
    \norm{\mA^{-1} - \mB^{-1}} \le \lambda^{-2} \norm{\mA - \mB} \qquad \mathrm{and} \qquad  \norm{\mA^{-1/2} - \mB^{-1/2}} \le \lambda^{-3/2} \norm{\mA - \mB}.
\end{align*}
\end{lem}

\begin{proof} Note that $\norm{\mA^{-1}}, \norm{\mB^{-1}} \le \lambda^{-1}$.  We have that
\begin{align*}
    \norm{\mA^{-1} - \mB^{-1}} &\le \norm{\mA^{-1}} \norm{\mI - \mA \mB^{-1}} 
    \le \lambda^{-1} \norm{\mI - (\mA - \mB + \mB) \mB^{-1}} \\
    &\le \lambda^{-1} \norm{\mB^{-1}} \norm{\mA - \mB} 
    \le \lambda^{-2} \norm{\mA - \mB}. 
\end{align*}
It remains to show that $\norm{\mA^{-1/2} - \mB^{-1/2}} \le \lambda^{-3/2} \norm{\mA - \mB}$. Since $\mA^{-1/2}, \mB^{-1/2} \succcurlyeq  \lambda^{1/2} \mI$, we have that
\begin{align}
    \norm{\mA^{-1/2} - \mB^{-1/2}} \le \lambda^{-1} \norm{\mA^{1/2} - \mB^{1/2}}. \label{eq:matinvlip-1}
\end{align}
To complete, we assume that $\norm{\mA - \mB} \le \varepsilon (<\lambda^{-1})$. For any unit vector $\vx \in \sS^{d-1}$, we have 
\begin{align*}
\vx^\top (\mA^{1/2} - \mB^{1/2}) \vx &= \sqrt{\vx^\top \mA \vx} - \sqrt{\vx^\top \mB \vx} \\ 
&= \sqrt{\vx^\top \mB \vx + \vx^\top (\mA - \mB) \vx} - \sqrt{\vx^\top \mB \vx} \\ 
&\le \sqrt{\vx^\top \mB \vx + \varepsilon} - \sqrt{\vx^\top \mB \vx} \\ 
&\le \varepsilon / (2 \sqrt{\vx^\top \mB \vx}) \\ 
&\le \varepsilon / (2 \sqrt \lambda).
\end{align*}
By swapping $\mA$ and $\mB$, we can show 
\begin{align*}
    \vx^\top (\mB^{1/2} - \mA^{1/2}) \vx \le \varepsilon / (2 \sqrt \lambda).
\end{align*}
Therefore, we have 
\begin{align*}
    \abs{\vx^\top (\mA^{1/2} - \mB^{1/2}) \vx} \le \varepsilon / (2 \sqrt \lambda),
\end{align*}
which implies $\norm{\mA^{1/2} - \mB^{1/2}} \le \varepsilon / (2 \sqrt \lambda) \le \norm{\mA - \mB} / (2 \sqrt \lambda)$ by the definition of the matrix norm. We conclude with \eqref{eq:matinvlip-1}. 
\end{proof}

\subsection{The Generalized Elliptical Potential Lemma}

Below we prove a generalized version of the elliptical potential lemma. Compared to the usual version in literature (e.g., \citep{abbasi2011improved}), our versions works for positive semi-definite matrices $\mX_1, \dots, \mX_n$ with traces upper bounded by $1$ instead of just rank-$1$ positive semi-definite matrices. However, we also need the extra assumption that $\Tr(\mX_i  \mV_0^{-1}) \le 1$ for all $i \in [n]$.

\begin{lem}[Generalized Elliptical Potential Lemma] \label{lem:ellip} Suppose we are given a sequence of positive semi-definite matrices  $\mX_1, \dots, \mX_n$ such that $\Tr(\mX_i) \le 1$ for every $i \in [n]$. Let $\mLambda_0$ be a positive semi-definite matrix and let $\mLambda_i = \mLambda_{i - 1} + \mX_i$ for $i \in [n]$. When $\Tr(\mX_i  \mLambda_0^{-1}) \le 1$ for $i \in [n]$, we have 
\begin{align*}
\sum_{i = 1}^n \Tr(\mX_i \mLambda_{i-1}^{-1}) \le 2 \ln \frac{\det \mLambda_n}{\det \mLambda_0}.
\end{align*}
\end{lem}

\begin{proof} 
Note that 
\begin{align*}
    \mLambda_i = \mLambda_{i - 1} + \mX_i = \mLambda_{i - 1}^{1/2}(\mI + \mLambda_{i - 1}^{-1/2} \mX_i \mLambda_{i-1}^{-1/2}) \mLambda_{i - 1}^{1/2},
\end{align*}
so we have 
\begin{align*}
    \det \mLambda_i &= \det(\mLambda_{i - 1}) \times \det(\mI + \mLambda_{i - 1}^{-1/2} \mX_i \mLambda_{i-1}^{-1/2})  \\
    &\ge \det(\mLambda_{i - 1}) \times \left(1 + \Tr(\mLambda_{i - 1}^{-1/2} \mX_i \mLambda_{i-1}^{-1/2}) \right) \\ 
    &= \det(\mLambda_{i - 1}) \times \left(1 + \Tr(\mLambda_{i - 1}^{-1} \mX_i) \right),
\end{align*}
where the inequality follows from that 
\begin{align*}
\det(\mI + \mA) = \prod_{j = 1}^d (1 + \lambda_j) \ge 1 + \sum_{j = 1}^d \lambda_j = 1 + \Tr(\mA), 
\end{align*}
where $\lambda_j \ge 0$ is the $j$-th eigenvalue of $\mA = \mLambda_{i - 1}^{-1/2} \mX_i \mLambda_{i-1}^{-1/2}$. Together with the fact that $x \le 2 \ln(1 + x)$ for $x \in [0, 1]$, we have 
\[
   \sum_{i = 1}^n \Tr(\mX_i \mLambda_{i - 1}^{-1}) \le \sum_{i = 1}^n 2\ln(1 + \Tr(\mX_i \mLambda_{i - 1}^{-1})) \le 2 \sum_{i = 1}^n \ln \frac{\det \mLambda_i}{\det \mLambda_{i - 1}} = 2 \ln \frac{\det \mLambda_n}{\det \mLambda_0}. \qedhere
\]
\end{proof}

\subsection{Pinsker's Inequality}
\begin{lem}\label{lem:pinsker}
If $P$ and $Q$ are two probability distributions on a measurable space $(X, \Sigma)$, then for any event $A \in \Sigma$, it holds that
\[
\left| P(A) - Q(A) \right| \leq \sqrt{\frac{1}{2} \mathrm{KL}(P \| Q)},
\]
where 
\[ 
\mathrm{KL}(P \Vert Q) = \int_X \left(\ln \dv{P}{Q}\right) \dd{P}
\]
is the Kullback--Leibler divergence.
\end{lem}

\section{Omitted Algorithms, Lemmas and Proofs in \autoref{sec:blinucb}}

\subsection{Full Description of \BatchedLinUCBKW} \label{app:blinucbkw}

The algorithm is presented in \autoref{algo:linucbkw}.

\begin{algorithm}[t]
\caption{\BatchedLinUCBKW}
\label{algo:linucbkw}
$M = \lceil \log \log T \rceil, \alpha \gets 10\sqrt{ \ln \frac{2dKT}{\delta}}$, $\gT = \{\gT_1, \gT_2, \dots, \gT_M\}, \gT_0 = 0, \gT_M = T, \forall i \in [M - 1]: \gT_i = T^{1 - 2^{-i}}$\;
\For{$k \gets 1, 2, \dots, M$}{
 $\lambda \gets 16 \ln (2dT/\delta), \mLambda_{k} \gets \lambda \mI, \vxi_{k} \gets \vzero$\;
\For{$t \gets\gT_{k - 1} + 1, \gT_{k - 1} + 2, \dots, \gT_k$}{
$A_{t}^{(0)} \gets [K], \hat r_{ti}^{(0)} \gets 0, \omega_{ti}^{(0)} \gets 1$\;
\For(\Comment*[f]{Eliminate}){$\kappa \gets 1, 2, \dots, k - 1$ \label{line:batchedlinucbkw-6}}{
$\forall i \in A_{t}^{(\kappa-1)} : \hat r_{ti}^{(\kappa)} \gets  \vx_{t i}^\top \hat \vtheta_{\kappa}, \omega_{t i}^{(\kappa)} \gets \alpha \sqrt{\vx_{t i}^\top \mLambda_{\kappa}^{-1} \vx_{t i}}$\;
$ A_{t }^{(\kappa)} \gets \{i \in  A_{t }^{(\kappa-1)} \mid \hat r_{ti}^{(\kappa)} +   \omega_{ti}^{(\kappa)} \ge \hat r_{tj}^{(\kappa)} -  \omega_{tj}^{(\kappa)},  \forall j \in  A_{t }^{(\kappa-1)}\}$\; \label{line:batchedlinucbkw-8}
}
$A_t \gets A_{t}^{(k - 1)}$\; 
Select $i_t$ such that $\vx_{t,i_t} \sim \gopt(\{\vx_{t,i} : i \in A_t\})$, play arm $i_t$, and receive reward $r_t$\;\label{line:batchedlinucbkw-10}
    $\vx_t \gets \vx_{t, i_t}, \mLambda_{k} \gets \mLambda_{k} + \vx_{t} \vx_{t}^\top, \vxi_{k} \gets \vxi_{k} + r_t \vx_t$\;
}
$\hat \vtheta_{k} \gets \mLambda_{k}^{-1} \vxi_k$\;
}
\end{algorithm}

\subsection{Proof of \autoref{lem:conf} (Analysis of Linear Regression)} \label{app:analysis-lr}

\autoref{lem:conf} can be proved by a straightforward union bound over all stages and candidate arms, and the application of the following lemma.
\begin{lem}\label{lem:analysis-lr}
Given $\vtheta, \vx_1, \vx_2, \dots, \vx_n \in \R^d$ such that $\norm{\vtheta} \leq 1$, for all $i \in [n]$, let $y_i = \vx_i^\top \vtheta + \epsilon_i$ where $\epsilon_i$ is an independent sub-Gaussian random variable with variance proxy $1$. Let $\mLambda = \lambda \mI + \sum_{i=1}^n \vx_i \vx_i^\top$, and $\hat\vtheta = \mLambda^{-1} \sum_{i=1}^n y_i \vx_i$. For any $\vx \in \R^d$ and any $\gamma > 0$, we have that
\begin{align*}
    \Pr[\abs{\vx^\top (\vtheta - \hat\vtheta)} > (\gamma + \sqrt{\lambda}) \sqrt{\vx^\top \Lambda^{-1} \vx}] \leq 2\exp(-\gamma^2/2).
\end{align*}
\end{lem}
\begin{proof}
Note that 
\begin{align}
& \abs{\vx^\top (\vtheta - \hat\vtheta)} =\abs{ \vx^\top \left(\mLambda^{-1} \sum_{i=1}^n \vx_i (\vx_i^\top \vtheta + \epsilon_i)  - \vtheta\right)} \notag \\
& \qquad = \abs{\vx^\top \left(\mLambda^{-1} \sum_{i=1}^n \vx_i \epsilon_i  + \mLambda^{-1}(\mLambda - \lambda \mI) \vtheta - \vtheta\right)} = \abs{\vx^\top \mLambda^{-1} \left(\sum_{i=1}^n \vx_i \epsilon_i  - \lambda \vtheta\right)} \notag \\
& \qquad \leq \lambda \abs{\vx^\top \mLambda^{-1} \vtheta} +   \abs{\sum_{i=1}^n \vx^\top \mLambda^{-1}\vx_i \epsilon_i}. \label{eq:lemma-lr-0}
\end{align}
For the first term, since $\|\vtheta\| \leq 1$ and $\mLambda \succcurlyeq \lambda \mI$, we have that 
\begin{align}
  \lambda \abs{\vx^\top \mLambda^{-1} \vtheta} \leq \lambda  \sqrt{\vx^\top \mLambda^{-2} \vx} \leq \sqrt{ \lambda \vx^\top \mLambda^{-1} \vx}. \label{eq:lemma-lr-part1}
\end{align}
For the second term, since $\sum_{i=1}^n \vx^\top \mLambda^{-1}\vx_i \epsilon_i$ is independent sub-Gaussian with variance proxy
\begin{align*}
   \vx^\top \mLambda^{-1} \left( \sum_{i=1}^{n}  \vx_i \vx_i^\top \right) \mLambda^{-1} \vx \leq \vx^\top \mLambda^{-1} \vx,
\end{align*}
by sub-Gaussian concentration inequalities, we have
\begin{align}
\Pr[\abs{\sum_{i=1}^n \vx^\top \mLambda^{-1}\vx_i \epsilon_i} >  \gamma \sqrt{\vx^\top \mLambda^{-1} \vx}] \leq 2 \exp(-\gamma^2/2). \label{eq:lemma-lr-part2}
\end{align}
Combining \eqref{eq:lemma-lr-0}, \eqref{eq:lemma-lr-part1}, and \eqref{eq:lemma-lr-part2}, we prove the lemma.
\end{proof}

\end{document}